\documentclass[twoside,11pt]{article}

\usepackage{blindtext}

 \usepackage[abbrvbib, nohyperref, preprint]{jmlr2e}

\usepackage{amsmath}
\usepackage{amssymb}
\usepackage{mathtools}
\usepackage{dsfont}
\usepackage{microtype}

\DeclareMathOperator{\Hom}{Hom}

\DeclareMathOperator{\Bell}{B}

\DeclareMathOperator{\Rep}{Rep}

\usepackage{mathdots}
\usepackage{bm}
\usepackage{nicematrix}
\usepackage{xcolor}

\usepackage{tabularray}
\usepackage{comment}

\usepackage{xurl}

\usepackage{tikzit}

\tikzstyle{node}=[fill=white, draw=black, shape=circle, minimum size=1mm, ultra thick]
\tikzstyle{red node}=[fill=white, draw=red, shape=circle, minimum size=1mm, ultra thick]
\tikzstyle{green node}=[fill=white, draw={rgb,255: red,0; green,138; blue,0}, shape=circle, minimum size=1mm, ultra thick]
\tikzstyle{small box}=[fill=white, draw=black, shape=rectangle, minimum height=0.5cm, minimum width=0.5cm, ultra thick]
\tikzstyle{weyl}=[fill=white, draw={rgb,255: red,0; green,0; blue,109}, shape=rectangle, minimum height=0.5cm, minimum width=0.5cm]
\tikzstyle{filled_node}=[fill=black, draw=black, fill opacity=1, shape=circle, minimum size=1mm, ultra thick]
\tikzstyle{large box}=[fill=white, draw=black, shape=rectangle, minimum height=0.5cm, minimum width=1cm, ultra thick]
\tikzstyle{blue node}=[fill=white, draw=blue, shape=circle, minimum size=1mm, ultra thick]

\tikzstyle{pattern_node}=[
    fill=black,
    pattern={Hatch[angle=45,distance=1.05mm, line width=0.2mm]},
    draw=black,
    text=white,
    shape=circle,
    minimum size=1mm,
    ultra thick,
    inner sep=5pt,
    rounded corners
]

\tikzstyle{thick}=[-, ultra thick]
\tikzstyle{blue_thick}=[-, ultra thick, draw=blue]
\tikzstyle{dashes}=[-, dashed, draw={rgb,255: red,191; green,191; blue,191}, dash pattern=on 2mm off 1mm, fill={rgb,255: red,244; green,228; blue,0}]
\tikzstyle{thick_arrow}=[ultra thick, ->]
\tikzstyle{dash_1}=[-, dashed]
\tikzstyle{dash_2}=[-, dashed, fill={rgb,255: red,246; green,235; blue,255}]
\tikzstyle{dash_3}=[-, dashed, fill={rgb,255: red,227; green,255; blue,209}]
\tikzstyle{dash_4}=[-, dashed, fill={rgb,255: red,255; green,209; blue,153}]
\tikzstyle{red_thick}=[-, ultra thick, draw=red]
\tikzstyle{dash_5}=[-, dashed, fill={rgb,255: red,225; green,255; blue,254}]
\tikzstyle{dash_6}=[-, dashed, fill={rgb,255: red,251; green,245; blue,160}]
\tikzstyle{jellyfish}=[-, ultra thick, fill={rgb,255: red,34; green,48; blue,255}]
\tikzstyle{arrow}=[thick, ->]
\tikzstyle{dashed_arrow}=[dashed, ->]
\tikzstyle{dashed_thick_arrow}=[ultra thick, dashed, ->]
\tikzstyle{red_thick_arrow}=[ultra thick, ->, draw=red]
\tikzstyle{thick_blue_double_arrow}=[ultra thick, <->, draw=blue]
\tikzstyle{thick_green_double_arrow}=[ultra thick, <->, draw={rgb,255: red,0; green,138; blue,0}]

\usepackage[most]{tcolorbox}

\definecolor{melon}{RGB}{227, 168, 105} 
\definecolor{terracotta}{RGB}{204, 78, 63}

\usepackage{lastpage}
\jmlrheading{23}{2024}{1-\pageref{LastPage}}{1/21; Revised 5/22}{9/22}{21-0000}{
	Edward Pearce-Crump and William J. Knottenbelt}

\ShortHeadings{
	A Diagrammatic Approach for Efficient Group Equivariant Neural Networks
	}{Pearce-Crump and Knottenbelt}
\firstpageno{1}

\begin{document}

\title{
	A Diagrammatic Approach to Improve Computational Efficiency in Group Equivariant Neural Networks
	}

\author{\name Edward Pearce-Crump \email ep1011@imperial.ac.uk \\
       \addr Department of Computing\\
       Imperial College London\\
       London, SW7 2RH, United Kingdom
       \AND
       \name William J. Knottenbelt \email wjk@imperial.ac.uk \\
       \addr Department of Computing\\
       Imperial College London\\
       London, SW7 2RH, United Kingdom
}

\editor{My editor}

\maketitle

\begin{abstract}

	Group equivariant neural networks are growing in importance owing to 
	their ability to generalise well in applications where the data has
	known underlying symmetries.
	Recent characterisations of a class of these networks that use 
	high-order tensor power spaces as their layers suggest that they have significant
	potential; however, their implementation remains challenging owing to the 
	prohibitively expensive nature of the computations that are involved.
	In this work, we present a fast matrix multiplication algorithm 
	for any equivariant weight matrix that maps between tensor power layer spaces 
	in these networks for four groups:
	the symmetric, orthogonal, special orthogonal, and symplectic groups.
	We obtain this algorithm by developing a diagrammatic framework 
	based on category theory that enables us to not only express each weight matrix
	as a linear combination of diagrams but also 
	makes it possible for us 
	to use these diagrams to factor the original computation
	into a series of steps that are optimal.
	We show that this algorithm improves the Big-$O$ time complexity 
	exponentially 
	in comparison to
	a na\"{i}ve matrix multiplication.
\end{abstract}



\section{Introduction}

Incorporating symmetries as an inductive bias in neural network design
has emerged as an important method
for creating more structured and efficient models.
These models, commonly referred to as
\textit{group equivariant neural networks}
\citep{
cohenc16, cohen2017steerable,
qi2017pointnet, ravanbakhsh17a, deepsets, 
cohen2018, esteves2018, kondor18a, thomas2018, maron2018, 
cohen2019, weiler2019,  
villar2021scalars, finzi, pearcecrumpB, pearcecrumpJ, pearcecrump, godfrey, pearcecrumpG},
have proven to be useful across a wide range of applications,
including, but not limited to,
molecule generation \citep{satorras21a}; 
designing proteins \citep{Jumper2021};
natural language processing \citep{gordon2020, petrache2024};
computer vision \citep{chatzipantazis2023}; 
dynamics prediction \citep{guttenberg2016}; and
even auction design \citep{rahme}.

Notably, several recent works have characterised 
the weight matrices of a class of group equivariant neural networks 
that use high-order tensor power spaces as their layers 
\citep{ravanbakhsh17a, maron2018,pearcecrumpB, pearcecrumpJ, pearcecrump, godfrey, pearcecrumpG}.
However, because these layers are tensor power spaces,
the computations that are involved --- specifically, applying an equivariant weight matrix to an input vector to transform it to an output vector ---
can be very expensive.
Indeed, if the weight matrix is a function
$(\mathbb{R}^{n})^{\otimes k} \rightarrow (\mathbb{R}^{n})^{\otimes l}$, 
then a na\"{i}ve matrix multiplication on an input vector 
$v \in (\mathbb{R}^{n})^{\otimes k}$ has a time complexity of $O(n^{l+k})$.
Hence there is a need to develop methods that make
these computations feasible in practice
so that the networks that use
these weight matrices become more efficient.


We address this problem in the following way, for four groups whose
weight matrices have been characterised previously:
the symmetric \citep{ravanbakhsh17a, maron2018, pearcecrump, godfrey}, orthogonal, special orthogonal and symplectic groups \citep{pearcecrumpB}.
We first develop a diagrammatic framework based on category theory, 
enabling us to express each group equivariant weight matrix as the image 
under a monoidal functor of a linear combination of 
morphisms in a diagram category.
In this way, we can effectively 
interpret the morphisms in each diagram category as being equivalent to matrices.
We then use the diagrammatic framework to develop an algorithm 
that improves upon 
the na\"{i}ve weight matrix multiplication \textit{exponentially} in terms of
its Big-$O$ time complexity.
We achieve this by using the diagrams that are equivalent
to the original weight matrix to factor it
into a sequence of operations that carry out the original calculation
in the most efficient way possible.
We suggest that our result might enable more widespread adoption of 
group equivariant neural networks that use high-order tensor power spaces
in practical applications.

The rest of the paper is organised as follows. 
After discussing the related work in Section \ref{relatedwork}, 
we review in Section \ref{background}
the important concepts behind the characterisation
of the equivariant weight matrices between tensor power spaces of $\mathbb{R}^{n}$
that have appeared elsewhere in the literature.
In Section \ref{weightmatcat}, we introduce our diagrammatic framework
that is based on category theory, to show that each linear map
that we review in Section \ref{background} is actually the result 
of a monoidal functor between a diagram category and a representation category.
In particular, we conclude that section by discussing some fundamental results
that form the basis of our main contribution, which is introduced in Section 
\ref{algorithmsect}: the fast multiplication algorithm for passing
a vector through an equivariant weight matrix for the four groups in question.
In this section, we present our algorithm, discuss its implementation
for each of the four groups, and analyse its Big-$O$ time complexity in each case.
Finally, we conclude in Section \ref{conclusion}.



\section{Related Work} \label{relatedwork}

There is a growing body of literature demonstrating the application of category
theory in machine learning.
A number of works have used a category-theoretic approach either to
constuct, categorise, or analyse
neural network architectures
\citep{gavranovic24a, abbott2024, khatri2024, tull2024}.
Some researchers have explored 
expanding the concept of equivariance through category theory,
resulting in new insights on the structure of equivariant models
\citep{deHaan2020, gavranovic24a}
while others have focused on understanding
gradient-based learning using a category-theoretic framework
\citep{cruttwell2022}.
Additionally, 
recent works have examined stochastic equivariant models
\citep{cho2019, bloemreddy2020, fritz2020, cornish2024}
and causal reasoning
\citep{lorenz2023}
within the context of Markov categories.
Building on these works, we adopt a category-theoretic approach
to gain insights into the 
construction of specific neural networks; namely,
the weight matrices of certain tensor power based equivariant neural networks. 
The primary distinction and contribution of our work
lies in using this approach not only to understand the construction of the neural networks
themselves but also to significantly improve the time complexity of the forward pass 
that is involved in their execution.

Finally, in a loosely-related body of work, there is a whole subdomain
in quantum computing where formal category-theoretic languages,
such as the ZX-calculus, are used to rewrite quantum circuits 
for compilation onto quantum hardware 
\citep{coecke2017picturing, duncan2020graph, kissinger2012pictures, heunen2020categories, vandewetering2020}.
In this context, the original circuit needs to be rewritten 
in terms of a predefined set of quantum gates that is tailored to the specific
quantum hardware.
In our work, the algorithm that we present in Section \ref{algorithmsect} rewrites
the equivariant weight matrix as a composition of operations;
however, this algorithm is designed for classical computing, not quantum computing,
and we focus on the time complexity of the operations involved, rather than
the underlying hardware on which they will be executed.


\section{Background} \label{background}



In this section, we revisit the concepts
that led to the characterisation of the equivariant weight matrices
between high-order tensor power spaces of $\mathbb{R}^{n}$,
focusing on the four groups in question: $S_n$, $O(n)$, $SO(n)$ and $Sp(n)$.
This section is organised into three parts: 
first we recall the layer spaces, considered as representations of each group,
and the definition of an equivariant map between them; 
second, we review the diagrams that were used to derive the weight matrix 
characterisations; 
and finally we summarise the key results, providing references to where 
their proofs can be found in the literature.

In the following, we write $G(n)$ to refer to any of 
$S_n$, $O(n)$, $SO(n)$ and $Sp(n)$.
We consider $G(n)$ to be a subgroup of matrices in $GL(n)$ throughout,
having chosen the standard basis of $\mathbb{R}^{n}$.
We also write $[n]$ for the set $\{1, \dots, n\}$, and $[n]^p$ for the
$p$-fold Cartesian product of the set $\{1, \dots, n\}$.

\subsection{Layer Spaces as Group Representations} \label{sectLayerSpaces}

For each neural network that is equivariant to $G(n)$, 
the layer spaces are not just vector spaces, they are representations of $G(n)$.
This makes it possible to incorporate the symmetries as an inductive bias 
into the network not only through the action of the group on each layer space 
but also through the equivariant maps between them.
We define each of these concepts in turn.



Each layer space is a high-order tensor power of $\mathbb{R}^{n}$, denoted
$(\mathbb{R}^{n})^{\otimes k}$,
that comes with a representation of $G(n)$.
This representation,
written
$\rho_k: G(n) \rightarrow GL((\mathbb{R}^n)^{\otimes k})$,
is defined on the standard basis of $(\mathbb{R}^{n})^{\otimes k}$ 
\begin{equation} \label{tensorelementfirst}
	e_I \coloneqq e_{i_1} \otimes e_{i_2} \otimes \dots \otimes e_{i_k} 
\end{equation}
for all $I \coloneqq (i_1, i_2, \dots, i_k) \in [n]^k$ by
\begin{equation} 
	g \cdot e_I \coloneqq ge_{i_1} \otimes ge_{i_2} \otimes \dots \otimes ge_{i_k} 
\end{equation}
and is extended linearly on the basis elements of $(\mathbb{R}^{n})^{\otimes k}$.

We note that if $G(n) = Sp(n)$, then $n = 2m$, and we label and 
order the indices by
$1, 1', \dots, m, m'$ instead.
In this case, 
we sometimes call the standard basis of $\mathbb{R}^{n}$ the symplectic basis.

There is a special class of maps between representations of a group
that are known as equivariant maps.
Recall that a map 
$\phi : (\mathbb{R}^{n})^{\otimes k} \rightarrow (\mathbb{R}^{n})^{\otimes l}$ 
between two representations of $G(n)$ is said to be $G(n)$-equivariant if
\begin{equation} \label{Gequivmapdefn}
	\phi(\rho_{k}(g)[v]) = \rho_{l}(g)[\phi(v)]
\end{equation}
for all $g \in G(n)$ and $v \in (\mathbb{R}^{n})^{\otimes k}$.
We denote the vector space of all \textit{linear} $G(n)$-equivariant maps between 
$(\mathbb{R}^{n})^{\otimes k}$ and $(\mathbb{R}^{n})^{\otimes l}$ by 
$\Hom_{G(n)}((\mathbb{R}^{n})^{\otimes k},(\mathbb{R}^{n})^{\otimes l})$. 

We are interested in the equivariant weight matrices that map
between any two adjacent layer spaces. 
It is clear that the equivariant weight matrices mapping 
$(\mathbb{R}^{n})^{\otimes k}$ to $(\mathbb{R}^{n})^{\otimes l}$
can be found by obtaining either a basis or 
a spanning set of matrices
for 
$\Hom_{G(n)}((\mathbb{R}^{n})^{\otimes k},(\mathbb{R}^{n})^{\otimes l})$.
The diagrams that can be used to obtain such a basis or spanning set 
form the topic of the next section.

\subsection{Set Partition Diagrams} \label{sectSetPart}

There are potentially a number of different ways to obtain 
such a basis or spanning set of matrices for each group.
For example, for the symmetric group $S_n$, \citet{maron2018} considered
so-called fixed-point equations to obtain a basis of matrices 
known as the orbit basis, 
and \citet{godfrey} constructed an $\mathbb{R}$-algebra 
from the $S_n$-invariant subspace for each tensor power space of $\mathbb{R}^{n}$
to obtain a different basis of matrices known as the diagram basis.
We focus, however, on using set partition diagrams, which originally appeared in
\citet{pearcecrumpB, pearcecrump},
as they serve as the foundation for our contributions that appear
in the upcoming sections.

To introduce these diagrams, we need to begin with the definition
of a set partition.

\begin{definition}
	For any non-negative integers $l$ and $k$, consider the set
	$[l+k] \coloneqq \{1, \dots, l+k\}$. 
	A \textbf{set partition} $\pi$ of $[l+k]$ is a partition of $[l+k]$
	into a disjoint union of subsets, each of which we call a \textbf{block}.
\end{definition}

\begin{example} \label{setpartex1}
	If $l = 4$ and $k = 6$, then
	\begin{equation} 
		\{1, 2, 5, 7 \mid 3, 4, 10 \mid 6, 8 \mid 9\}
	\end{equation}
	is a set partition of $[4+6]$ having $4$ blocks.
\end{example}

\begin{definition}
	Let $\pi$ be a set partition of $[l+k]$.
	We can associate to $\pi$ a diagram $d_\pi$, 
	called a $\bm{(k,l)}$\textbf{--partition diagram}, that consists of two
	rows of vertices and edges between vertices such that there are
\begin{itemize}
	\item $l$ vertices on the top row, labelled left to right by $1, \dots, l$
	\item $k$ vertices on the bottom row, labelled left to right by $l+1, \dots, l+k$, and
	\item the edges between the vertices correspond to the connected components of $\pi$. 
\end{itemize}
As a result, $d_\pi$ represents the equivalence class of all diagrams with connected components equal to the blocks of $\pi$.
\end{definition}

\begin{example}
	The $(6,4)$--partition diagram corresponding to the set partition given
	in Example~\ref{setpartex1} is
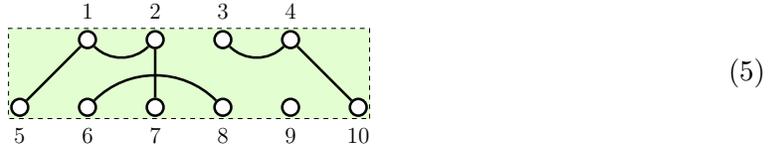
	\begin{equation}
		\begin{aligned}
			\scalebox{0.6}{\begin{tikzpicture}
	\begin{pgfonlayer}{nodelayer}
		\node [style=none] (3) at (-0.5, 0.5) {};
		\node [style=none] (4) at (-0.5, -3.5) {};
		\node [style=node] (7) at (3, 0) {};
		\node [style=node] (8) at (0, -3) {};
		\node [style=none] (13) at (15.5, -3.5) {};
		\node [style=none] (14) at (15.5, 0.5) {};
		\node [style=node] (15) at (12, 0) {};
		\node [style=node] (17) at (12, -3) {};
		\node [style=node] (22) at (3, -3) {};
		\node [style=node] (23) at (15, -3) {};
		\node [style=node] (24) at (6, 0) {};
		\node [style=node] (25) at (9, 0) {};
		\node [style=node] (26) at (6, -3) {};
		\node [style=node] (27) at (9, -3) {};
		\node [style=none] (28) at (3, 1.25) {\Large$1$};
		\node [style=none] (29) at (6, 1.25) {\Large$2$};
		\node [style=none] (30) at (9, 1.25) {\Large$3$};
		\node [style=none] (31) at (12, 1.25) {\Large$4$};
		\node [style=none] (32) at (0, -4.25) {\Large$5$};
		\node [style=none] (33) at (3, -4.25) {\Large$6$};
		\node [style=none] (34) at (6, -4.25) {\Large$7$};
		\node [style=none] (35) at (9, -4.25) {\Large$8$};
		\node [style=none] (36) at (12, -4.25) {\Large$9$};
		\node [style=none] (37) at (15, -4.25) {\Large$10$};
	\end{pgfonlayer}
	\begin{pgfonlayer}{edgelayer}
		\draw [style={dash_3}] (13.center)
			 to (14.center)
			 to (3.center)
			 to (4.center)
			 to cycle;
		\draw [style=thick] (26) to (24);
		\draw [style=thick] (8) to (7);
		\draw [style=thick, bend right=45] (7) to (24);
		\draw [style=thick, bend left=45] (22) to (27);
		\draw [style=thick] (15) to (23);
		\draw [style=thick, bend right=45] (25) to (15);
	\end{pgfonlayer}
\end{tikzpicture}}
		\end{aligned}
	\end{equation}
\end{example}

\begin{definition}
	We are also interested in the following types of $(k,l)$--partition diagrams.
\begin{itemize}
	\item A $\bm{(k,l)}$\textbf{--Brauer diagram} $d_\beta$ is a $(k,l)$--partition diagram where the size of every block in $\beta$ is exactly two.
	\item Given $k$ and $l$, an $\bm{(l+k)\backslash n}$\textbf{--diagram} $d_\alpha$ is a $(k,l)$--partition diagram where exactly $n$ blocks in $\alpha$ have size one, with the rest having exactly size two. The vertices corresponding to the blocks of size one are called \textbf{free} vertices.
\end{itemize}
\end{definition}

\begin{example} \label{exBrauerGrooddiag}
	The first diagram below is a $(7,5)$--Brauer diagram
	\begin{equation}
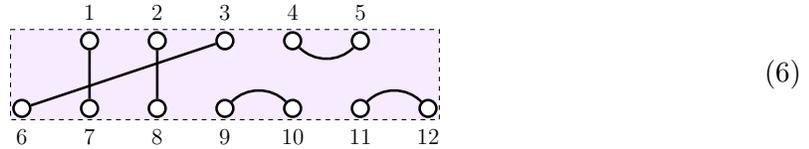
		\begin{aligned}
			\scalebox{0.6}{\begin{tikzpicture}
	\begin{pgfonlayer}{nodelayer}
		\node [style=node] (0) at (0, -1.5) {};
		\node [style=node] (1) at (12, -1.5) {};
		\node [style=node] (2) at (9, -1.5) {};
		\node [style=node] (3) at (6, -1.5) {};
		\node [style=node] (4) at (12, 1.5) {};
		\node [style=node] (5) at (6, 1.5) {};
		\node [style=node] (6) at (3, 1.5) {};
		\node [style=node] (7) at (3, -1.5) {};
		\node [style=none] (8) at (-0.5, 2) {};
		\node [style=none] (9) at (-0.5, -2) {};
		\node [style=none] (10) at (18.5, -2) {};
		\node [style=none] (11) at (18.5, 2) {};
		\node [style=node] (12) at (15, 1.5) {};
		\node [style=node] (13) at (15, -1.5) {};
		\node [style=node] (14) at (9, 1.5) {};
		\node [style=none] (15) at (3, 2.75) {\Large$1$};
		\node [style=none] (16) at (6, 2.75) {\Large$2$};
		\node [style=none] (17) at (9, 2.75) {\Large$3$};
		\node [style=none] (18) at (12, 2.75) {\Large$4$};
		\node [style=none] (19) at (15, 2.75) {\Large$5$};
		\node [style=none] (20) at (0, -2.75) {\Large$6$};
		\node [style=none] (21) at (3, -2.75) {\Large$7$};
		\node [style=none] (22) at (6, -2.75) {\Large$8$};
		\node [style=none] (23) at (9, -2.75) {\Large$9$};
		\node [style=none] (24) at (12, -2.75) {\Large$10$};
		\node [style=none] (25) at (15, -2.75) {\Large$11$};
		\node [style=node] (26) at (18, -1.5) {};
		\node [style=none] (27) at (18, -2.75) {\Large$12$};
	\end{pgfonlayer}
	\begin{pgfonlayer}{edgelayer}
		\draw [style={dash_2}] (8.center)
			 to (11.center)
			 to (10.center)
			 to (9.center)
			 to cycle;
		\draw [style=thick, bend right=45] (4) to (12);
		\draw [style=thick, bend left=45] (2) to (1);
		\draw [style=thick, bend left=45] (13) to (26);
		\draw [style=thick] (0) to (14);
		\draw [style=thick] (3) to (5);
		\draw [style=thick] (7) to (6);
	\end{pgfonlayer}
\end{tikzpicture}}
		\end{aligned}
	\end{equation}
	whereas the second is a $(5+6) \backslash 3$--diagram
	\begin{equation}
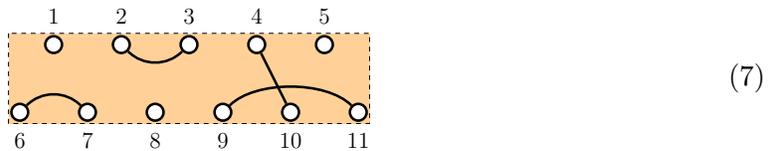
		\begin{aligned}
			\scalebox{0.6}{\begin{tikzpicture}
	\begin{pgfonlayer}{nodelayer}
		\node [style=node] (1) at (6, -1.5) {};
		\node [style=node] (3) at (12, -1.5) {};
		\node [style=node] (4) at (7.5, 1.5) {};
		\node [style=node] (5) at (4.5, 1.5) {};
		\node [style=node] (6) at (1.5, 1.5) {};
		\node [style=node] (7) at (9, -1.5) {};
		\node [style=none] (8) at (-0.5, 2) {};
		\node [style=none] (9) at (-0.5, -2) {};
		\node [style=none] (10) at (15.5, -2) {};
		\node [style=none] (11) at (15.5, 2) {};
		\node [style=node] (12) at (10.5, 1.5) {};
		\node [style=node] (13) at (15, -1.5) {};
		\node [style=node] (14) at (13.5, 1.5) {};
		\node [style=none] (15) at (1.5, 2.75) {\Large$1$};
		\node [style=none] (16) at (4.5, 2.75) {\Large$2$};
		\node [style=none] (17) at (7.5, 2.75) {\Large$3$};
		\node [style=none] (18) at (10.5, 2.75) {\Large$4$};
		\node [style=none] (19) at (13.5, 2.75) {\Large$5$};
		\node [style=none] (20) at (0, -2.75) {\Large$6$};
		\node [style=none] (21) at (3, -2.75) {\Large$7$};
		\node [style=none] (22) at (6, -2.75) {\Large$8$};
		\node [style=none] (23) at (9, -2.75) {\Large$9$};
		\node [style=none] (24) at (12, -2.75) {\Large$10$};
		\node [style=none] (25) at (15, -2.75) {\Large$11$};
		\node [style=node] (26) at (0, -1.5) {};
		\node [style=node] (27) at (3, -1.5) {};
	\end{pgfonlayer}
	\begin{pgfonlayer}{edgelayer}
		\draw [style={dash_4}] (8.center)
			 to (9.center)
			 to (10.center)
			 to (11.center)
			 to cycle;
		\draw [style=thick, bend left=315] (27) to (26);
		\draw [style=thick, bend right=45] (5) to (4);
		\draw [style=thick] (12) to (3);
		\draw [style=thick, bend left=315, looseness=0.75] (13) to (7);
	\end{pgfonlayer}
\end{tikzpicture}}
		\end{aligned}
	\end{equation}
\end{example}

In order to state the characterisations more succinctly
in the next section, we introduce
a number of vector spaces as the formal $\mathbb{R}$--linear span of certain subsets 
of $(k,l)$--partition diagrams, as follows.
\begin{definition} \label{partvecspace}
	\hphantom{x}
\begin{itemize}
	\item The \textbf{partition vector space} $P_k^l(n)$ is defined to be the $\mathbb{R}$--linear span of the set of all $(k,l)$--partition diagrams.
	\item The \textbf{Brauer vector space} $B_k^l(n)$ is defined to be the $\mathbb{R}$--linear span of the set of all $(k,l)$--Brauer diagrams.
	\item The \textbf{Brauer--Grood vector space} $D_k^l(n)$ is defined to be the $\mathbb{R}$--linear span of the set of all $(k,l)$--Brauer diagrams together with the set of all $(l+k)\backslash n$--diagrams.
\end{itemize}
\end{definition}

\subsection{Characterisation of Equivariant Weight Matrices} \label{sectCharacterisation}

We now use the vector spaces given in Definition \ref{partvecspace}
to summarise the equivariant weight matrix characterisation results 
for each of the four groups.
We state each basis/spanning set result as a theorem, and give its associated
equivariant weight matrix characterisation as a corollary.

In order to state the results, recall that, 
for all non-negative integers $l, k$,
the vector space of matrices
$\Hom((\mathbb{R}^{n})^{\otimes k}, (\mathbb{R}^{n})^{\otimes l})$
has a standard basis of matrix units
\begin{equation} \label{standardbasisunits}
	\{E_{I,J}\}_{I \in [n]^l, J \in [n]^k}
\end{equation}
where $I$ is a tuple $(i_1, i_2, \dots, i_l) \in [n]^l$, 
$J$ is a tuple $(j_1, j_2, \dots, j_k) \in [n]^k$
and $E_{I,J}$ has a $1$ in the $(I,J)$ position and is $0$ elsewhere.

\begin{theorem}[Diagram Basis when $G(n) = S_n$] \cite[Theorem 5.4]{godfrey}
	\label{diagbasisSn}

	For all non-negative integers $l, k$ and any positive integer $n$,
	there is a surjection of vector spaces
	\begin{equation} \label{surjectionSn}
		\Theta_{k,n}^l : P_k^l(n) \rightarrow 
		\Hom_{S_n}((\mathbb{R}^{n})^{\otimes k}, (\mathbb{R}^{n})^{\otimes l})
	\end{equation}
	that is given by
	\begin{equation}
		d_\pi \mapsto D_\pi
	\end{equation}
	for all $(k,l)$--partition diagrams $d_\pi$, 
	where $D_\pi$ is defined as follows.
	If $S_\pi((I,J))$ is defined to be the set
	\begin{equation} \label{Snindexingset}
		\left\{(I,J) \in [n]^{l+k} \mid \text{if } x,y \text{ are in the same block of } \pi, \text{then } i_x = i_y \right\}
	\end{equation}
	(where we have momentarily replaced the elements of $J$ by $i_{l+m} \coloneqq j_m$ for all $m \in [k]$),
	then we have that
	\begin{equation} \label{mappeddiagbasisSn}
		D_\pi
		\coloneqq
		\sum_{I \in [n]^l, J \in [n]^k}
		\delta_{\pi, (I,J)}
		E_{I,J}
	\end{equation}
	where
	\begin{equation}
		\delta_{\pi, (I,J)}
		\coloneqq
		\begin{cases}
			1 & \text{if } (I,J) \in S_\pi((I,J)) \\
			0 & \text{otherwise}
		\end{cases}
	\end{equation}	
	In particular, the set
	\begin{equation} \label{klSnSpanningSet}
		\{D_\pi \mid d_\pi \text{ is a } (k,l) \text{--partition diagram having at most } n \text{ blocks} \}
	\end{equation}
	is a basis for
	$\Hom_{S_n}((\mathbb{R}^{n})^{\otimes k}, (\mathbb{R}^{n})^{\otimes l})$
	in the standard basis of $\mathbb{R}^{n}$, 
	of size $\Bell(l+k,n) \coloneqq
	\sum_{t=1}^{n} 
		\begin{Bsmallmatrix}
		l+k\\
		t 
		\end{Bsmallmatrix}
	$, where 
	$
		\begin{Bsmallmatrix}
		l+k\\
		t 
		\end{Bsmallmatrix}
	$
	is the Stirling number of the second kind.
\end{theorem}

\begin{corollary}[Permutation Equivariant Weight Matrices] 
	\label{diagweightmatclass}
	For all non-negative integers $l, k$ and positive integers $n$, 
	the weight matrix $W$ that appears in
	an $S_n$-equivariant linear layer function
	from $(\mathbb{R}^{n})^{\otimes k}$ to $(\mathbb{R}^{n})^{\otimes l}$
	must be of the form
	\begin{equation}
		W = \sum_{d_\pi \in P_k^l(n)} \lambda_\pi{D_\pi}
	\end{equation}
\end{corollary}


\begin{theorem}
	[Spanning set when $G(n) = O(n)$]
	\cite[Theorem 6.5]{pearcecrumpB}
	\label{spanningsetO(n)}

	For all non-negative integers $l, k$ and any positive integer $n$,
	there is a surjection of vector spaces
	\begin{equation} \label{surjectionO(n)}
		\Phi_{k,n}^l : B_k^l(n) \rightarrow 
		\Hom_{O(n)}((\mathbb{R}^{n})^{\otimes k}, (\mathbb{R}^{n})^{\otimes l})
	\end{equation}
	that is given by
	\begin{equation}
		d_\beta \mapsto E_\beta
	\end{equation}
	for all $(k,l)$--Brauer diagrams $d_\beta$, where 
	$E_\beta \coloneqq D_\beta$ 
	is given in (\ref{mappeddiagbasisSn}).
	\noindent	
	In particular, the set
	\begin{equation} \label{klOnSpanningSet}
		\{E_\beta \mid d_\beta \text{ is a } (k,l) \text{--Brauer diagram} \}
	\end{equation}
	is a spanning set for
	$\Hom_{O(n)}((\mathbb{R}^{n})^{\otimes k}, (\mathbb{R}^{n})^{\otimes l})$
	in the standard basis of $\mathbb{R}^{n}$, 
	of size $0$ when $l+k$ is odd, and of size $(l+k-1)!!$ when $l+k$ is even.
\end{theorem}

\begin{corollary}[Orthogonal Group Equivariant Weight Matrices] \label{Onweightmatclass}
	For all non-negative integers $l, k$ and positive integers $n$, 
	the weight matrix $W$ that appears in
	an $O(n)$-equivariant linear layer function
	from $(\mathbb{R}^{n})^{\otimes k}$ to $(\mathbb{R}^{n})^{\otimes l}$
	must be of the form
	\begin{equation}
		W = \sum_{d_\beta \in B_k^l(n)} \lambda_\beta{E_\beta}
	\end{equation}
\end{corollary}

\begin{theorem} 
	[Spanning set when $G(n) = Sp(n), n = 2m$]
	\cite[Theorem 6.6]{pearcecrumpB}
	\label{spanningsetSp(n)}

	For all non-negative integers $l, k$ and any positive integer $n$
	such that $n = 2m$, there is a surjection of vector spaces
	\begin{equation} \label{surjectionSp(n)}
		X_{k,n}^l : B_k^l(n) \rightarrow 
		\Hom_{Sp(n)}((\mathbb{R}^{n})^{\otimes k}, (\mathbb{R}^{n})^{\otimes l})
	\end{equation}
	that is given by
	\begin{equation}
		d_\beta \mapsto F_\beta
	\end{equation}
	for all $(k,l)$--Brauer diagrams $d_\beta$, where $F_\beta$ is 
	defined as follows.

	Associate the indices $i_1, i_2, \dots, i_l$ with the vertices in the top row of $d_\beta$, and $j_1, j_2, \dots, j_k$ with the vertices in the bottom row of $d_\beta$.
	Then, we have that
	\begin{equation} \label{matrixSp(n)}
		F_\beta 
		\coloneqq
		\sum_{I, J} 
		\gamma_{r_1, u_1}
		\gamma_{r_2, u_2}
		\dots
		\gamma_{r_{\frac{l+k}{2}}, u_{\frac{l+k}{2}}}
		E_{I,J}
	\end{equation}
	where the indices $i_p, j_p$ range over $1, 1', \dots, m, m'$,
	where $r_1, u_1, \dots, r_{\frac{l+k}{2}}, u_{\frac{l+k}{2}}$ is any permutation of the indices $i_1, i_2, \dots, i_l, j_1, j_2, \dots, j_k$ such that the vertices corresponding to
	$r_p, u_p$ 
	are in the same block of $\beta$, and
	\begin{equation} \label{gammarpup}
		\gamma_{r_p, u_p} \coloneqq
		\begin{cases}
			\delta_{r_p, u_p} & \text{if the vertices corresponding to } r_p, u_p \text{ are in different rows of } d_\beta \\
			\epsilon_{r_p, u_p} & \text{if the vertices corresponding to } r_p, u_p \text{ are in the same row of } d_\beta
    		\end{cases}
	\end{equation}
	Here, $\epsilon_{r_p, u_p}$ is given by
\begin{equation} \label{epsilondef1}
	\epsilon_{\alpha, \beta} = \epsilon_{{\alpha'}, {\beta'}} = 0
\end{equation}
\begin{equation} \label{epsilondef2}
	\epsilon_{\alpha, {\beta'}} = - \epsilon_{{\alpha'}, {\beta}} = \delta_{\alpha, \beta}
\end{equation}

	\noindent
	In particular, the set
	\begin{equation} \label{klSpnSpanningSet}
		\{F_\beta \mid d_\beta \text{ is a } (k,l) \text{--Brauer diagram} \}
	\end{equation}
	is a spanning set for
	$\Hom_{Sp(n)}((\mathbb{R}^{n})^{\otimes k}, (\mathbb{R}^{n})^{\otimes l})$, for $n = 2m$,
	in the symplectic basis of $\mathbb{R}^{n}$,
	of size $0$ when $l+k$ is odd, and of size $(l+k-1)!!$ when $l+k$ is even.
\end{theorem}

\begin{corollary}[Symplectic Group Equivariant Weight Matrices] \label{Spnweightmatclass}
	For all non-negative 
	integers $l, k$ and positive even integers $n$, 
	the weight matrix $W$ that appears in
	an $Sp(n)$-equivariant linear layer function
	from $(\mathbb{R}^{n})^{\otimes k}$ to $(\mathbb{R}^{n})^{\otimes l}$
	must be of the form
	\begin{equation}
		W = \sum_{d_\beta \in B_k^l(n)} \lambda_\beta{F_\beta}
	\end{equation}
\end{corollary}


\begin{theorem} 
	[Spanning set when $G(n) = SO(n)$]
	\cite[Theorem 6.7]{pearcecrumpB}
	\label{spanningsetSO(n)}

	For all non-negative integers $l, k$ and positive even integers $n$,
	there is a surjection of vector spaces
	\begin{equation} \label{surjectionSO(n)}
		\Psi_{k,n}^l : D_k^l(n) \rightarrow 
		\Hom_{SO(n)}((\mathbb{R}^{n})^{\otimes k}, (\mathbb{R}^{n})^{\otimes l})
	\end{equation}
	that is given by
	\begin{equation}
		d_\beta \mapsto E_\beta
	\end{equation}
	if $d_\beta$ is a $(k,l)$--Brauer diagram, where $E_\beta$
	was defined in Theorem \ref{spanningsetO(n)},
	and by
	\begin{equation} \label{surjdalpha}
		d_\alpha \mapsto H_\alpha
	\end{equation}
	if $d_\alpha$ is a $(k+l)\backslash n$--diagram, where $H_\alpha$
	is defined as follows.

	Associate 
the indices $i_1, i_2, \dots, i_l$ with the vertices in the top row of $d_\alpha$, and $j_1, j_2, \dots, j_k$ with the vertices in the bottom row of $d_\alpha$.
	Suppose that there are $s$ free vertices in the top row. Then there are $n-s$ free vertices in the bottom row.
	Relabel the $s$ free indices in the top row from left-to-right by 
	$t_1, \dots, t_s$, and the $n-s$ free indices in the bottom row 
	from left-to-right by $b_1, \dots, b_{n-s}$. 
	Then, we have that
	\begin{equation} \label{SO(n)Halpha}
		H_\alpha \coloneqq
		\sum_{I \in [n]^l, J \in [n]^k} 
		\det(e_{T,B})
		\delta(R,U)
		E_{I,J}
	\end{equation}
	where $\det(e_{T,B})$ is the determinant of the $n \times n$ matrix
	\begin{equation} \label{detmapdefn}
		\lvert e_{t_1} \; e_{t_2} \; \cdots \; e_{t_s} \;
		 e_{b_1} \; \cdots \;  e_{b_{n-s}}\rvert
	\end{equation}
	and
	\begin{equation}
	\delta(R,U) \coloneqq 
		\delta_{r_1, u_1}
		\delta_{r_2, u_2}
		\dots
		\delta_{r_{\frac{l+k-n}{2}}, u_{\frac{l+k-n}{2}}}
	\end{equation}
	Here, $r_1, u_1, \dots, r_{\frac{l+k-n}{2}}, u_{\frac{l+k-n}{2}}$
	is any permutation of the indices 
	\begin{equation}	
	\{i_1, \dots, i_l, j_1, \dots, j_k\} \backslash \{t_1, \dots, t_s, b_1, \dots, b_{n-s}\}
	\end{equation}
	such that the vertices corresponding to $r_p, u_p$ are in the same block of $\alpha$. 

	In particular, the set
	\begin{equation} \label{SOn2SpanningSet}
		\{E_\beta\}_{\beta}
		\cup
		\{H_\alpha\}_{\alpha}
	\end{equation}
	where $d_\beta$ is a $(k,l)$--Brauer diagram, and $d_\alpha$ is a $(l+k) \backslash n$--diagram,
	is a spanning set for
	$\Hom_{SO(n)}((\mathbb{R}^{n})^{\otimes k}, (\mathbb{R}^{n})^{\otimes l})$
	in the standard basis of $\mathbb{R}^{n}$.
\end{theorem}

\begin{corollary}[Special Orthogonal Group Equivariant Weight Matrices] \label{SOnweightmatclass}
	For all \allowbreak{} non-negative integers 
	$l, k$ and positive integers $n$, 
	the weight matrix $W$ that appears in
	an $SO(n)$-equivariant linear layer function
	from $(\mathbb{R}^{n})^{\otimes k}$ to $(\mathbb{R}^{n})^{\otimes l}$
	must be of the form
	\begin{equation}
		W = 
		\sum_{d_\beta \in D_k^l(n)} \lambda_\beta{E_\beta}
		+	
		\sum_{d_\alpha \in D_k^l(n)} \lambda_\alpha{H_\alpha}
	\end{equation}
\end{corollary}



\section{Equivariant Weight Matrices from Categories} \label{weightmatcat}

In the previous section,
we recalled the linear maps that can be used to obtain the equivariant
weight matrices between tensor power spaces of $\mathbb{R}^{n}$
for each of the four groups in question.
To motivate the content of this section, we note that,
for a given group $G(n)$, 
each of the weight matrices 
is in some sense similar, since they are determined 
in part by the layer spaces,
which are tensor power representations 
that differ only in their order.
Likewise, the vector spaces containing certain types of set partition diagrams
are also, in some sense, similar, in that, for a fixed value of $n$, 
the vector spaces
mostly differ by the values chosen for $l$ and $k$, 
which represent the number of vertices
in each row of a set partition diagram.
We would like to formalise this intuition 
of similarity across different values of $l$ and $k$
for both the vector spaces of equivariant linear maps
and the vector spaces containing certain types of set partition diagrams.
For this, the concepts that appear in category theory are ideal,
in that they are useful for abstracting away the specifics of structures 
in order to study their general properties.
From the vector spaces that have similar properties, we create so-called
monoidal categories, which are categories that have an additional operation
for composing objects and morphisms,
known as the tensor product,
and then use them to create so-called monoidal functors, 
which are functors that preserve the tensor product between monoidal categories.
We are particularly interested in the tensor product 
because the layer spaces
are tensor power spaces of $\mathbb{R}^{n}$, and they differ only
in the number of tensor products that are involved.

We wish to emphasise that we are not simply rewriting
existing results in a different language, but that, 
as an outcome of this process,
we obtain new insights into the equivariant neural networks for each group.
In particular, we show that set partition diagrams that appear in monoidal
categories have a string-like quality to them.
By pulling on the strings or dragging their ends to different locations, 
we can use the monoidal functors to obtain new results
that are relevant
for the weight matrices that appear in these group equivariant neural networks.
We summarise these new results at the end of this section.

\subsection{Strict $\mathbb{R}$--Linear Monoidal Categories and String Diagrams}

We begin by introducing categories and functors that are \textbf{monoidal}.
The monoidal property gives additional structure to the way 
in which objects and morphisms can be related.
In particular, monoidal categories have an additional operation, 
known as a \textbf{bifunctor}, or a \textbf{tensor product}, 
that enables objects and morphisms to be composed in a second way,
relative to the usual definition of composition that exists in every category.
We often call this additional composition \textbf{horizontal},
in contrast to the \textbf{vertical} composition that comes with any category. 
Crucially, monoidal functors preserve the tensor product 
across monoidal categories.

In this section, we assume knowledge of
the definition of a category 
and the definition of a functor between categories.
These definitions can be found in an introductory text on category theory,
such as \citet{leinster2014, maclane} or \citet{riehl2017}.



We assume throughout that all categories are \textbf{locally small}, which
means that the collection of morphisms between any two objects is a set.
In fact, all of the categories that we consider going forward 
have morphism sets that are vector spaces.
Hence, the morphisms between objects are actually linear maps.
For the definitions in this section, we follow the presentation that is given in 
\citet{Hu2019} and \citet{Savage2021}.


\begin{definition} \label{categorystrictmonoidal}
	A category $\mathcal{C}$ is said to be \textbf{strict monoidal} if it comes with
		a functor $\otimes: \mathcal{C} \times \mathcal{C} \rightarrow \mathcal{C}$, 
		known as a \textbf{bifunctor} or a \textbf{tensor product}, and
	a unit object $\mathds{1}$,
	such that, for all objects $X, Y, Z$ in $\mathcal{C}$, we have that
	\begin{equation}
		(X \otimes Y) \otimes Z = X \otimes (Y \otimes Z)
	\end{equation}
	\begin{equation}
		(\mathds{1} \otimes X) = X = (X \otimes \mathds{1})
	\end{equation}
	and, for all morphisms $f, g, h$ in $\mathcal{C}$, we have that
	\begin{equation} \label{assocbifunctor}
		(f \otimes g) \otimes h = f \otimes (g \otimes h)
	\end{equation}
	\begin{equation}
		(1_\mathds{1} \otimes f) = f = (f \otimes 1_\mathds{1})
	\end{equation}
	where $1_\mathds{1}$ is the identity morphism $\mathds{1} \rightarrow \mathds{1}$.
	We often use the tuple
	$(\mathcal{C}, \otimes_\mathcal{C}, \mathds{1}_\mathcal{C})$
	to refer to the strict monoidal category $\mathcal{C}$.
\end{definition}

We have started with the definition of a \textit{strict} monoidal category as opposed to
that of a monoidal category 
since we can assume that all
monoidal categories are strict, 
owing to a technical result known as Mac Lane's Coherence Theorem. 
See \citet{maclane} for more details.

\begin{definition} \label{categorylinear}
	A category $\mathcal{C}$ is said to be $\mathbb{R}$\textbf{--linear} if,
	for any two objects $X, Y$ in $\mathcal{C}$, the morphism space 
	$\Hom_{\mathcal{C}}(X,Y)$ is a vector space over $\mathbb{R}$, and
	the composition of morphisms is $\mathbb{R}$--bilinear.
\end{definition}

Combining Definitions \ref{categorystrictmonoidal} and \ref{categorylinear},
we get
\begin{definition}
	A category $\mathcal{C}$ is said to be \textbf{strict} $\mathbb{R}$\textbf{--linear monoidal} if it is a category that is both strict monoidal and $\mathbb{R}$--linear, such that the bifunctor $\otimes$ is $\mathbb{R}$--bilinear.
\end{definition}


We now consider functors between 
strict $\mathbb{R}$--linear monoidal categories 
that preserve the tensor product.
\begin{definition} \label{monoidalfunctordefn}
	Suppose that
	$(\mathcal{C}, \otimes_\mathcal{C}, \mathds{1}_\mathcal{C})$
	and
	$(\mathcal{D}, \otimes_\mathcal{D}, \mathds{1}_\mathcal{D})$
	are two strict $\mathbb{R}$--linear monoidal categories.

	A \textbf{strict} $\mathbb{R}$\textbf{--linear monoidal functor} 
	from $\mathcal{C}$ to $\mathcal{D}$ is a functor 
	$\mathcal{F}: \mathcal{C} \rightarrow \mathcal{D}$ 
	such that
	\begin{enumerate}
		\item for all objects $X, Y$ in $\mathcal{C}$,
			$\mathcal{F}(X \otimes_\mathcal{C} Y) =
			\mathcal{F}(X) \otimes_\mathcal{D} \mathcal{F}(Y)$
		\item for all morphisms $f, g$ in $\mathcal{C}$,
			$\mathcal{F}(f \otimes_\mathcal{C} g) =
			\mathcal{F}(f) \otimes_\mathcal{D} \mathcal{F}(g)$
		\item $\mathcal{F}(\mathds{1}_\mathcal{C}) = \mathds{1}_\mathcal{D}$, and
		\item for all objects $X, Y$ in $\mathcal{C}$, the map
		\begin{equation} \label{maphomsets}
			\Hom_{\mathcal{C}}(X,Y) 
			\rightarrow 
			\Hom_{\mathcal{D}}(\mathcal{F}(X),\mathcal{F}(Y))
		\end{equation}
		given by
		$f \mapsto \mathcal{F}(f)$
		is $\mathbb{R}$--linear.
	\end{enumerate}
\end{definition}


Strict monoidal categories are particularly interesting because 
they can be represented by a very useful diagrammatic language 
known as \textbf{string diagrams}.
We will see that, 
as this language is, in some sense, geometric in nature,
it is much easier to work with these diagrams 
compared with their equivalent 
algebraic form.

\begin{definition} [String Diagrams]
	Suppose that $\mathcal{C}$ is a strict monoidal category.
	Let $W, X, Y$ and $Z$ be objects in $\mathcal{C}$, and let
$f: X \rightarrow Y$, $g: Y \rightarrow Z$, and $h: W \rightarrow Z$ be morphisms in $\mathcal{C}$.
	Then we can represent the morphisms
	$1_{X}: X \rightarrow X$,
	$f: X \rightarrow Y$,
	$g \circ f: X \rightarrow Z$ and
	$f \otimes h: X \otimes W \rightarrow Y \otimes Z$
	as string diagrams in the following way:
	\begin{equation}
		\begin{aligned}
		\scalebox{0.75}{\begin{tikzpicture}
	\begin{pgfonlayer}{nodelayer}
		\node [style=none] (1) at (0, 3) {};
		\node [style=none] (2) at (0, -4) {};
		\node [style=none] (3) at (2, -4) {};
		\node [style=none] (4) at (2, 3) {};
		\node [style=none] (36) at (6, 3) {};
		\node [style=none] (37) at (6, -4) {};
		\node [style=none] (38) at (9, -4) {};
		\node [style=none] (39) at (9, 3) {};
		\node [style=large box] (42) at (7.5, -0.5) {$f$};
		\node [style=none] (43) at (13, 3) {};
		\node [style=none] (44) at (13, -4) {};
		\node [style=none] (45) at (16, -4) {};
		\node [style=none] (46) at (16, 3) {};
		\node [style=large box] (49) at (14.5, -1.25) {$f$};
		\node [style=large box] (50) at (14.5, 0.25) {$g$};
		\node [style=none] (51) at (20, 3) {};
		\node [style=none] (52) at (20, -4) {};
		\node [style=none] (53) at (26, -4) {};
		\node [style=none] (54) at (26, 3) {};
		\node [style=large box] (57) at (21.5, -0.5) {$f$};
		\node [style=large box] (58) at (24.5, -0.5) {$h$};
		\node [style=none] (62) at (1, 2) {};
		\node [style=none] (63) at (1, -3) {};
		\node [style=none] (64) at (1, -3.5) {$X$};
		\node [style=none] (65) at (1, 2.5) {$X$};
		\node [style=none] (66) at (7.5, 2.5) {$Y$};
		\node [style=none] (67) at (7.5, -3.5) {$X$};
		\node [style=none] (68) at (7.5, 2) {};
		\node [style=none] (69) at (7.5, -3) {};
		\node [style=none] (70) at (14.5, 2.5) {$Z$};
		\node [style=none] (71) at (14.5, -3.5) {$X$};
		\node [style=none] (72) at (14.5, 2) {};
		\node [style=none] (73) at (14.5, -3) {};
		\node [style=none] (74) at (24.5, 2.5) {$Z$};
		\node [style=none] (75) at (21.5, -3.5) {$X$};
		\node [style=none] (76) at (24.5, -3.5) {$W$};
		\node [style=none] (77) at (21.5, 2.5) {$Y$};
		\node [style=none] (78) at (21.5, 2) {};
		\node [style=none] (79) at (21.5, -3) {};
		\node [style=none] (80) at (24.5, 2) {};
		\node [style=none] (81) at (24.5, -3) {};
	\end{pgfonlayer}
	\begin{pgfonlayer}{edgelayer}
		\draw [style={dash_2}] (2.center) to (3.center);
		\draw [style={dash_2}] (3.center) to (4.center);
		\draw [style={dash_2}] (4.center) to (1.center);
		\draw [style={dash_2}] (1.center) to (2.center);
		\draw [style={dash_2}] (37.center) to (38.center);
		\draw [style={dash_2}] (38.center) to (39.center);
		\draw [style={dash_2}] (39.center) to (36.center);
		\draw [style={dash_2}] (36.center) to (37.center);
		\draw [style={dash_2}] (46.center) to (43.center);
		\draw [style={dash_2}] (43.center) to (44.center);
		\draw [style={dash_2}] (44.center) to (45.center);
		\draw [style={dash_2}] (45.center) to (46.center);
		\draw [style={dash_2}] (54.center) to (51.center);
		\draw [style={dash_2}] (51.center) to (52.center);
		\draw [style={dash_2}] (52.center) to (53.center);
		\draw [style={dash_2}] (53.center) to (54.center);
		\draw [style=thick] (62.center) to (63.center);
		\draw [style=thick] (69.center) to (68.center);
		\draw [style=thick] (73.center) to (72.center);
		\draw [style=thick] (79.center) to (78.center);
		\draw [style=thick] (81.center) to (80.center);
	\end{pgfonlayer}
\end{tikzpicture}}
		\end{aligned}
	\end{equation}
	In particular, the vertical composition of morphisms $g \circ f$ is obtained by placing $g$ above $f$, and the horizontal composition of morphisms $f \otimes h$ is obtained by horizontally placing $f$ to the left of $h$.

	We will often omit the labelling of the objects when they are clear or when they are not important.
\end{definition}

As an example of how useful string diagrams are when working with strict monoidal categories, the associativity of the bifunctor given in (\ref{assocbifunctor}) becomes immediately apparent.
Another, more involved, example is given by the interchange law that exists for any strict monoidal category. It can be expressed algebraically as
\begin{equation} \label{interchange}
	(\mathds{1} \otimes g) \circ (f \otimes \mathds{1})
	=
	f \otimes g
	=
	(f \otimes \mathds{1}) \circ (\mathds{1} \otimes g) 
\end{equation}
Without string diagrams, it is somewhat tedious to prove this result.
Indeed, for the left hand equality of (\ref{interchange}), we have that
\begin{equation}
	(\mathds{1} \otimes g) \circ (f \otimes \mathds{1})	
	=
	(\otimes (\mathds{1}, g)) 
	\circ 
	(\otimes (f, \mathds{1})) 
	=
	\otimes((\mathds{1}, g) \circ (f, \mathds{1}))
	=
	\otimes((f, g))
	=
	f \otimes g
\end{equation}
where we have used the definition of composition in 
$\mathcal{C} \times \mathcal{C}$ and the fact that
$\otimes : \mathcal{C} \times \mathcal{C} \rightarrow \mathcal{C}$ is a
functor.
A similar calculation also holds for the right hand equality of (\ref{interchange}).

However, with string diagram, the result is intuitively obvious, if we allow ourselves to deform the diagrams by pulling on the strings:
\begin{equation}
	\begin{aligned}
		\scalebox{0.75}{\begin{tikzpicture}
	\begin{pgfonlayer}{nodelayer}
		\node [style=none] (0) at (0, 3.5) {};
		\node [style=none] (1) at (0, -3.5) {};
		\node [style=none] (2) at (6, -3.5) {};
		\node [style=none] (3) at (6, 3.5) {};
		\node [style=large box] (6) at (1.5, -1.25) {$f$};
		\node [style=large box] (7) at (4.5, 1.25) {$g$};
		\node [style=none] (10) at (9, 3.5) {};
		\node [style=none] (11) at (9, -3.5) {};
		\node [style=none] (12) at (15, -3.5) {};
		\node [style=none] (13) at (15, 3.5) {};
		\node [style=large box] (16) at (10.5, 0) {$f$};
		\node [style=large box] (17) at (13.5, 0) {$g$};
		\node [style=none] (20) at (7.5, 0) {$=$};
		\node [style=none] (21) at (18, 3.5) {};
		\node [style=none] (22) at (18, -3.5) {};
		\node [style=none] (23) at (24, -3.5) {};
		\node [style=none] (24) at (24, 3.5) {};
		\node [style=large box] (27) at (19.5, 1.25) {$f$};
		\node [style=large box] (28) at (22.5, -1.25) {$g$};
		\node [style=none] (31) at (16.5, 0) {$=$};
		\node [style=none] (36) at (1.5, 2.5) {};
		\node [style=none] (37) at (1.5, -2.5) {};
		\node [style=none] (38) at (22.5, 2.5) {};
		\node [style=none] (39) at (22.5, -2.5) {};
		\node [style=none] (40) at (19.5, 2.5) {};
		\node [style=none] (41) at (19.5, -2.5) {};
		\node [style=none] (42) at (13.5, 2.5) {};
		\node [style=none] (43) at (13.5, -2.5) {};
		\node [style=none] (44) at (10.5, 2.5) {};
		\node [style=none] (45) at (10.5, -2.5) {};
		\node [style=none] (46) at (4.5, 2.5) {};
		\node [style=none] (47) at (4.5, -2.5) {};
	\end{pgfonlayer}
	\begin{pgfonlayer}{edgelayer}
		\draw [style={dash_2}] (3.center) to (0.center);
		\draw [style={dash_2}] (0.center) to (1.center);
		\draw [style={dash_2}] (1.center) to (2.center);
		\draw [style={dash_2}] (2.center) to (3.center);
		\draw [style={dash_2}] (13.center) to (10.center);
		\draw [style={dash_2}] (10.center) to (11.center);
		\draw [style={dash_2}] (11.center) to (12.center);
		\draw [style={dash_2}] (12.center) to (13.center);
		\draw [style={dash_2}] (24.center) to (21.center);
		\draw [style={dash_2}] (21.center) to (22.center);
		\draw [style={dash_2}] (22.center) to (23.center);
		\draw [style={dash_2}] (23.center) to (24.center);
		\draw [style=thick] (37.center) to (36.center);
		\draw [style=thick] (39.center) to (38.center);
		\draw [style=thick] (41.center) to (40.center);
		\draw [style=thick] (43.center) to (42.center);
		\draw [style=thick] (45.center) to (44.center);
		\draw [style=thick] (47.center) to (46.center);
	\end{pgfonlayer}
\end{tikzpicture}}
	\end{aligned}
\end{equation}


\subsection{Categorification}

Previously,
we defined a vector space for each non-negative integer $l$ and $k$
that is the $\mathbb{R}$--linear span of a 
certain subset of $(k,l)$--partition diagrams. 
However, it is clear that, for all values of $l$ and $k$,
these vector spaces are all similar in nature, in that the set partition diagrams only differ by the number of vertices that appear in each row and by the connections that are made between vertices.
Moreover, 
set partition diagrams
look like string diagrams.
Given that string diagrams represent strict monoidal categories, and that we have
a collection of vector spaces 
for certain subsets of set partition diagrams, this implies that we should have a number of strict $\mathbb{R}$--linear monoidal categories.
We formalise this intuition below.

\subsubsection{Partition Categories}

We assume throughout that $l$ and $k$ are non-negative integers 
and that $n$ is a positive integer.
In Section \ref{sectSetPart},
we defined the partition vector space $P_k^l(n)$,
which has a basis consisting of all possible $(k,l)$--partition diagrams.
In order to obtain a category from these vector spaces, we need to define the following two 
$\mathbb{R}$--bilinear operations on $(k,l)$--partition diagrams:
\begin{align}
	\text{composition:   } & \bullet: P_l^m(n) \times P_k^l(n) \rightarrow P_k^m(n) \label{compositionstatement} \\
	\text{tensor product:   } & \otimes: P_k^l(n) \times P_q^m(n) \rightarrow P_{k+q}^{l+m}(n) \label{tensorprodstatement}
\end{align}
They are given as follows:

\begin{definition}[Composition] \label{composition}
	Let $d_{\pi_1} \in P_k^l(n)$ and $d_{\pi_2} \in P_l^m(n)$. 
		First, we concatenate the diagrams, written $d_{\pi_2} \circ d_{\pi_1}$, by putting $d_{\pi_1}$ below $d_{\pi_2}$, concatenating the edges in the middle row of vertices, and then removing all connected components that lie entirely in the middle row of the concatenated diagrams. 
Let $c(d_{\pi_2}, d_{\pi_1})$ be the number of connected components that are removed from the middle row in $d_{\pi_2} \circ d_{\pi_1}$.
		Then the composition is defined, using infix notation, as
\begin{equation} 
	d_{\pi_2} \bullet d_{\pi_1} 
	\coloneqq 
	n^{c(d_{\pi_2}, d_{\pi_1})} (d_{\pi_2} \circ d_{\pi_1})
\end{equation}
\end{definition}

\begin{example} \label{partcomp}
	If $d_{\pi_2}$ is the 
	$(6,4)$--partition diagram 
	\begin{equation}
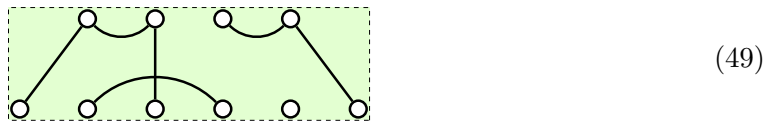
	\begin{aligned}
		\scalebox{0.6}{\begin{tikzpicture}
	\begin{pgfonlayer}{nodelayer}
		\node [style=none] (3) at (-0.5, 0.5) {};
		\node [style=none] (4) at (-0.5, -4.5) {};
		\node [style=node] (7) at (3, 0) {};
		\node [style=node] (8) at (0, -4) {};
		\node [style=none] (13) at (15.5, -4.5) {};
		\node [style=none] (14) at (15.5, 0.5) {};
		\node [style=node] (15) at (12, 0) {};
		\node [style=node] (17) at (12, -4) {};
		\node [style=node] (22) at (3, -4) {};
		\node [style=node] (23) at (15, -4) {};
		\node [style=node] (24) at (6, 0) {};
		\node [style=node] (25) at (9, 0) {};
		\node [style=node] (26) at (6, -4) {};
		\node [style=node] (27) at (9, -4) {};
	\end{pgfonlayer}
	\begin{pgfonlayer}{edgelayer}
		\draw [style={dash_3}] (13.center)
			 to (14.center)
			 to (3.center)
			 to (4.center)
			 to cycle;
		\draw [style=thick] (26) to (24);
		\draw [style=thick] (8) to (7);
		\draw [style=thick, bend right=45] (7) to (24);
		\draw [style=thick, bend left=45] (22) to (27);
		\draw [style=thick] (15) to (23);
		\draw [style=thick, bend right=45] (25) to (15);
	\end{pgfonlayer}
\end{tikzpicture}}
	\end{aligned}
	\end{equation}
	and
	$d_{\pi_1}$ is the 
	$(3,6)$--partition diagram 
	\begin{equation}
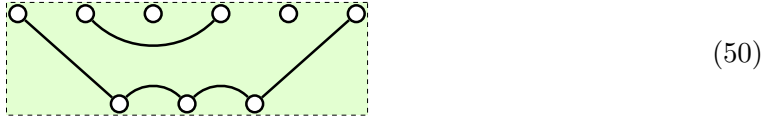
	\begin{aligned}
		\scalebox{0.6}{\begin{tikzpicture}
	\begin{pgfonlayer}{nodelayer}
		\node [style=none] (3) at (-0.5, 0.5) {};
		\node [style=none] (4) at (-0.5, -4.5) {};
		\node [style=node] (5) at (4.5, -4) {};
		\node [style=node] (6) at (0, 0) {};
		\node [style=none] (9) at (15.5, -4.5) {};
		\node [style=none] (10) at (15.5, 0.5) {};
		\node [style=node] (11) at (9, 0) {};
		\node [style=node] (13) at (12, 0) {};
		\node [style=node] (17) at (3, 0) {};
		\node [style=node] (18) at (10.5, -4) {};
		\node [style=node] (20) at (15, 0) {};
		\node [style=node] (21) at (7.5, -4) {};
		\node [style=node] (22) at (6, 0) {};
	\end{pgfonlayer}
	\begin{pgfonlayer}{edgelayer}
		\draw [style={dash_3}] (10.center)
			 to (3.center)
			 to (4.center)
			 to (9.center)
			 to cycle;
		\draw [style=thick] (18) to (20);
		\draw [style=thick] (6) to (5);
		\draw [style=thick, bend right=45] (17) to (11);
		\draw [style=thick, bend left=315] (21) to (5);
		\draw [style=thick, bend left=45] (21) to (18);
	\end{pgfonlayer}
\end{tikzpicture}}
	\end{aligned}
	\end{equation}
	then
	we have that $d_{\pi_2} \circ d_{\pi_1}$ is 
	the 
	$(3,4)$--partition diagram 
	\begin{equation} \label{partcompdiagram}
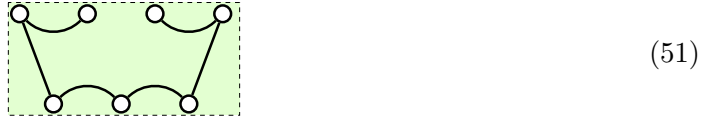
	\begin{aligned}
		\scalebox{0.6}{\begin{tikzpicture}
	\begin{pgfonlayer}{nodelayer}
		\node [style=none] (3) at (2.5, 0.5) {};
		\node [style=none] (4) at (2.5, -4.5) {};
		\node [style=node] (7) at (3, 0) {};
		\node [style=none] (13) at (12.75, -4.5) {};
		\node [style=none] (14) at (12.75, 0.5) {};
		\node [style=node] (15) at (12, 0) {};
		\node [style=node] (16) at (4.5, -4) {};
		\node [style=node] (17) at (10.5, -4) {};
		\node [style=node] (23) at (6, 0) {};
		\node [style=node] (24) at (9, 0) {};
		\node [style=node] (25) at (7.5, -4) {};
	\end{pgfonlayer}
	\begin{pgfonlayer}{edgelayer}
		\draw [style={dash_3}] (3.center)
			 to (4.center)
			 to (13.center)
			 to (14.center)
			 to cycle;
		\draw [style=thick, bend left=45] (16) to (25);
		\draw [style=thick, bend left=45] (25) to (17);
		\draw [style=thick] (7) to (16);
		\draw [style=thick] (17) to (15);
		\draw [style=thick, bend right=45] (7) to (23);
		\draw [style=thick, bend right=45] (24) to (15);
	\end{pgfonlayer}
\end{tikzpicture}}
	\end{aligned}
	\end{equation}
	and so $d_{\pi_2} \bullet d_{\pi_1}$ is the diagram (\ref{partcompdiagram}) 
	multiplied by $n^2$, since two connected components were removed from 
	the middle row of $d_{\pi_2} \circ d_{\pi_1}$.	
\end{example}

\begin{definition}[Tensor Product] \label{tensorprod}
	Let $d_{\pi_1} \in P_k^l(n)$ and $d_{\pi_2} \in P_q^m(n)$. Then $d_{\pi_1} \otimes d_{\pi_2}$ is defined to be the $(k+q, l+m)$--partition diagram that is obtained by horizontally placing $d_{\pi_1}$ to the left of $d_{\pi_2}$ without any overlapping of vertices.
\end{definition}

\begin{example}
	The tensor product $d_{\pi_1} \otimes d_{\pi_2}$, for $d_{\pi_1}$ and $d_{\pi_2}$ given in Example \ref{partcomp}, is the $(9,10)$--partition diagram 
	\begin{equation}
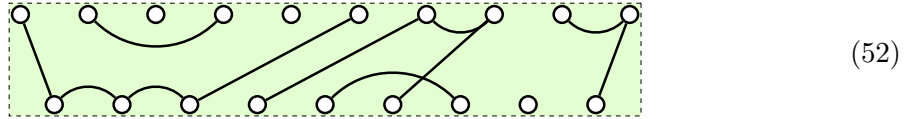
	\begin{aligned}
		\scalebox{0.6}{\begin{tikzpicture}
	\begin{pgfonlayer}{nodelayer}
		\node [style=node] (5) at (1.5, -4) {};
		\node [style=node] (6) at (0, 0) {};
		\node [style=none] (9) at (27.5, -4.5) {};
		\node [style=none] (10) at (27.5, 0.5) {};
		\node [style=node] (11) at (9, 0) {};
		\node [style=node] (13) at (12, 0) {};
		\node [style=node] (17) at (3, 0) {};
		\node [style=node] (18) at (7.5, -4) {};
		\node [style=node] (20) at (15, 0) {};
		\node [style=node] (21) at (4.5, -4) {};
		\node [style=node] (22) at (6, 0) {};
		\node [style=none] (23) at (-0.5, 0.5) {};
		\node [style=none] (24) at (-0.5, -4.5) {};
		\node [style=node] (25) at (18, 0) {};
		\node [style=node] (26) at (10.5, -4) {};
		\node [style=node] (29) at (27, 0) {};
		\node [style=node] (30) at (22.5, -4) {};
		\node [style=node] (31) at (13.5, -4) {};
		\node [style=node] (32) at (25.5, -4) {};
		\node [style=node] (33) at (21, 0) {};
		\node [style=node] (34) at (24, 0) {};
		\node [style=node] (35) at (16.5, -4) {};
		\node [style=node] (36) at (19.5, -4) {};
	\end{pgfonlayer}
	\begin{pgfonlayer}{edgelayer}
		\draw [style={dash_3}] (10.center)
			 to (9.center)
			 to (24.center)
			 to (23.center)
			 to cycle;
		\draw [style=thick, bend right=45] (17) to (11);
		\draw [style=thick] (18) to (20);
		\draw [style=thick] (6) to (5);
		\draw [style=thick, bend left=315] (21) to (5);
		\draw [style=thick, bend left=315] (18) to (21);
		\draw [style=thick] (26) to (25);
		\draw [style=thick, bend left=45] (31) to (36);
		\draw [style=thick] (32) to (29);
		\draw [style=thick, bend left=45] (29) to (34);
		\draw [style=thick] (35) to (33);
		\draw [style=thick, bend left=45] (33) to (25);
	\end{pgfonlayer}
\end{tikzpicture}}
	\end{aligned}
	\end{equation}
\end{example}
It is clear that both the composition and the tensor product operations are associative.
Consequently, we have the following category:

\begin{definition} \label{partitioncategory}
	We define the \textbf{partition category} $\mathcal{P}(n)$ to be 
	the category whose objects are the non--negative integers,
	and, for any pair of objects $k$ and $l$, the morphism space 
	$\Hom_{\mathcal{P}(n)}(k,l)$ is $P_k^l(n)$.
	
	The vertical composition of morphisms is 
	the composition of partition diagrams given in Definition \ref{composition};
	the horizontal composition of morphisms is 
	the tensor product of partition diagrams given in Definition \ref{tensorprod};
	and the unit object is $0$.
\end{definition}
For $B_k^l(n)$,
we can inherit the composition and tensor product operations 
from the composition and tensor product operations for $P_k^l(n)$, 
giving the following category:

\begin{definition}
	We define the \textbf{Brauer category} $\mathcal{B}(n)$ to be the category
	whose objects are the same as those of $\mathcal{P}(n)$
	and, for any pair of objects $k$ and $l$, the morphism space 
	$\Hom_{\mathcal{B}(n)}(k,l)$ is $B_k^l(n)$.
	
	The vertical and horizontal composition of morphisms 
	and the unit object are the same as those defined for $\mathcal{P}(n)$.
\end{definition}

We claim the following result.

\begin{proposition}
	The partition category $\mathcal{P}(n)$ and
	the Brauer category $\mathcal{B}(n)$
	are strict $\mathbb{R}$--linear monoidal categories.
\end{proposition}

\begin{proof}
	We prove this result for the partition category since the proof for 
	the Brauer category is effectively the same.

	$\mathcal{P}(n)$ is a strict monoidal category because the bifunctor
	on objects reduces to the addition of natural numbers, which is associative
	and the bifunctor on morphisms is the tensor product of linear
	combinations of partition diagrams that is given in Definition \ref{composition},
	which is associative by definition.

	$\mathcal{P}(n)$ is $\mathbb{R}$--linear because the morphism space 
	between any two objects is by definition a vector space,
	and the composition of morphisms is $\mathbb{R}$--bilinear 
	by definition.
	For the same reason, the bifunctor is also $\mathbb{R}$--bilinear.
\end{proof}

We can also obtain a strict $\mathbb{R}$--linear monoidal category 
$\mathcal{BG}(n)$ from the Brauer--Grood vector spaces $D_k^l(n)$.
We will call $\mathcal{BG}(n)$ the \textbf{Brauer--Grood category}, 
although it sometimes appears 
in the literature under different names, such as
the Jellyfish Brauer category in \cite{comes}.
Showing that $\mathcal{BG}(n)$ is strict $\mathbb{R}$--linear monoidal  
is a long and arduous computation that has previously appeared in
\citet{LehrerZhang, LehrerZhang2} -- namely, it 
is the result of
combining the proof of \citet[Theorem 2.6]{LehrerZhang}
together with the proof of \citet[Theorem 6.1]{LehrerZhang2}.
Hence, for brevity, we omit it here.

Instead, we show only part of the definition of the horizontal composition,
since it will be used in Section \ref{algorithmsect}.
To define the horizontal composition in full, we would need to define it on 
four possible cases of ordered pairs of diagrams. 
One of those ordered pairs is $(d_\beta, d_\alpha)$, where $d_\beta$
is a $(k,l)$--Brauer diagram 
and $d_\alpha$ is an $(m+q) \backslash n$--diagram.
In this case, the horizontal composition gives
the $(l+m+k+q) \backslash n$--diagram that is obtained by 
horizontally placing $d_\beta$ to the left of $d_\alpha$ 
without any overlapping of vertices.

\begin{example}
	Recalling the $(7,5)$--Brauer diagram $d_\beta$ and 
	the $(5+6) \backslash 3$--diagram $d_\alpha$
	that was given in Example \ref{exBrauerGrooddiag}, we see that
	$d_\beta \otimes d_\alpha$ is the 
	the $(10+13) \backslash 3$--diagram
	\begin{equation}
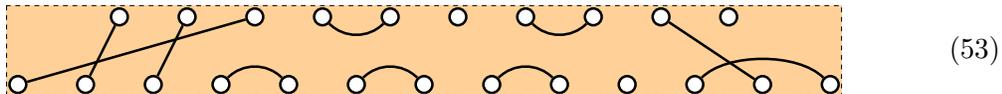
		\begin{aligned}
			\scalebox{0.6}{\begin{tikzpicture}
	\begin{pgfonlayer}{nodelayer}
		\node [style=node] (28) at (27, -1.5) {};
		\node [style=node] (29) at (33, -1.5) {};
		\node [style=node] (30) at (25.5, 1.5) {};
		\node [style=node] (31) at (22.5, 1.5) {};
		\node [style=node] (32) at (19.5, 1.5) {};
		\node [style=node] (33) at (30, -1.5) {};
		\node [style=none] (34) at (-0.5, 2) {};
		\node [style=none] (35) at (-0.5, -2) {};
		\node [style=none] (36) at (36.5, -2) {};
		\node [style=none] (37) at (36.5, 2) {};
		\node [style=node] (38) at (28.5, 1.5) {};
		\node [style=node] (39) at (36, -1.5) {};
		\node [style=node] (40) at (31.5, 1.5) {};
		\node [style=node] (52) at (21, -1.5) {};
		\node [style=node] (53) at (24, -1.5) {};
		\node [style=node] (0) at (0, -1.5) {};
		\node [style=node] (1) at (12, -1.5) {};
		\node [style=node] (2) at (9, -1.5) {};
		\node [style=node] (3) at (6, -1.5) {};
		\node [style=node] (4) at (13.5, 1.5) {};
		\node [style=node] (5) at (7.5, 1.5) {};
		\node [style=node] (6) at (4.5, 1.5) {};
		\node [style=node] (7) at (3, -1.5) {};
		\node [style=node] (12) at (16.5, 1.5) {};
		\node [style=node] (13) at (15, -1.5) {};
		\node [style=node] (14) at (10.5, 1.5) {};
		\node [style=node] (26) at (18, -1.5) {};
	\end{pgfonlayer}
	\begin{pgfonlayer}{edgelayer}
		\draw [style={dash_4}] (34.center)
			 to (35.center)
			 to (36.center)
			 to (37.center)
			 to cycle;
		\draw [style=thick, bend left=315] (53) to (52);
		\draw [style=thick, bend right=45] (31) to (30);
		\draw [style=thick] (38) to (29);
		\draw [style=thick, bend left=315, looseness=0.75] (39) to (33);
		\draw [style=thick, bend left=45] (13) to (26);
		\draw [style=thick, bend right=45] (4) to (12);
		\draw [style=thick, bend left=45] (2) to (1);
		\draw [style=thick] (0) to (14);
		\draw [style=thick] (7) to (6);
		\draw [style=thick] (3) to (5);
	\end{pgfonlayer}
\end{tikzpicture}}
		\end{aligned}
	\end{equation}
\end{example}

For the sake of completeness, 
we provide the framework of the definition of the Brauer--Grood category below.
\begin{definition}
	The \textbf{Brauer--Grood category} 
	$\mathcal{BG}(n)$ is the category
	whose objects are the same as those of $\mathcal{P}(n)$
	and, for any pair of objects $k$ and $l$, the morphism space 
	$\Hom_{\mathcal{BG}(n)}(k,l)$ is $D_k^l(n)$.
	[We have omitted the full definition of
	the vertical composition of morphisms 
	and the horizontal composition of morphisms.]
	The unit object is $0$.
\end{definition}

%


	We will refer to these categories in general as \textbf{partition categories}. 
	When we wish to specifically reference the partition category $\mathcal{P}(n)$,
	we will explicitly write $\mathcal{P}(n)$ to clarify that we mean this particular category.

\subsubsection{Group Representation Categories}

It is not just the partition vector spaces that are similar 
for different values of $l$ and $k$.
The tensor power spaces of $\mathbb{R}^{n}$ that make up the layer spaces
and, in particular, the equivariant linear maps between them,
are also similar for different values of $l$ and $k$.
To form a category from these representations and the linear maps between them,
we first need to define a category for all representations of a group.

\begin{definition}
	Let $G$ be a group. Then $\Rep(G)$ is the category
	whose objects are pairs $(V, \rho_V)$,
	where $\rho_V: G \rightarrow GL(V)$ is a representation of $G$,
	and,
	for any pair of objects 
	$(V, \rho_V)$ and $(W, \rho_W)$,
	the morphism space, 
	$\Hom_{\Rep(G)}((V, \rho_V), (W, \rho_W))$,
	is precisely the vector space of $G$--equivariant linear maps $V \rightarrow W$,
	$\Hom_{G}(V, W)$.

	The vertical composition of morphisms is given by the 
	composition of linear maps, the horizontal composition of morphisms 
	is given by the tensor product of linear maps, 
	both of which are associative operations, 
	and the unit object is given by $(\mathbb{R}, 1_\mathbb{R})$, 
	where $1_\mathbb{R}$ is the one-dimensional trivial representation of $G$.
\end{definition}

If $G = G(n)$ is a subgroup of $GL(n)$ 
(and, in particular, one of $S_n, O(n), SO(n)$ or $Sp(n)$), 
then we have the following category. 

\begin{definition} \label{CGcategory}
	If $G(n)$ is a subgroup of $GL(n)$, then 
	the \textbf{tensor power representation category}
	$\mathcal{C}(G(n))$ is the category
	whose objects are pairs 
	$\{((\mathbb{R}^n)^{\otimes k},\rho_k)\}_{k \in \mathbb{Z}_{\geq 0}}$, 
	where $\rho_k: G(n) \rightarrow GL((\mathbb{R}^n)^{\otimes k})$ 
	is the representation of $G(n)$ given 
	in Section \ref{sectLayerSpaces},
	and, for any pair of objects
	$((\mathbb{R}^n)^{\otimes k},\rho_k)$ and 
	$((\mathbb{R}^n)^{\otimes l},\rho_l)$,
	the morphism space 
	is precisely
	$\Hom_{G(n)}((\mathbb{R}^n)^{\otimes k}, (\mathbb{R}^n)^{\otimes l})$.

	The vertical and horizontal composition of morphisms 
	together with the unit object are inherited from $\Rep(G(n))$.

\end{definition}

\begin{proposition}
	For any subgroup $G(n)$ of $GL(n)$, 
	$\mathcal{C}(G(n))$ is a full subcategory of 
	the category of representations of $G(n)$, $\Rep(G(n))$; that is,
	for every pair of objects in $\mathcal{C}(G(n))$, every morphism
	between them in $\Rep(G(n))$ is a morphism in $\mathcal{C}(G(n))$.
	In particular, $\mathcal{C}(G(n))$ is a strict $\mathbb{R}$--linear monoidal category.
\end{proposition}

\begin{proof}
	It is clear from the definitions that $\mathcal{C}(G(n))$ is a subcategory of $\Rep(G(n))$.
	$\mathcal{C}(G(n))$ 
	is a full subcategory
	of $\Rep(G(n))$, 
	as, for any pair of objects  
	$((\mathbb{R}^n)^{\otimes k},\rho_k)$ and $((\mathbb{R}^n)^{\otimes l},\rho_l)$
	in $C(G(n))$,
	the morphism space in 
	$\mathcal{C}(G(n))$ is the same as the one in $\Rep(G(n))$.

	$\mathcal{C}(G(n))$ can immediately be seen to be a 
	strict monoidal category
	because the bifunctor on objects is
	the tensor product of vector spaces,
	which is associative,
	and the bifunctor on morphisms is 
	the tensor product
	on linear maps between vector spaces,
	which is also associative.
	
	$\mathcal{C}(G(n))$ is $\mathbb{R}$--linear because the morphism space 
	for any two objects 
	is a vector space over $\mathbb{R}$ 
	and the composition of morphisms is $\mathbb{R}$--bilinear because
	composition is $\mathbb{R}$--bilinear for linear maps on vector spaces.
	It is also clear that the bifunctor $\otimes$ is $\mathbb{R}$--bilinear
	since it is the standard tensor product for vector spaces.
\end{proof}

\subsection
{Monoidal Functors from Partition Categories to Group Representation Categories}
	\label{fullstrictmonfunct}

In Section \ref{sectCharacterisation},
we reviewed a number of linear maps 
between certain partition vector spaces and
the vector space of equivariant linear maps 
between tensor power representations 
for four subgroups $G(n)$ of $GL(n)$.
We now show that these linear maps are the result of something stronger, 
namely that they are the linear maps that come from the existence of
a strict $\mathbb{R}$--linear monoidal functor
between a partition category and its associated
tensor power representation category. 
In each of the following four theorems, 
we state that the functors are \textbf{full}; 
that is, the underlying map between morphism spaces, 
having chosen a pair of objects in the domain category, is surjective.




\begin{theorem} \label{partfunctor}
	There exists a full, strict $\mathbb{R}$--linear monoidal functor
	\begin{equation}
		\Theta : \mathcal{P}(n) \rightarrow \mathcal{C}(S_n)
	\end{equation}
	that is defined on the objects of $\mathcal{P}(n)$ by 
	$\Theta(k) \coloneqq ((\mathbb{R}^{n})^{\otimes k}, \rho_k)$ 
	and, for any objects $k,l$ of $\mathcal{P}(n)$, the map
	\begin{equation} \label{partmorphism}
		\Hom_{\mathcal{P}(n)}(k,l) 
		\rightarrow 
		\Hom_{\mathcal{C}(S_n)}(\Theta(k),\Theta(l))
	\end{equation}
	is precisely the map 
	\begin{equation} \label{partmorphismmap}
		\Theta_{k,n}^l : P_k^l(n) \rightarrow 
		\Hom_{S_n}((\mathbb{R}^{n})^{\otimes k}, (\mathbb{R}^{n})^{\otimes l})
	\end{equation}
	given in Theorem \ref{diagbasisSn}.
\end{theorem}

\begin{theorem} \label{brauerO(n)functor}
	There exists a full, strict $\mathbb{R}$--linear monoidal functor
	\begin{equation}
		\Phi : \mathcal{B}(n) \rightarrow \mathcal{C}(O(n))
	\end{equation}
	that is defined on the objects of $\mathcal{B}(n)$ by 
	$\Phi(k) \coloneqq ((\mathbb{R}^{n})^{\otimes k}, \rho_k)$
	and, for any objects $k,l$ of $\mathcal{B}(n)$, the map
	\begin{equation}	
		\Hom_{\mathcal{B}(n)}(k,l) 
		\rightarrow 
		\Hom_{\mathcal{C}(O(n))}(\Phi(k),\Phi(l))
	\end{equation}
	is the map
	\begin{equation} 
		\Phi_{k,n}^l : B_k^l(n) \rightarrow 
		\Hom_{O(n)}((\mathbb{R}^{n})^{\otimes k}, (\mathbb{R}^{n})^{\otimes l})
	\end{equation}
	given in Theorem \ref{spanningsetO(n)}.
\end{theorem}

\begin{theorem} \label{brauerSp(n)functor}
	There exists a full, strict $\mathbb{R}$--linear monoidal functor
	\begin{equation}
		X : \mathcal{B}(n) \rightarrow \mathcal{C}(Sp(n))
	\end{equation}
	that is defined on the objects of $\mathcal{B}(n)$ by 
	$X(k) \coloneqq ((\mathbb{R}^{n})^{\otimes k}, \rho_k)$
	and, for any objects $k,l$ of $\mathcal{B}(n)$, the map
	\begin{equation}	
		\Hom_{\mathcal{B}(n)}(k,l) 
		\rightarrow 
		\Hom_{\mathcal{C}(Sp(n))}(\Phi(k),\Phi(l))
	\end{equation}
	is the map
	\begin{equation} 
		X_{k,n}^l : B_k^l(n) \rightarrow 
		\Hom_{Sp(n)}((\mathbb{R}^{n})^{\otimes k}, (\mathbb{R}^{n})^{\otimes l})
	\end{equation}
	given in Theorem \ref{spanningsetSp(n)}.
\end{theorem}

\begin{theorem} \label{brauerSO(n)functor}
	There exists a full, strict $\mathbb{R}$--linear monoidal functor
	\begin{equation}
		\Psi : \mathcal{BG}(n) \rightarrow \mathcal{C}(SO(n))
	\end{equation}
	that is defined on the objects of $\mathcal{BG}(n)$ by 
	$\Psi(k) \coloneqq ((\mathbb{R}^{n})^{\otimes k}, \rho_k)$
	and, for any objects $k,l$ of $\mathcal{B}(n)$, the map
	\begin{equation}	
		\Hom_{\mathcal{BG}(n)}(k,l) 
		\rightarrow 
		\Hom_{\mathcal{C}(SO(n))}(\Phi(k),\Phi(l))
	\end{equation}
	is the map
	\begin{equation} 
		\Psi_{k,n}^l : D_k^l(n) \rightarrow 
		\Hom_{SO(n)}((\mathbb{R}^{n})^{\otimes k}, (\mathbb{R}^{n})^{\otimes l})
	\end{equation}
	given in Theorem \ref{spanningsetSO(n)}.
\end{theorem}

We only show the proof of Theorem \ref{partfunctor} in full since the proofs of 
Theorems \ref{brauerO(n)functor} and \ref{brauerSp(n)functor} 
are almost identical.
The proof of Theorem \ref{brauerSO(n)functor} can be found in 
\citet[Theorem 6.1]{LehrerZhang2}. \\

\begin{proof}[Proof of Theorem \ref{partfunctor}]
	We prove this theorem in a series of steps. \\

	\noindent
	\textbf{Step 1}: We need to show that 
	$\Theta : \mathcal{P}(n) \rightarrow \mathcal{C}(S_n)$ is a functor.

	We prove this step in three parts.
	

\begin{itemize}
	\item $\Theta$ maps the objects of $\mathcal{P}(n)$ to the objects of
		$\mathcal{C}(S_n)$ by definition.
	\item It is enough to show that $\Theta(g)\Theta(f) = \Theta(g \bullet f)$ on arbitrary basis elements of arbitrary morphism spaces where
		the codomain of $f$ is the domain of $g$
		because the morphism spaces are vector spaces. 

		Let $f = d_{\pi_1}$ be a $(k,l)$--partition diagram, and
		let $g = d_{\pi_2}$ be a $(l,m)$--partition diagram.
		Then, by Definition \ref{composition}, we have that
		\begin{equation}
			d_{\pi_2} \bullet d_{\pi_1} 
			=
			n^{c(d_{\pi_2}, d_{\pi_1})} d_{\pi_3}
		\end{equation}
		where $d_{\pi_3}$ be the composition $d_{\pi_2} \circ d_{\pi_1}$
		expressed as a $(k,m)$--partition diagram.
		Then 
		\begin{equation} \label{catpartbulletmatmult}
			\Theta(g \bullet f) 
			= 
			\Theta(d_{\pi_2} \bullet d_{\pi_1}) 
			=
			n^{c(d_{\pi_2}, d_{\pi_1})} D_{\pi_3}
		\end{equation}
		We also have that
		\begin{align}
			\Theta(g)\Theta(f)
			=
			D_{\pi_2}D_{\pi_1}
			& =
			\left(
			\sum_{I \in [n]^m, K \in [n]^l}
			\delta_{\pi_2, (I,K)}
			E_{I,K}
			\right)
			\left(
			\sum_{L \in [n]^l, J \in [n]^k}
			\delta_{\pi_1, (L,J)}
			E_{L,J} 
			\right)
			\\
			& =
			\sum_{I \in [n]^m, K \in [n]^l, J \in [n]^k}
			\delta_{\pi_2, (I,K)}
			\delta_{\pi_1, (K,J)}
			E_{I,K}E_{K,J} \label{catpartcompmultmat}
		\end{align}

		For fixed $I, J$, consider
		\begin{equation} \label{catpartcompdeltas}
			\sum_{K \in [n]^l}
			\delta_{\pi_2, (I,K)}
			\delta_{\pi_1, (K,J)}
		\end{equation}
		This is equal to
		\begin{equation}
			n^{c(d_{\pi_2}, d_{\pi_1})}
			\delta_{\pi_3, (I,J)}
		\end{equation}
		Indeed, for fixed $I,J$, 
		$\delta_{\pi_3, (I,J)}$ is $1$
		if and only if 
		both $\delta_{\pi_2, (I,K)}$ and 
		$\delta_{\pi_1, (K,J)}$ are $1$
		for some $K \in [n]^l$
		since $d_{\pi_3}$ is the composition
		$d_{\pi_2} \circ d_{\pi_1}$.
		The number of such $K$ is determined
		only by the vertices that appear in connected components
		that are removed from the middle row of 
		$d_{\pi_2} \circ d_{\pi_1}$, since, for fixed $I,J$,
		only these vertices can be freely chosen
		if we want both $\delta_{\pi_2, (I,K)}$ and
		$\delta_{\pi_1, (K,J)}$ to be $1$.
		However, since the entries in each connected component must 
		be the same for both $\delta_{\pi_2, (I,K)}$ and 
		$\delta_{\pi_1, (K,J)}$ to be $1$,
		this implies that the number of $K \in [n]^l$ such that
		both $\delta_{\pi_2, (I,K)}$ and 
		$\delta_{\pi_1, (K,J)}$ are $1$
		is 
		$n^{c(d_{\pi_2}, d_{\pi_1})}$.

		Hence (\ref{catpartcompmultmat}) becomes
		\begin{equation}
			\sum_{I \in [n]^m, J \in [n]^k}
			n^{c(d_{\pi_2}, d_{\pi_1})}
			\delta_{\pi_3, (I,J)}
			E_{I,J}
			=
			n^{c(d_{\pi_2}, d_{\pi_1})} E_{\pi_3}
		\end{equation}
		and so, by (\ref{catpartbulletmatmult}),
		we have that $\Theta(g)\Theta(f) = \Theta(g \bullet f)$, as required.
	\item 
		For each object $k$ in $\mathcal{P}(n)$,
		the identity morphism $1_{k}$ is the 
		$(k,k)$--partition diagram $d_\pi$
		where $\pi$ is the set partition
		\begin{equation}
			\{1, k+1 \mid 2, k+2 \mid \dots \mid k, 2k\}
		\end{equation}
		of $[2k]$, and its image under $d_\pi \mapsto D_\pi$ is
		the $n^k \times n^k$ identity matrix, as required. 
\end{itemize}

\noindent
\textbf{Step 2}: We need to show that $\Theta : \mathcal{P}(n) \rightarrow \mathcal{C}(S_n)$ is full. \\

	The functor $\Theta$ is full because,
	for all objects $k, l$ in $\mathcal{P}(n)$,
	the map (\ref{partmorphism})
	is (\ref{partmorphismmap}), by definition, and
	this map is surjective by Theorem \ref{diagbasisSn}. \\

\noindent
\textbf{Step 3}: We need to show that $\Theta : \mathcal{P}(n) \rightarrow \mathcal{C}(S_n)$ is strict $\mathbb{R}$--linear monoidal. \\

	To show that $\Theta$ is strict $\mathbb{R}$--linear monoidal,
	we need to show that $\Theta$ satisfies the conditions 
	given in Definition \ref{monoidalfunctordefn}.
	The picture to have in mind for point 2 below is that 
	the tensor product on set partition diagrams places 
	the diagrams side-by-side, without any overlapping of vertices.
	\begin{enumerate}
		\item Let $k, l$ be any two objects in $\mathcal{P}(n)$. Then 
\begin{align}
	\Theta(k \otimes l) 
	& = \Theta(k + l) \\
	& = ((\mathbb{R}^{n})^{\otimes k+l}, \rho_{k+l}) \\
	& = ((\mathbb{R}^{n})^{\otimes k} \otimes (\mathbb{R}^{n})^{\otimes l}, \rho_k \otimes \rho_l) \\
	& = \Theta(k) \otimes \Theta(l)
\end{align}
		\item	
			It is enough to prove this condition on arbitrary basis 
			elements of arbitrary morphism spaces 
			as the morphism spaces are vector spaces.
		Suppose that $f: k \rightarrow l$ and $g: q \rightarrow m$ are two basis elements in $\Hom_{\mathcal{P}(n)}(k,l)$ and $\Hom_{\mathcal{P}(n)}(q,m)$ respectively.
			Then $f = d_\pi$ for some set partition $\pi$ of $[l+k]$, 
			and $g = d_\tau$ for some set partition $\tau$ 
			of $[m+q]$. 

			As $\Hom_{\mathcal{P}(n)}(k,l) = P_k^l(n)$
			and $\Hom_{\mathcal{P}(n)}(q,m) = P_q^m(n)$, we have, 
			by Definition \ref{tensorprod}, 
			that $f \otimes g$ is an element of 
			$P_{k+q}^{l+m}(n) = \Hom_{\mathcal{P}(n)}(k+q,l+m)$.
			
			In particular, $f \otimes g = d_{\omega}$ for the 
			set partition $\omega \coloneqq \pi \cup \tau$
			of $[l+m+k+q]$. 
			
			By Theorem \ref{diagbasisSn}, we have that 
			$\Theta(f) = D_\pi$, $\Theta(g) = D_\tau$, 
			and $\Theta(f \otimes g) = D_\omega$.

			We now show that 
			$\Theta(f) \otimes \Theta(g) = \Theta(f \otimes g)$. 
			Indeed, we have that
\begin{align} 
	\Theta(f) \otimes \Theta(g)
	& =
	D_\pi \otimes D_\tau \\
	& =
	\left(
	\sum_{I \in [n]^l, J \in [n]^k}
	\delta_{\pi, (I,J)}
	E_{I,J}
	\right)
	\otimes
	\left(
	\sum_{X \in [n]^m, Y \in [n]^q}
	\delta_{\tau, (X,Y)}
	E_{X,Y} 
	\right)
	\\
	& =
	\sum_{(I,X) \in [n]^{l+m}, (J,Y) \in [n]^{k+q}}
	\delta_{\omega, (I,X),(J,Y))}
	E_{(I,X),(J,Y)} \label{deltaprod} \\
	& = D_\omega \\
	& = \Theta(f \otimes g)
\end{align}
where (\ref{deltaprod}) holds because $S_\pi((I,J)) \cup S_\tau((X,Y))
				= S_\omega((I,X),(J,Y))$.
			\item It is clear from the statement of the theorem that
				$\Theta$ sends the unit object $0$ in $\mathcal{P}(n)$ to $(\mathbb{R}, 1_\mathbb{R})$, which is the unit object in $\mathcal{C}(S_n)$.
			\item This is immediate because the map (\ref{partmorphismmap}) is $\mathbb{R}$--linear by Theorem \ref{diagbasisSn}. 
		\myqedhere
	\end{enumerate} 
\end{proof}

%
	%

\subsection{Implications for Group Equivariant Neural Networks} \label{implicationsgroupequiv}

The theorems in Section \ref{fullstrictmonfunct}
imply several key statements that are crucial
for the group equivariant neural networks 
under consideration.
These statements can be formulated as follows.

\begin{enumerate}
	\item To understand and work with any matrix in 
		$\Hom_{G(n)}((\mathbb{R}^{n})^{\otimes k}, (\mathbb{R}^{n})^{\otimes l})$,
		it is enough to work with the subset of $(k,l)$--partition diagrams 
		that correspond to $G(n)$.
		This is because 
		we can express any matrix in terms of the set of spanning set elements for
		$\Hom_{G(n)}((\mathbb{R}^{n})^{\otimes k}, (\mathbb{R}^{n})^{\otimes l})$,
		and these correspond
		bijectively with the subset of $(k,l)$--partition diagrams that corresponds to $G(n)$. 
		We can recover the matrix itself by 
		applying the appropriate monoidal functor to the set partition diagrams because the monoidal functors are full.
	\item We can manipulate the connected components and vertices of 
		$(k,l)$--partition diagrams 
		to obtain new set partition diagrams.
		Indeed, as the partition categories are strict monoidal, this
		means that
		the $(k,l)$--partition diagrams in a partition category 
		are string diagrams.
		Statement 1 immediately implies that we will obtain new matrices 
		between tensor power spaces of $\mathbb{R}^{n}$ 
		that are equivariant to $G(n)$
		from the resulting set partition diagrams.
	\item If a $(k,l)$--partition diagram can be decomposed as a tensor product of smaller set partition diagrams, then the corresponding matrix can also be decomposed as a tensor product of smaller sized matrices, each of which is equivariant to $G(n)$. This is because the functors are monoidal.
\end{enumerate}

In the following section, we present our main contribution.
We use
the monoidal functors 
that we have established for each of the four groups,
together with the implications for 
the group equivariant neural networks given above,
to construct a fast multiplication algorithm for passing a vector 
through a linear layer function that appears in such a network.

 


\section{Algorithm for Computing with Equivariant Weight Matrices} \label{algorithmsect}

Suppose that 
we would like to pass a vector
$v \in (\mathbb{R}^{n})^{\otimes k}$
through a linear layer function that lives in 
$\Hom_{G(n)}((\mathbb{R}^{n})^{\otimes k}, (\mathbb{R}^{n})^{\otimes l})$.
A naive implementation of the matrix multiplication 
between the weight matrix and the vector 
has a time complexity equal to $O(n^{l+k})$.
We show that
we can construct a fast multiplication algorithm 
that reduces the time complexity exponentially
for each of the groups $G(n)$ in question.

It is important to highlight that,
since we have at least a spanning set of the vector space
$\Hom_{G(n)}((\mathbb{R}^{n})^{\otimes k}, (\mathbb{R}^{n})^{\otimes l})$ 
for each group $G(n)$ that appeared in the previous chapter,
it is enough to describe an algorithm for how to multiply $v$ by a spanning set
element since we can extend the result by linearity.
This linearity 
provides an extra advantage 
in that it enables the fast matrix multiplication of
an input vector $v \in (\mathbb{R}^{n})^{\otimes k}$
with a weight matrix in 
$\Hom_{G(n)}((\mathbb{R}^{n})^{\otimes k}, (\mathbb{R}^{n})^{\otimes l})$
to be executed in parallel, namely by performing the fast 
matrix multiplication
on each of the spanning set
elements that make up the weight matrix.

We begin by introducing a new class of set partition diagrams
that we have termed \textbf{algorithmically planar}.

\subsection{Algorithmically Planar Set Partition Diagrams}

Recall that each spanning set element in
$\Hom_{G(n)}((\mathbb{R}^{n})^{\otimes k}, (\mathbb{R}^{n})^{\otimes l})$
corresponds to a set partition diagram
under a monoidal functor for $G(n)$,
and that a set partition diagram is a string diagram
since it is a morphism in a monoidal category.
This leads to the following question that we use as motivation:
can we use the string-like property of each set partition diagram
to gain control over how the matrix multiplication operation is performed 
between the spanning set element and the vector?

There are a few guiding principles that we use to construct the algorithm.

Ideally, we would like to break the matrix multiplication operation down 
into a series of computations, each of which is indecomposable.
Given the monoidal functor relationships between 
set partition diagrams and matrix operations,
we know that the smallest possible computations that are available to us 
correspond precisely with the smallest indecomposable set partition diagrams.
In general, set partition diagrams that consist only of one block 
are the smallest indecomposable diagrams.
The exception is for $(l+k) \backslash n$--diagrams, where
the smallest indecomposable diagrams are either those consisting of a single block of two vertices
or a diagram consisting of all of the free vertices taken together.
Moreover, a single block can live either entirely in the top row, 
entirely in the bottom row, or in both rows. 
Hence, we see that there are only a few classes of indecomposable diagrams that are smallest.
As a result, there are only a few classes of
indecomposable matrix operations that are smallest.

We can say more.
Since the matrices that we obtain must come from set partition diagrams,
the most efficient matrices that we can construct 
are a Kronecker product of indecomposable matrices.
But also, since indecomposable matrices come from indecomposable set partition diagrams, 
this implies that the most efficient matrices that we can construct come from
set partition diagrams that decompose into a tensor product
of the smallest indecomposable set partition diagrams.
Moreover, we would like to choose the order in which the indecomposable matrices 
appear in the Kronecker product because each indecomposable 
matrix performs a certain operation 
and we would like them to be executed in an order that is as efficient as possible. 
For example, it is best to zero out as many terms as possible in the input vector first, if such an operation can be performed, since this reduces the number of elements that 
need to be operated on 
going forward.
After this, it is better to perform tensor contraction operations before performing copying operations, because this reduces the number of elements that need to be copied. 
It is clear that choosing any other order for these operations would make the overall algorithm less efficient.
We will see that each of these operations can be understood easily in their equivalent diagram form.

Consequently, from the set partition diagram that corresponds to the spanning set element,
we would like to create another set partition diagram that 
decomposes into a tensor product of the smallest possible indecomposable set partition diagrams
and are ordered such that the indecomposable matrix operations 
that appear in the corresponding Kronecker product
will be performed most efficiently.


We call these special set partition diagrams algorithmically planar, and by construction they are the most efficient diagrams for performing matrix multiplication. They are defined as follows.

\begin{definition} \label{algplanardefnSn}
	A $(k,l)$--partition diagram is said to be 
	\textbf{algorithmically planar}
	if it satisfies the following conditions:
	\begin{itemize}
		\item The connected components that are solely in the bottom row are
			lined up sequentially,
			starting from the far right hand side, such that the vertices in each connected component are in consecutive order.
			We also require that these connected components are ordered by their size, 
			starting on the far right hand side with the largest
			and decreasing in size as we move towards the left, for reasons that will become clear
			when we conduct a time complexity analysis 
			of the multiplication algorithm for the symmetric group.
		\item The connected components that are solely in the top row are
			lined up sequentially,
			starting from the far left hand side, such that the vertices in each connected component are in consecutive order.
			(Note that these connected components can be in any order.)
		\item The connected components between different rows are such 
			that no two subsets cross.
	\end{itemize}
\end{definition}

\begin{example}
	The $(6,5)$--partition diagram
	\begin{equation}
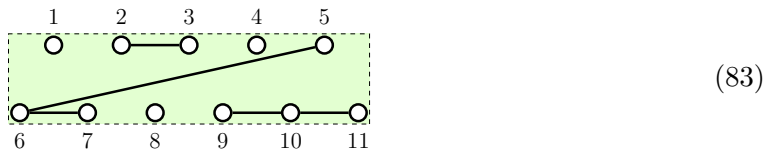
		\begin{aligned}
		\scalebox{0.6}{\begin{tikzpicture}
	\begin{pgfonlayer}{nodelayer}
		\node [style=none] (8) at (-0.5, 2) {};
		\node [style=none] (9) at (-0.5, -2) {};
		\node [style=none] (10) at (15.5, -2) {};
		\node [style=none] (11) at (15.5, 2) {};
		\node [style=none] (15) at (1.5, 2.75) {\Large$1$};
		\node [style=none] (16) at (4.5, 2.75) {\Large$2$};
		\node [style=none] (17) at (7.5, 2.75) {\Large$3$};
		\node [style=none] (18) at (10.5, 2.75) {\Large$4$};
		\node [style=none] (19) at (13.5, 2.75) {\Large$5$};
		\node [style=none] (20) at (0, -2.75) {\Large$6$};
		\node [style=none] (21) at (3, -2.75) {\Large$7$};
		\node [style=none] (22) at (6, -2.75) {\Large$8$};
		\node [style=none] (23) at (9, -2.75) {\Large$9$};
		\node [style=none] (24) at (12, -2.75) {\Large$10$};
		\node [style=none] (25) at (15, -2.75) {\Large$11$};
		\node [style=node] (0) at (0, -1.5) {};
		\node [style=node] (1) at (12, -1.5) {};
		\node [style=node] (2) at (9, -1.5) {};
		\node [style=node] (3) at (6, -1.5) {};
		\node [style=node] (4) at (7.5, 1.5) {};
		\node [style=node] (5) at (4.5, 1.5) {};
		\node [style=node] (6) at (1.5, 1.5) {};
		\node [style=node] (7) at (3, -1.5) {};
		\node [style=node] (12) at (13.5, 1.5) {};
		\node [style=node] (13) at (15, -1.5) {};
		\node [style=node] (14) at (10.5, 1.5) {};
	\end{pgfonlayer}
	\begin{pgfonlayer}{edgelayer}
		\draw [style={dash_3}] (10.center)
			 to (11.center)
			 to (8.center)
			 to (9.center)
			 to cycle;
		\draw [style=thick] (5) to (4);
		\draw [style=thick] (12) to (0);
		\draw [style=thick] (0) to (7);
		\draw [style=thick] (1) to (13);
		\draw [style=thick] (2) to (1);
	\end{pgfonlayer}
\end{tikzpicture}}
		\end{aligned}
	\end{equation}
	is algorithmicallly planar. 
	However, 
	the $(6,5)$--partition diagram
	\begin{equation}
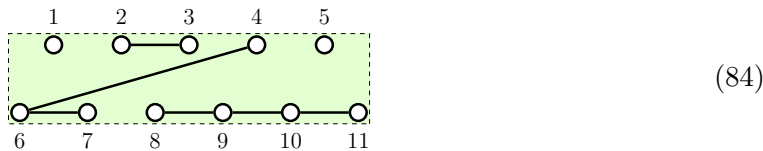
		\begin{aligned}
		\scalebox{0.6}{\begin{tikzpicture}
	\begin{pgfonlayer}{nodelayer}
		\node [style=node] (0) at (0, -1.5) {};
		\node [style=node] (1) at (12, -1.5) {};
		\node [style=node] (2) at (9, -1.5) {};
		\node [style=node] (3) at (6, -1.5) {};
		\node [style=node] (4) at (7.5, 1.5) {};
		\node [style=node] (5) at (4.5, 1.5) {};
		\node [style=node] (6) at (1.5, 1.5) {};
		\node [style=node] (7) at (3, -1.5) {};
		\node [style=none] (8) at (-0.5, 2) {};
		\node [style=none] (9) at (-0.5, -2) {};
		\node [style=none] (10) at (15.5, -2) {};
		\node [style=none] (11) at (15.5, 2) {};
		\node [style=node] (12) at (10.5, 1.5) {};
		\node [style=node] (13) at (15, -1.5) {};
		\node [style=node] (14) at (13.5, 1.5) {};
		\node [style=none] (15) at (1.5, 2.75) {\Large$1$};
		\node [style=none] (16) at (4.5, 2.75) {\Large$2$};
		\node [style=none] (17) at (7.5, 2.75) {\Large$3$};
		\node [style=none] (18) at (10.5, 2.75) {\Large$4$};
		\node [style=none] (19) at (13.5, 2.75) {\Large$5$};
		\node [style=none] (20) at (0, -2.75) {\Large$6$};
		\node [style=none] (21) at (3, -2.75) {\Large$7$};
		\node [style=none] (22) at (6, -2.75) {\Large$8$};
		\node [style=none] (23) at (9, -2.75) {\Large$9$};
		\node [style=none] (24) at (12, -2.75) {\Large$10$};
		\node [style=none] (25) at (15, -2.75) {\Large$11$};
	\end{pgfonlayer}
	\begin{pgfonlayer}{edgelayer}
		\draw [style={dash_3}] (10.center)
			 to (11.center)
			 to (8.center)
			 to (9.center)
			 to cycle;
		\draw [style=thick] (5) to (4);
		\draw [style=thick] (12) to (0);
		\draw [style=thick] (0) to (7);
		\draw [style=thick, bend left=360] (2) to (3);
		\draw [style=thick] (1) to (13);
		\draw [style=thick] (1) to (2);
	\end{pgfonlayer}
\end{tikzpicture}}
		\end{aligned}
	\end{equation}
	is not algorithmically planar because the connected component 
	consisting solely of vertex $5$ is not adjacent 
	to a connected component that lives solely in the top row,
	and the $(6,5)$--partition diagram
	\begin{equation}
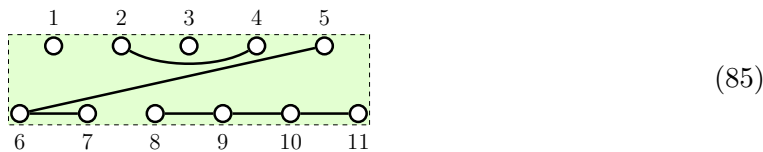
		\begin{aligned}
		\scalebox{0.6}{\begin{tikzpicture}
	\begin{pgfonlayer}{nodelayer}
		\node [style=node] (0) at (0, -1.5) {};
		\node [style=node] (1) at (12, -1.5) {};
		\node [style=node] (2) at (9, -1.5) {};
		\node [style=node] (3) at (6, -1.5) {};
		\node [style=node] (4) at (10.5, 1.5) {};
		\node [style=node] (5) at (4.5, 1.5) {};
		\node [style=node] (6) at (1.5, 1.5) {};
		\node [style=node] (7) at (3, -1.5) {};
		\node [style=none] (8) at (-0.5, 2) {};
		\node [style=none] (9) at (-0.5, -2) {};
		\node [style=none] (10) at (15.5, -2) {};
		\node [style=none] (11) at (15.5, 2) {};
		\node [style=node] (12) at (13.5, 1.5) {};
		\node [style=node] (13) at (15, -1.5) {};
		\node [style=node] (14) at (7.5, 1.5) {};
		\node [style=none] (15) at (1.5, 2.75) {\Large$1$};
		\node [style=none] (16) at (4.5, 2.75) {\Large$2$};
		\node [style=none] (17) at (7.5, 2.75) {\Large$3$};
		\node [style=none] (18) at (10.5, 2.75) {\Large$4$};
		\node [style=none] (19) at (13.5, 2.75) {\Large$5$};
		\node [style=none] (20) at (0, -2.75) {\Large$6$};
		\node [style=none] (21) at (3, -2.75) {\Large$7$};
		\node [style=none] (22) at (6, -2.75) {\Large$8$};
		\node [style=none] (23) at (9, -2.75) {\Large$9$};
		\node [style=none] (24) at (12, -2.75) {\Large$10$};
		\node [style=none] (25) at (15, -2.75) {\Large$11$};
	\end{pgfonlayer}
	\begin{pgfonlayer}{edgelayer}
		\draw [style={dash_3}] (10.center)
			 to (11.center)
			 to (8.center)
			 to (9.center)
			 to cycle;
		\draw [style=thick, bend right, looseness=0.75] (5) to (4);
		\draw [style=thick] (12) to (0);
		\draw [style=thick] (0) to (7);
		\draw [style=thick, bend left=360] (2) to (3);
		\draw [style=thick] (1) to (13);
		\draw [style=thick] (1) to (2);
	\end{pgfonlayer}
\end{tikzpicture}}
		\end{aligned}
	\end{equation} 
	is also not algorithmically planar because the vertices
	in the connected component $\{2, 4\}$
	are not in consecutive order.
\end{example}

\begin{definition}
	A $(k,l)$--Brauer diagram is said to be 
	\textbf{algorithmically planar}
	if it is an algorithmically planar 
	$(k,l)$--partition diagram 
	that is also a Brauer diagram.
\end{definition}

\begin{example}
	The $(7,5)$--Brauer diagram
	\begin{equation}
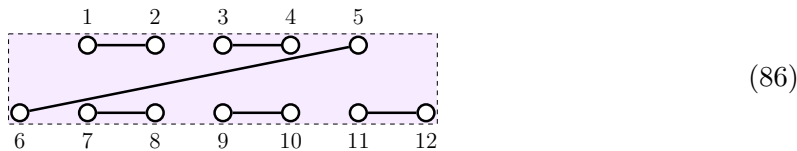
		\begin{aligned}
			\scalebox{0.6}{\begin{tikzpicture}
	\begin{pgfonlayer}{nodelayer}
		\node [style=node] (0) at (0, -1.5) {};
		\node [style=node] (1) at (12, -1.5) {};
		\node [style=node] (2) at (9, -1.5) {};
		\node [style=node] (3) at (6, -1.5) {};
		\node [style=node] (4) at (9, 1.5) {};
		\node [style=node] (5) at (6, 1.5) {};
		\node [style=node] (6) at (3, 1.5) {};
		\node [style=node] (7) at (3, -1.5) {};
		\node [style=none] (8) at (-0.5, 2) {};
		\node [style=none] (9) at (-0.5, -2) {};
		\node [style=none] (10) at (18.5, -2) {};
		\node [style=none] (11) at (18.5, 2) {};
		\node [style=node] (12) at (12, 1.5) {};
		\node [style=node] (13) at (15, -1.5) {};
		\node [style=node] (14) at (15, 1.5) {};
		\node [style=none] (15) at (3, 2.75) {\Large$1$};
		\node [style=none] (16) at (6, 2.75) {\Large$2$};
		\node [style=none] (17) at (9, 2.75) {\Large$3$};
		\node [style=none] (18) at (12, 2.75) {\Large$4$};
		\node [style=none] (19) at (15, 2.75) {\Large$5$};
		\node [style=none] (20) at (0, -2.75) {\Large$6$};
		\node [style=none] (21) at (3, -2.75) {\Large$7$};
		\node [style=none] (22) at (6, -2.75) {\Large$8$};
		\node [style=none] (23) at (9, -2.75) {\Large$9$};
		\node [style=none] (24) at (12, -2.75) {\Large$10$};
		\node [style=none] (25) at (15, -2.75) {\Large$11$};
		\node [style=node] (26) at (18, -1.5) {};
		\node [style=none] (27) at (18, -2.75) {\Large$12$};
	\end{pgfonlayer}
	\begin{pgfonlayer}{edgelayer}
		\draw [style={dash_2}] (8.center)
			 to (11.center)
			 to (10.center)
			 to (9.center)
			 to cycle;
		\draw [style=thick] (6) to (5);
		\draw [style=thick] (4) to (12);
		\draw [style=thick] (7) to (3);
		\draw [style=thick] (2) to (1);
		\draw [style=thick] (13) to (26);
		\draw [style=thick] (0) to (14);
	\end{pgfonlayer}
\end{tikzpicture}}
		\end{aligned}
	\end{equation}
	is algorithmicallly planar, whereas the
	$(7,5)$--Brauer diagram
	\begin{equation}
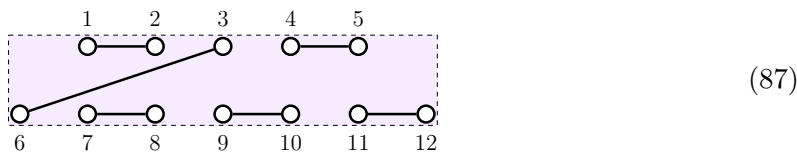
		\begin{aligned}
			\scalebox{0.6}{\begin{tikzpicture}
	\begin{pgfonlayer}{nodelayer}
		\node [style=node] (0) at (0, -1.5) {};
		\node [style=node] (1) at (12, -1.5) {};
		\node [style=node] (2) at (9, -1.5) {};
		\node [style=node] (3) at (6, -1.5) {};
		\node [style=node] (4) at (12, 1.5) {};
		\node [style=node] (5) at (6, 1.5) {};
		\node [style=node] (6) at (3, 1.5) {};
		\node [style=node] (7) at (3, -1.5) {};
		\node [style=none] (8) at (-0.5, 2) {};
		\node [style=none] (9) at (-0.5, -2) {};
		\node [style=none] (10) at (18.5, -2) {};
		\node [style=none] (11) at (18.5, 2) {};
		\node [style=node] (12) at (15, 1.5) {};
		\node [style=node] (13) at (15, -1.5) {};
		\node [style=node] (14) at (9, 1.5) {};
		\node [style=none] (15) at (3, 2.75) {\Large$1$};
		\node [style=none] (16) at (6, 2.75) {\Large$2$};
		\node [style=none] (17) at (9, 2.75) {\Large$3$};
		\node [style=none] (18) at (12, 2.75) {\Large$4$};
		\node [style=none] (19) at (15, 2.75) {\Large$5$};
		\node [style=none] (20) at (0, -2.75) {\Large$6$};
		\node [style=none] (21) at (3, -2.75) {\Large$7$};
		\node [style=none] (22) at (6, -2.75) {\Large$8$};
		\node [style=none] (23) at (9, -2.75) {\Large$9$};
		\node [style=none] (24) at (12, -2.75) {\Large$10$};
		\node [style=none] (25) at (15, -2.75) {\Large$11$};
		\node [style=node] (26) at (18, -1.5) {};
		\node [style=none] (27) at (18, -2.75) {\Large$12$};
	\end{pgfonlayer}
	\begin{pgfonlayer}{edgelayer}
		\draw [style={dash_2}] (8.center)
			 to (11.center)
			 to (10.center)
			 to (9.center)
			 to cycle;
		\draw [style=thick] (6) to (5);
		\draw [style=thick] (4) to (12);
		\draw [style=thick] (7) to (3);
		\draw [style=thick] (2) to (1);
		\draw [style=thick] (13) to (26);
		\draw [style=thick] (0) to (14);
	\end{pgfonlayer}
\end{tikzpicture}}
		\end{aligned}
	\end{equation}
	is not algorithmicallly planar.
\end{example}

\begin{definition}
	An $(l+k) \backslash n$--diagram
	is said to be 
	\textbf{algorithmically planar} 
	if it satisfies the following conditions:
	\begin{itemize}
		\item The free vertices in the bottom row start from the far right hand side and are lined up sequentially.
		\item The free vertices in the top row also start from the far right hand side and are lined up sequentially.
		\item The connected components that are solely in the bottom row are lined up sequentially, starting from the right hand side to the left of the free vertices, such that the vertices in each connected component are in consecutive order.
		\item The connected components that are solely in the top row are lined up sequentially, starting from the far left hand side, such that the vertices in each connected component are in consecutive order.
		\item The connected components between different rows are 
			such that no two subsets cross.
	\end{itemize}
\end{definition}

\begin{example}
	The $(5+6) \backslash 3$--diagram
	\begin{equation}
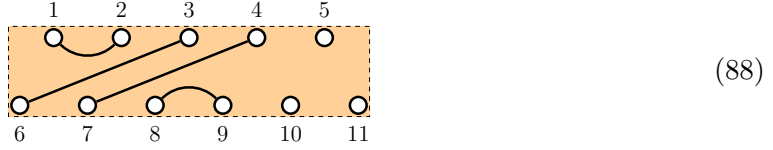
		\begin{aligned}
		\scalebox{0.6}{\begin{tikzpicture}
	\begin{pgfonlayer}{nodelayer}
		\node [style=node] (0) at (0, -1.5) {};
		\node [style=node] (1) at (6, -1.5) {};
		\node [style=node] (2) at (3, -1.5) {};
		\node [style=node] (3) at (12, -1.5) {};
		\node [style=node] (4) at (7.5, 1.5) {};
		\node [style=node] (5) at (4.5, 1.5) {};
		\node [style=node] (6) at (1.5, 1.5) {};
		\node [style=node] (7) at (9, -1.5) {};
		\node [style=none] (8) at (-0.5, 2) {};
		\node [style=none] (9) at (-0.5, -2) {};
		\node [style=none] (10) at (15.5, -2) {};
		\node [style=none] (11) at (15.5, 2) {};
		\node [style=node] (12) at (10.5, 1.5) {};
		\node [style=node] (13) at (15, -1.5) {};
		\node [style=node] (14) at (13.5, 1.5) {};
		\node [style=none] (15) at (1.5, 2.75) {\Large$1$};
		\node [style=none] (16) at (4.5, 2.75) {\Large$2$};
		\node [style=none] (17) at (7.5, 2.75) {\Large$3$};
		\node [style=none] (18) at (10.5, 2.75) {\Large$4$};
		\node [style=none] (19) at (13.5, 2.75) {\Large$5$};
		\node [style=none] (20) at (0, -2.75) {\Large$6$};
		\node [style=none] (21) at (3, -2.75) {\Large$7$};
		\node [style=none] (22) at (6, -2.75) {\Large$8$};
		\node [style=none] (23) at (9, -2.75) {\Large$9$};
		\node [style=none] (24) at (12, -2.75) {\Large$10$};
		\node [style=none] (25) at (15, -2.75) {\Large$11$};
	\end{pgfonlayer}
	\begin{pgfonlayer}{edgelayer}
		\draw [style={dash_4}] (11.center)
			 to (8.center)
			 to (9.center)
			 to (10.center)
			 to cycle;
		\draw [style=thick] (0) to (4);
		\draw [style=thick, bend right=45] (6) to (5);
		\draw [style=thick] (2) to (12);
		\draw [style=thick, bend left=315] (7) to (1);
	\end{pgfonlayer}
\end{tikzpicture}}
		\end{aligned}
	\end{equation}
	is algorithmicallly planar. 
	However, the $(5+6) \backslash 3$--diagram
	\begin{equation}
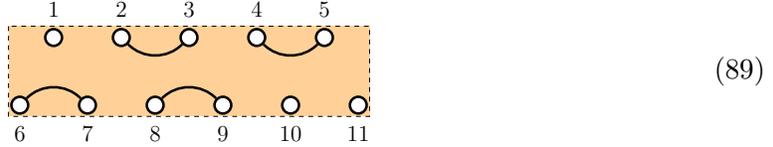
		\begin{aligned}
		\scalebox{0.6}{\begin{tikzpicture}
	\begin{pgfonlayer}{nodelayer}
		\node [style=node] (1) at (6, -1.5) {};
		\node [style=node] (3) at (12, -1.5) {};
		\node [style=node] (4) at (7.5, 1.5) {};
		\node [style=node] (5) at (4.5, 1.5) {};
		\node [style=node] (6) at (1.5, 1.5) {};
		\node [style=node] (7) at (9, -1.5) {};
		\node [style=none] (8) at (-0.5, 2) {};
		\node [style=none] (9) at (-0.5, -2) {};
		\node [style=none] (10) at (15.5, -2) {};
		\node [style=none] (11) at (15.5, 2) {};
		\node [style=node] (12) at (10.5, 1.5) {};
		\node [style=node] (13) at (15, -1.5) {};
		\node [style=node] (14) at (13.5, 1.5) {};
		\node [style=none] (15) at (1.5, 2.75) {\Large$1$};
		\node [style=none] (16) at (4.5, 2.75) {\Large$2$};
		\node [style=none] (17) at (7.5, 2.75) {\Large$3$};
		\node [style=none] (18) at (10.5, 2.75) {\Large$4$};
		\node [style=none] (19) at (13.5, 2.75) {\Large$5$};
		\node [style=none] (20) at (0, -2.75) {\Large$6$};
		\node [style=none] (21) at (3, -2.75) {\Large$7$};
		\node [style=none] (22) at (6, -2.75) {\Large$8$};
		\node [style=none] (23) at (9, -2.75) {\Large$9$};
		\node [style=none] (24) at (12, -2.75) {\Large$10$};
		\node [style=none] (25) at (15, -2.75) {\Large$11$};
		\node [style=node] (26) at (0, -1.5) {};
		\node [style=node] (27) at (3, -1.5) {};
	\end{pgfonlayer}
	\begin{pgfonlayer}{edgelayer}
		\draw [style={dash_4}] (8.center)
			 to (9.center)
			 to (10.center)
			 to (11.center)
			 to cycle;
		\draw [style=thick, bend left=315] (7) to (1);
		\draw [style=thick, bend left=315] (27) to (26);
		\draw [style=thick, bend right=45] (5) to (4);
		\draw [style=thick, bend right=45] (12) to (14);
	\end{pgfonlayer}
\end{tikzpicture}}
		\end{aligned}
	\end{equation}
	is not algorithmically planar because the free vertex labelled $1$ is not at the far right hand side of the top row.
\end{example}

\begin{remark}
	It is clear that an algorithmically planar set partition diagram is \textbf{planar}, that is, none of the connected components in the diagram cross.
\end{remark}


\subsection{Multiplication Algorithm}

In Algorithm 1, we outline a procedure called \textbf{MatrixMult} 
that performs the matrix multiplication of $v$ by a spanning set element in 
$\Hom_{G(n)}((\mathbb{R}^{n})^{\otimes k}, (\mathbb{R}^{n})^{\otimes l})$ 
using the guiding principles that were given in the previous section.
We assume that we have the set partition diagram that is associated with the spanning set element. 
Note that in the description of the algorithm, we have used $d_\pi$ to represent a generic 
$(k,l)$--partition diagram; however, the type of set partition diagrams that we can use as 
input depends entirely upon the group $G(n)$. For example, if $G(n) = O(n)$, 
then only $(k,l)$--Brauer 
diagrams are valid as inputs to the procedure.


\begin{figure*}[tb]
	\begin{tcolorbox}[colback=melon!10, colframe=melon!40, coltitle=black, title={{\bfseries Algorithm 1:} \textbf{MatrixMult}($G(n), d_\pi, v$)}]
	\textbf{Inputs:}
	\begin{itemize}
		\item $G(n)$ is a group.
		\item $d_\pi$ is an appropriate $(k,l)$--partition diagram for $G(n)$.
		\item $v \in (\mathbb{R}^{n})^{\otimes k}$.
	\end{itemize}
		\textbf{Perform the following steps:}
	\begin{enumerate}
		\item $\sigma_k, d_\pi, \sigma_l \leftarrow \textbf{Factor}(G(n), d_\pi)$
		\item $v \leftarrow \textbf{Permute}(v, \sigma_k)$
		\item $w \leftarrow \textbf{PlanarMult}(G(n), d_\pi, v)$
		\item $w \leftarrow \textbf{Permute}(w, \sigma_l)$	
	\end{enumerate}
	\textbf{Output:} $w \in (\mathbb{R}^{n})^{\otimes l}$.
	\end{tcolorbox}
	\label{alg1}
\end{figure*}

We now describe each of the subprocedures that appear in Algorithm 1 in more detail.

\textbf{Factor} takes as input a $(k,l)$--partition diagram that is appropriate for the group $G(n)$.
The procedure uses the string-like property of these diagrams to output three diagrams whose 
composition is equivalent to the original input diagram.
The first is a $(k,k)$--partition diagram that corresponds to a permutation $\sigma_k \in S_k$. 
Specifically, if $i$ is a vertex in the top row, 
then the vertex in the bottom row that is connected to it is $k + \sigma_k(i)$.
The second is a $(k,l)$--partition diagram that is algorithmically planar.
The third is an $(l,l)$--partition diagram 
which can be interpreted as a permutation $\sigma_l \in S_l$ in the same way as the first diagram,
replacing $k$ with $l$.

\textbf{Permute} takes as input a vector
$w \in (\mathbb{R}^{n})^{\otimes m}$,
for some $n, m$, 
that is expressed in the standard basis of $\mathbb{R}^{n}$,
and a permutation $\sigma$ in $S_m$,
and outputs another vector in 
$(\mathbb{R}^{n})^{\otimes m}$ where only the indices of the basis vectors~-- 
and not the indices of the coefficients of $w$~-- have been permuted 
according to $\sigma$.
Said differently, \textbf{Permute} performs the following operation, which is extended linearly:
\begin{equation} \label{permuteop}
	\sigma \cdot w_Ie_I \coloneqq w_I(e_{i_{\sigma(1)}} \otimes e_{i_{\sigma(2)}} \otimes \dots \otimes e_{i_{\sigma(m)}})
\end{equation}
Note that in Algorithm 1, $m$ will be equal to either $k$ or $l$.

\textbf{PlanarMult} takes as input an algorithmically planar set partition diagram of the form returned by \textbf{Factor}
together with a vector, and performs a fast matrix multiplication on this vector.
Since the set partition diagram is algorithmically planar, we first use the monoidal property of 
the category in which it is a morphism to decompose it as a tensor product of smaller set partition diagrams.
Next, we apply the monoidal functor that is appropriate for $G(n)$ to express this tensor product of diagrams 
as a Kronecker product of smaller matrices.
Finally, we perform matrix multiplication by applying these smaller matrices to the input vector from 
``right-to-left, diagram-by-diagram" -- to be described in more detail for each group below -- returning another vector as output.

\begin{remark}
	In effect, the \textbf{Factor} procedure takes as input a $(k,l)$--partition diagram that is appropriate for the group $G(n)$ and swaps as many pairs of vertices as necessary in each row in order to obtain an algorithmically planar $(k,l)$--partition diagram. 
	The composition of the swaps in each row is represented by a permutation diagram. 
\end{remark}


\begin{remark}
	We note that the implementation of the \textbf{Factor} and \textbf{PlanarMult} procedures 
	vary according to the group $G(n)$ and the type of set partition diagrams that correspond to $G(n)$,
	although they share many commonalities.
	We describe the implementation of these procedures 
	for each of the groups below.  
\end{remark}

\begin{remark}
	For each group $G(n)$, we also analyse the time complexity of the implementation of
	\textbf{MatrixMult} and compare it with the time complexity of the naive implementation
	of matrix multiplication.
	Note that in our analysis, we are viewing memory operations, such as permuting basis vectors and making copies of coefficients, as having no cost.
	Hence, we view the procedures \textbf{Factor} and \textbf{Permute} as having no cost. 
	As a result, it is enough to consider the computational cost of \textbf{PlanarMult} only.
\end{remark}


\subsubsection{Symmetric Group $S_n$}

	The implementation of Algorithm 1 that we give here for the symmetric group effectively recovers 
	the algorithm that was presented in \citet[Appendix C]{godfrey}; however, we take an entirely different approach 
	that uses monoidal categories instead.
	The major difference in implementation between the two versions relates to how, in our terminology,
	the connected components between vertices in different rows of $d_\pi$ are pulled into the middle diagram in \textbf{Factor}.
	We choose to make sure that the connected components do not cross, making the resulting diagram algorithmically planar, 
	whereas in \citet{godfrey} they choose to connect the vertices in different rows such that the 
	left-most vertices in the top row of the new diagram connect to the right-most vertices in the bottom row, that is,
	they make the connected components cross in ``opposites".
	While this does not make any significant difference in terms of performing the matrix multiplication for the symmetric group --
	indeed, there is only a difference in how the tensor indices are ordered~--
	our decision to make the middle diagram in the composition algorithmically planar leads to significant performance improvements when we extend 
	our approach to the other groups below, since for these groups, we will show that these operations reduce to the identity transformation.

\paragraph{Implementation}


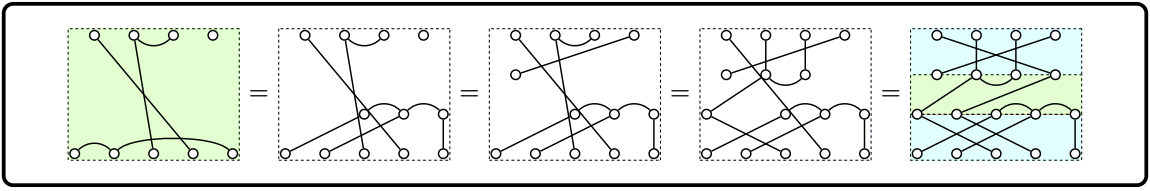
\begin{figure*}[tb]
	\begin{tcolorbox}[colback=white!02, colframe=black]
	\begin{center}
		\scalebox{0.35}{\begin{tikzpicture}
	\begin{pgfonlayer}{nodelayer}
		\node [style=none] (105) at (63.5, -6) {};
		\node [style=none] (106) at (76.5, -6) {};
		\node [style=none] (108) at (63.5, -3) {};
		\node [style=none] (109) at (76.5, -3) {};
		\node [style=none] (0) at (14, -4.5) {\Huge$\bm{=}$};
		\node [style=node] (1) at (3, -9) {};
		\node [style=none] (2) at (-0.5, 0.5) {};
		\node [style=none] (3) at (-0.5, -9.5) {};
		\node [style=none] (4) at (12.5, -9.5) {};
		\node [style=none] (5) at (12.5, 0.5) {};
		\node [style=node] (6) at (1.5, 0) {};
		\node [style=node] (8) at (12, -9) {};
		\node [style=node] (9) at (0, -9) {};
		\node [style=node] (16) at (6, -9) {};
		\node [style=node] (17) at (9, -9) {};
		\node [style=node] (18) at (7.5, 0) {};
		\node [style=node] (19) at (4.5, 0) {};
		\node [style=node] (20) at (28, -6) {};
		\node [style=none] (21) at (15.5, 0.5) {};
		\node [style=none] (22) at (15.5, -9.5) {};
		\node [style=none] (23) at (28.5, -9.5) {};
		\node [style=none] (24) at (28.5, 0.5) {};
		\node [style=node] (25) at (17.5, 0) {};
		\node [style=node] (26) at (28, -9) {};
		\node [style=node] (27) at (25, -6) {};
		\node [style=node] (28) at (22, -9) {};
		\node [style=node] (29) at (25, -9) {};
		\node [style=node] (33) at (19, -9) {};
		\node [style=none] (37) at (30, -4.5) {\Huge$\bm{=}$};
		\node [style=node] (38) at (35, -9) {};
		\node [style=none] (39) at (31.5, 0.5) {};
		\node [style=none] (40) at (31.5, -9.5) {};
		\node [style=none] (41) at (44.5, -9.5) {};
		\node [style=none] (42) at (44.5, 0.5) {};
		\node [style=node] (43) at (33.5, 0) {};
		\node [style=node] (44) at (44, -9) {};
		\node [style=node] (45) at (32, -9) {};
		\node [style=node] (46) at (38, -9) {};
		\node [style=node] (47) at (41, -9) {};
		\node [style=node] (51) at (44, -6) {};
		\node [style=node] (52) at (41, -6) {};
		\node [style=none] (54) at (46, -4.5) {\Huge$\bm{=}$};
		\node [style=none] (56) at (47.5, 0.5) {};
		\node [style=none] (57) at (47.5, -9.5) {};
		\node [style=none] (58) at (60.5, -9.5) {};
		\node [style=none] (59) at (60.5, 0.5) {};
		\node [style=node] (60) at (49.5, 0) {};
		\node [style=node] (71) at (58.5, 0) {};
		\node [style=node] (72) at (49.5, -3) {};
		\node [style=node] (73) at (39.5, 0) {};
		\node [style=node] (74) at (36.5, 0) {};
		\node [style=node] (75) at (23.5, 0) {};
		\node [style=node] (76) at (20.5, 0) {};
		\node [style=node] (77) at (52.5, 0) {};
		\node [style=node] (78) at (55.5, 0) {};
		\node [style=none] (79) at (62, -4.5) {\Huge$\bm{=}$};
		\node [style=node] (80) at (67, -9) {};
		\node [style=none] (82) at (63.5, -9.5) {};
		\node [style=none] (83) at (76.5, -9.5) {};
		\node [style=node] (85) at (65.5, 0) {};
		\node [style=node] (86) at (76, -9) {};
		\node [style=node] (87) at (64, -9) {};
		\node [style=node] (88) at (70, -9) {};
		\node [style=node] (89) at (73, -9) {};
		\node [style=node] (90) at (64, -6) {};
		\node [style=node] (91) at (76, -6) {};
		\node [style=node] (92) at (73, -6) {};
		\node [style=node] (93) at (70, -6) {};
		\node [style=node] (94) at (71.5, -3) {};
		\node [style=node] (95) at (68.5, -3) {};
		\node [style=node] (96) at (68.5, 0) {};
		\node [style=node] (97) at (71.5, 0) {};
		\node [style=node] (98) at (65.5, -3) {};
		\node [style=node] (99) at (67, -6) {};
		\node [style=none] (102) at (63.5, -6) {};
		\node [style=none] (103) at (76.5, -6) {};
		\node [style=none] (104) at (63.5, -3) {};
		\node [style=none] (107) at (76.5, -3) {};
		\node [style=none] (110) at (76.5, 0.5) {};
		\node [style=none] (111) at (63.5, 0.5) {};
		\node [style=node] (112) at (51, -9) {};
		\node [style=node] (113) at (60, -9) {};
		\node [style=node] (114) at (48, -9) {};
		\node [style=node] (115) at (54, -9) {};
		\node [style=node] (116) at (57, -9) {};
		\node [style=node] (118) at (60, -6) {};
		\node [style=node] (119) at (57, -6) {};
		\node [style=node] (120) at (10.5, 0) {};
		\node [style=node] (121) at (26.5, 0) {};
		\node [style=node] (122) at (16, -9) {};
		\node [style=node] (123) at (42.5, 0) {};
		\node [style=node] (124) at (33.5, -3) {};
		\node [style=node] (125) at (55.5, -3) {};
		\node [style=node] (126) at (52.5, -3) {};
		\node [style=node] (127) at (48, -6) {};
		\node [style=node] (128) at (74.5, 0) {};
		\node [style=node] (129) at (74.5, -3) {};
		\node [style=node] (130) at (54, -6) {};
		\node [style=node] (131) at (38, -6) {};
		\node [style=node] (132) at (22, -6) {};
		\node [style=none] (133) at (63.5, -3) {};
		\node [style=none] (134) at (63.5, -6) {};
		\node [style=none] (135) at (76.5, -3) {};
		\node [style=none] (136) at (76.5, -6) {};
	\end{pgfonlayer}
	\begin{pgfonlayer}{edgelayer}
		\draw [style={dash_2}] (22.center) to (23.center);
		\draw [style={dash_2}] (23.center) to (24.center);
		\draw [style={dash_2}] (24.center) to (21.center);
		\draw [style={dash_2}] (21.center) to (22.center);
		\draw [style=thick] (29) to (25);
		\draw [style={dash_2}] (40.center) to (41.center);
		\draw [style={dash_2}] (41.center) to (42.center);
		\draw [style={dash_2}] (42.center) to (39.center);
		\draw [style={dash_2}] (39.center) to (40.center);
		\draw [style=thick] (47) to (43);
		\draw [style={dash_2}] (59.center) to (56.center);
		\draw [style={dash_2}] (56.center) to (57.center);
		\draw [style={dash_2}] (57.center) to (58.center);
		\draw [style={dash_2}] (58.center) to (59.center);
		\draw [style=thick, bend right=45] (74) to (73);
		\draw [style=thick, bend right=45] (76) to (75);
		\draw [style={dash_5}] (110.center)
			 to (111.center)
			 to (108.center)
			 to (109.center)
			 to cycle;
		\draw [style={dash_5}] (103.center)
			 to (102.center)
			 to (82.center)
			 to (83.center)
			 to cycle;
		\draw [style=thick, bend left=45] (27) to (20);
		\draw [style=thick, bend left=45] (52) to (51);
		\draw [style=thick, bend left=45] (119) to (118);
		\draw [style=thick] (60) to (116);
		\draw [style={dash_4}] (107.center) to (104.center);
		\draw [style={dash_4}] (105.center) to (106.center);
		\draw [style=thick, bend right=45] (126) to (125);
		\draw [style=thick] (76) to (28);
		\draw [style=thick] (124) to (123);
		\draw [style=thick] (74) to (46);
		\draw [style=thick] (126) to (77);
		\draw [style=thick] (78) to (125);
		\draw [style=thick] (72) to (71);
		\draw [style=thick] (127) to (126);
		\draw [style=thick] (115) to (127);
		\draw [style=thick] (95) to (96);
		\draw [style=thick] (94) to (97);
		\draw [style=thick] (129) to (85);
		\draw [style=thick] (98) to (128);
		\draw [style=thick] (88) to (90);
		\draw [style=thick] (89) to (99);
		\draw [style=thick] (87) to (93);
		\draw [style=thick] (80) to (92);
		\draw [style=thick] (86) to (91);
		\draw [style=thick] (114) to (130);
		\draw [style=thick] (112) to (119);
		\draw [style=thick] (113) to (118);
		\draw [style=thick, bend left=45] (130) to (119);
		\draw [style=thick, bend left=45] (131) to (52);
		\draw [style=thick] (44) to (51);
		\draw [style=thick] (38) to (52);
		\draw [style=thick] (45) to (131);
		\draw [style=thick, bend left=45] (132) to (27);
		\draw [style=thick] (122) to (132);
		\draw [style=thick] (33) to (27);
		\draw [style=thick] (26) to (20);
		\draw [style={dash_3}] (5.center)
			 to (2.center)
			 to (3.center)
			 to (4.center)
			 to cycle;
		\draw [style=thick, bend right=45] (19) to (18);
		\draw [style=thick, bend left=315] (1) to (9);
		\draw [style=thick] (6) to (17);
		\draw [style=thick] (19) to (16);
		\draw [style=thick, bend left=45, looseness=0.50] (1) to (8);
		\draw [style={dash_3}] (136.center)
			 to (135.center)
			 to (133.center)
			 to (134.center)
			 to cycle;
		\draw [style=thick, bend left=45] (92) to (91);
		\draw [style=thick, bend right=45] (95) to (94);
		\draw [style=thick] (90) to (95);
		\draw [style=thick] (99) to (129);
		\draw [style=thick, bend left=315] (92) to (93);
	\end{pgfonlayer}
\end{tikzpicture}}
	\end{center}
	\end{tcolorbox}
	\caption{
		We use the string-like property of $(k,l)$--partition diagrams to \textbf{Factor} them 
	as a composition of a permutation in $S_k$, an algorithmically planar $(k,l)$--partition diagram, 
	and a permutation in $S_l$. Here, $k=5$ and $l=4$.
			}
	\label{symmfactoringSn}
\end{figure*}


We provide specific implementations of \textbf{Factor} and \textbf{PlanarMult}
for the symmetric group $S_n$.

\textbf{Factor}: The input is a $(k,l)$--partition diagram $d_\pi$. 
We drag and bend the strings representing the connected components of 
$d_\pi$ to obtain three diagrams whose composition is equivalent to $d_\pi$:
a $(k,k)$--partition diagram that represents a permutation $\sigma_k$ in the symmetric group $S_k$; 
an algorithmically planar
$(k,l)$--partition diagram; 
and a $(l,l)$--partition diagram that represents a permutation $\sigma_l$ in the symmetric group $S_l$.

To obtain the algorithmically planar $(k,l)$--partition diagram, 
we drag and bend the strings in any way such that 
\begin{itemize}
	\item the connected components that are solely in the bottom row of $d_\pi$ are pulled up to be next to each other
		in the far right hand side of the bottom row of the algorithmically planar $(k,l)$--partition diagram,
		such that the connected components are ordered by their size, in decreasing order from right to left,
	\item the connected components that are solely in the top row of $d_\pi$ are pulled down to be next to each other
		in the far left hand side of the top row of the algorithmically planar $(k,l)$--partition diagram, and
	\item the connected components between vertices in different rows of $d_\pi$ are bent to be in between the other vertices of the algorithmically planar $(k,l)$--partition diagram such that no two subsets of connected components cross.
\end{itemize}
We give an example of this procedure in Figure~\ref{symmfactoringSn}.

\textbf{PlanarMult}: 
	We take as input the algorithmically planar $(k,l)$--partition diagram $d_\pi$ having at most $n$ blocks 
	that is the output of \textbf{Factor}, and a vector $v \in (\mathbb{R}^{n})^{\otimes k}$ that is the output 
	of \textbf{Permute}, as per Algorithm 1.
	
	Let $t$ be the number of blocks that are solely in the top row of $d_\pi$,
	let $d$ be the number of blocks that connect vertices in different rows of $d_\pi$, and
	let $b$ be the number of blocks that are solely in the bottom row of $d_\pi$.

	Given how \textbf{Factor} constructs the algorithmically planar $(k,l)$--partition diagram $d_\pi$, 
	the set partition $\pi$ corresponding to $d_\pi$ will be of the form
	\begin{equation} \label{symmpartfactor}
		\left(\bigcup_{i = 1}^{t} T_i \right)
		\bigcup
		\left(\bigcup_{i = 1}^{d} D_i \right)
		\bigcup
		\left(\bigcup_{i = 1}^{b} B_i \right)
	\end{equation}
	such that
	\begin{equation}
		|B_1| \leq |B_2| \leq \dots \leq |B_b|
	\end{equation}
	where we have used $T_i$ to refer to a top row block, $D_i$ to refer to a different row block, and $B_i$ to refer to a bottom row block.
	If we let
	\begin{equation}
		D_i \coloneqq D_i^{U} \cup D_i^{L}
	\end{equation}
	where $D_i^{U}$ is the subset of $D_i$ whose vertices are in the top row of $d_\pi$, and
	$D_i^{L}$ is the subset of $D_i$ whose vertices are in the bottom row of $d_\pi$,
	then we have that the vertices in each subset are given by
	\begin{itemize}
		\item $T_i = \left\{
				\sum_{j=1}^{i-1} |T_j| + 1, \dots, \sum_{j=1}^{i} |T_j|
			\right\}$ for all $i = 1 \rightarrow t$
		\item $D_i^{U} = \left\{
				\sum_{j=1}^{t} |T_j| + \sum_{j=1}^{i-1} |D_j^{U}| + 1, \dots, 
		\sum_{j=1}^{t} |T_j| + \sum_{j=1}^{i} |D_j^{U}|
			\right\}$ for all $i = 1 \rightarrow d$,
		\item $D_i^{L} = \left\{
				l + \sum_{j=1}^{i-1} |D_j^{L}| + 1, \dots, 
			l + \sum_{j=1}^{i} |D_j^{L}|
			\right\}$ for all $i = 1 \rightarrow d$, and
		\item $B_i = \left\{
			l + \sum_{j=1}^{d} |D_j^{L}| + 
			\sum_{j=1}^{i-1} |B_j| + 1, \dots, 
			l + \sum_{j=1}^{d} |D_j^{L}| +
			\sum_{j=1}^{i} |B_j|
			\right\}$ for all $i = 1 \rightarrow b$.
	\end{itemize}
	Note, in particular, that
	\begin{equation} \label{symmupper}
		l = \sum_{j = 1}^{t} |T_j| + \sum_{j = 1}^{d} |D_j^{U}|
	\end{equation}
	and
	\begin{equation} \label{symmlower}
		k = \sum_{j = 1}^{d} |D_j^{L}| + \sum_{j = 1}^{b} |B_j|
	\end{equation}
	Next, we take $d_\pi$ and express it as a tensor product of three types of set partition diagrams.
	The right-most type is itself a tensor product of diagrams corresponding to the $B_i$; 
	the middle type is a single diagram corresponding to $\bigcup_{i = 1}^{d} D_i$; 
	and the left-most type is itself a tensor product of diagrams corresponding to the $T_i$.

	We apply the monoidal functor $\Theta$ to this tensor product decomposition of diagrams, which returns a Kronecker product of matrices.
	We perform the matrix multiplication by applying the matrices ``right-to-left, diagram-by-diagram", as follows.

	\textbf{Step 1: Apply each matrix corresponding to a bottom row block diagram, one-by-one, starting from the one that corresponds to $B_b$ and ending with the one that corresponds to $B_1$.}

	Suppose that we are performing the part of the matrix multiplication that corresponds to $B_i$, for some $i = 1 \rightarrow b$.
	The input will be a vector
	$w \in (\mathbb{R}^{n})^{\otimes k - \sum_{j = i+1}^{b} |B_j|}$.

	We can express $w$ in the standard basis of $\mathbb{R}^{n}$ as
	\begin{equation}
		w = \sum_{L \in [n]^{k - \sum_{j = i+1}^{b} |B_j|}}
			w_Le_L
	\end{equation}

	This will be mapped to the vector
	$r \in (\mathbb{R}^{n})^{\otimes k - \sum_{j = i}^{b} |B_j|}$,
	where $r$ is of the form
	\begin{equation} \label{rMSncoeff}
		r = \sum_{M \in [n]^{k - \sum_{j = i}^{b} |B_j|}} r_Me_M
	\end{equation}
	and
	\begin{equation} \label{indicesj}
		r_M = \sum_{j \in [n]} w_{M,j, \dots, j}
	\end{equation}
	where the number of indices $j$ in (\ref{indicesj}) is $|B_i|$.

	At the end of this process, we obtain a vector in 
	$(\mathbb{R}^{n})^{\otimes k - \sum_{j = 1}^{b} |B_j|}$.
	
	Note that the matrices corresponding to bottom row 
	connected components
	are merely performing indexing and summation operations, that is, ultimately, tensor contractions.

	\textbf{Step 2: Now apply the matrix corresponding to the middle diagram, that is, to the set $\bigcup_{i = 1}^{d} D_i$.}
	
	The input will be a vector
	$w \in 
	(\mathbb{R}^{n})^{\otimes k - \sum_{j = 1}^{b} |B_j|}$.
	Using (\ref{symmlower}), we can express $w$ in the standard basis of $\mathbb{R}^{n}$ as
	\begin{equation}
		w = 
		\sum_{j_1, \dots, j_d \in [n]} 
		w_{j_1, \dots, j_1, j_2, \dots, j_2, \dots, j_d, \dots, j_d}
		\bigotimes_{k = 1}^{d}
		\left(
		\bigotimes_{p = 1}^{|D_k^{L}|} e_{j_k}
		\right)
	\end{equation}
	where each index $j_k$ in the coefficient is repeated $|D_k^{L}|$ times.
	
	This will be mapped to the vector 
	$r \in (\mathbb{R}^{n})^{\otimes \sum_{j = i}^{d} |D_j^{U}|}$,
	where $r$ is of the form
	\begin{equation}
		r = 
		\sum_{j_1, \dots, j_d \in [n]} 
		r_{j_1, \dots, j_1, j_2, \dots, j_2, \dots, j_d, \dots, j_d}
		\bigotimes_{k = 1}^{d}
		\left(
		\bigotimes_{p = 1}^{|D_k^{U}|} e_{j_k}
		\right)
	\end{equation}
	where each index $j_k$ in the coefficient is repeated $|D_k^{U}|$ times, and
	\begin{equation} \label{transfercoeffSn}
		r_{j_1, \dots, j_1, j_2, \dots, j_2, \dots, j_d, \dots, j_d}
		=
		w_{j_1, \dots, j_1, j_2, \dots, j_2, \dots, j_d, \dots, j_d}
	\end{equation}
	These operations are called transfer operations, which is a term that first appeared in \cite{pan22}.

	\textbf{Step 3: Finally, apply each matrix corresponding to a top row block diagram, one-by-one, starting from the one that corresponds to $T_t$ and ending with the one that corresponds to $T_1$.}
	
	Suppose that we are performing the part of the matrix multiplication that corresponds to $T_i$, for some $i = 1 \rightarrow t$.

	Then we begin with a vector  
	$x \in (\mathbb{R}^{n})^{\otimes l - \sum_{j=1}^{i} |T_j|}$
	that is of the form
	\begin{equation}
		x = 
		\sum_{l_{i+1}, \dots, l_{t} \in [n]}
		\sum_{j_1, \dots, j_d \in [n]} 
		x_{j_1, j_2, \dots, j_d}
		\bigotimes_{q = i+1}^{t}
		\left(
		\bigotimes_{m = 1}^{|T_q|} e_{l_q}
		\right)
		\bigotimes_{k = 1}^{d}
		\left(
		\bigotimes_{p = 1}^{|D_k^{U}|} e_{j_k}
		\right)
	\end{equation}
	where $x_{j_1, j_2, \dots, j_d}$ is the coefficient 
	$r_{j_1, \dots, j_1, j_2, \dots, j_2, \dots, j_d, \dots, j_d}$
	appearing in (\ref{transfercoeffSn}).
	
	This will be mapped to the vector
	$y \in (\mathbb{R}^{n})^{\otimes l - \sum_{j=1}^{i-1} |T_j|}$,
	where $y$ is of the form
	\begin{equation}
		y = 
		\sum_{l_i, l_{i+1}, \dots, l_{t} \in [n]}
		\sum_{j_1, \dots, j_d \in [n]} 
		x_{j_1, j_2, \dots, j_d}
		\bigotimes_{q = i}^{t}
		\left(
		\bigotimes_{m = 1}^{|T_q|} e_{l_q}
		\right)
		\bigotimes_{k = 1}^{d}
		\left(
		\bigotimes_{p = 1}^{|D_k^{U}|} e_{j_k}
		\right)
	\end{equation}
	At the end of this process, we obtain a vector in 
	$(\mathbb{R}^{n})^{\otimes l}$, by (\ref{symmupper}),
	which is of the form
	\begin{equation}
		\sum_{l_1, \dots, l_{t} \in [n]}
		\sum_{j_1, \dots, j_d \in [n]} 
		x_{j_1, j_2, \dots, j_d}
		\bigotimes_{q = 1}^{t}
		\left(
		\bigotimes_{m = 1}^{|T_q|} e_{l_q}
		\right)
		\bigotimes_{k = 1}^{d}
		\left(
		\bigotimes_{p = 1}^{|D_k^{U}|} e_{j_k}
		\right)
	\end{equation}
	This is the vector that is returned by \textbf{PlanarMult} for the symmetric group $S_n$.

	To summarise, the implementation of \textbf{PlanarMult} for the symmetric group $S_n$ takes 
	as input
	the algorithmically planar $(k,l)$--partition diagram 
	that comes from \textbf{Factor} and expresses it as a tensor product of three types of set partition diagrams.
	The right-most type is itself a tensor product of set partition diagrams having only vertices in the bottom row, 
	where each diagram represents a single block.
	These diagrams correspond to tensor contraction operations under the functor $\Theta$.
	The middle type is a set partition diagram that consists of all of the connected components in the 
	algorithmically planar $(k,l)$--partition diagram between vertices in different rows.
	These diagrams correspond to transfer operations under the functor $\Theta$.
	The left-most type is itself a tensor product of set partition diagrams having only vertices in the top row, 
	where each diagram represents a single block.
	These diagrams correspond to indexing operations that perform copies under the functor $\Theta$.

	In 
	Figure \ref{tensorproddecompSn}, we present
	the 
	tensor product decomposition 
	of the algorithmically planar $(5,4)$--partition diagram 
	that is shown in Figure \ref{symmfactoringSn}.


\begin{figure*}[tb]
	\begin{tcolorbox}[colback=white!02, colframe=black]
	\begin{center}
		\scalebox{0.6}{\begin{tikzpicture}
	\begin{pgfonlayer}{nodelayer}
		\node [style=node] (0) at (0, -3) {};
		\node [style=node] (1) at (12, -3) {};
		\node [style=node] (2) at (9, -3) {};
		\node [style=node] (3) at (5.75, -3) {};
		\node [style=node] (4) at (7.5, 0) {};
		\node [style=node] (5) at (4.5, 0) {};
		\node [style=node] (6) at (1.5, 0) {};
		\node [style=node] (7) at (3, -3) {};
		\node [style=none] (10) at (-0.5, 0.5) {};
		\node [style=none] (11) at (-0.5, -3.5) {};
		\node [style=none] (12) at (12.5, -3.5) {};
		\node [style=none] (13) at (12.5, 0.5) {};
		\node [style=none] (14) at (14, -1.5) {\Huge$\bm{=}$};
		\node [style=node] (15) at (16, 0) {};
		\node [style=node] (16) at (31, -3) {};
		\node [style=node] (17) at (28, -3) {};
		\node [style=node] (18) at (25, -3) {};
		\node [style=node] (19) at (22, 0) {};
		\node [style=node] (20) at (19, 0) {};
		\node [style=node] (21) at (25, 0) {};
		\node [style=none] (23) at (15.5, -3.5) {};
		\node [style=none] (24) at (15.5, 0.5) {};
		\node [style=none] (25) at (16.5, 0.5) {};
		\node [style=none] (26) at (16.5, -3.5) {};
		\node [style=none] (35) at (25.5, -3.5) {};
		\node [style=none] (36) at (18.5, -3.5) {};
		\node [style=none] (37) at (25.5, 0.5) {};
		\node [style=none] (38) at (18.5, 0.5) {};
		\node [style=node] (39) at (10.5, 0) {};
		\node [style=node] (40) at (19, -3) {};
		\node [style=node] (41) at (34, -3) {};
		\node [style=none] (42) at (34.5, -3.5) {};
		\node [style=none] (43) at (27.5, -3.5) {};
		\node [style=none] (44) at (34.5, 0.5) {};
		\node [style=none] (45) at (27.5, 0.5) {};
	\end{pgfonlayer}
	\begin{pgfonlayer}{edgelayer}
		\draw [style={dash_3}] (10.center)
			 to (11.center)
			 to (12.center)
			 to (13.center)
			 to cycle;
		\draw [style=thick, bend left=45] (2) to (1);
		\draw [style=thick, bend right=45] (5) to (4);
		\draw [style=thick] (0) to (5);
		\draw [style=thick] (7) to (39);
		\draw [style=thick, bend left=45] (3) to (2);
		\draw [style={dash_3}] (25.center)
			 to (24.center)
			 to (23.center)
			 to (26.center)
			 to cycle;
		\draw [style={dash_3}] (37.center)
			 to (38.center)
			 to (36.center)
			 to (35.center)
			 to cycle;
		\draw [style=thick] (18) to (21);
		\draw [style=thick, bend right=45] (20) to (19);
		\draw [style=thick] (40) to (20);
		\draw [style={dash_3}] (44.center)
			 to (45.center)
			 to (43.center)
			 to (42.center)
			 to cycle;
		\draw [style=thick, bend left=45] (17) to (16);
		\draw [style=thick, bend left=45] (16) to (41);
	\end{pgfonlayer}
\end{tikzpicture}}
	\end{center}
	\end{tcolorbox}
	\caption{The decomposition of the algorithmically planar $(5,4)$--partition diagram that appears in Figure \ref{symmfactoringSn} into a tensor product of smaller partition diagrams. 
	These diagrams correspond, from right-to-left, to tensor contraction, transfer, and copying operations under the functor $\Theta$.
	This tensor product decomposition is used in \textbf{PlanarMult} for the symmetric group $S_n$.}
	\label{tensorproddecompSn}	
\end{figure*}
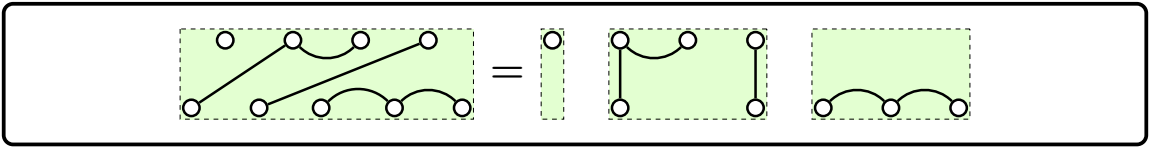


To show diagrammatically how the matrix multiplication is performed,
we see that the tensor product of the three types of set partition diagrams
corresponds to a Kronecker product of smaller matrices under the functor $\Theta$, defined in Theorem \ref{partfunctor},
by the monoidal property of $\Theta$.
We would like to apply each smaller matrix
to the input vector 
from right-to-left, diagram-by-diagram.
To do this, we first deform the entire tensor product decomposition of set partition diagrams by pulling each individual diagram up one level higher than the previous one, going from right-to-left,
and then apply the functor $\Theta$ at each level. 
The newly inserted strings correspond to an identity matrix, hence only the matrices corresponding to the original tensor product decomposition act on the input vector at each stage. 

Figure \ref{planarmultSn} gives an example
of how the computation takes place at each stage
for the tensor product decomposition given in Figure \ref{tensorproddecompSn}, 
using its equivalent diagram form. 


\begin{figure*}[tb]
	\begin{tcolorbox}[colback=white!02, colframe=black]
	\begin{center}
		\scalebox{0.4}{\input{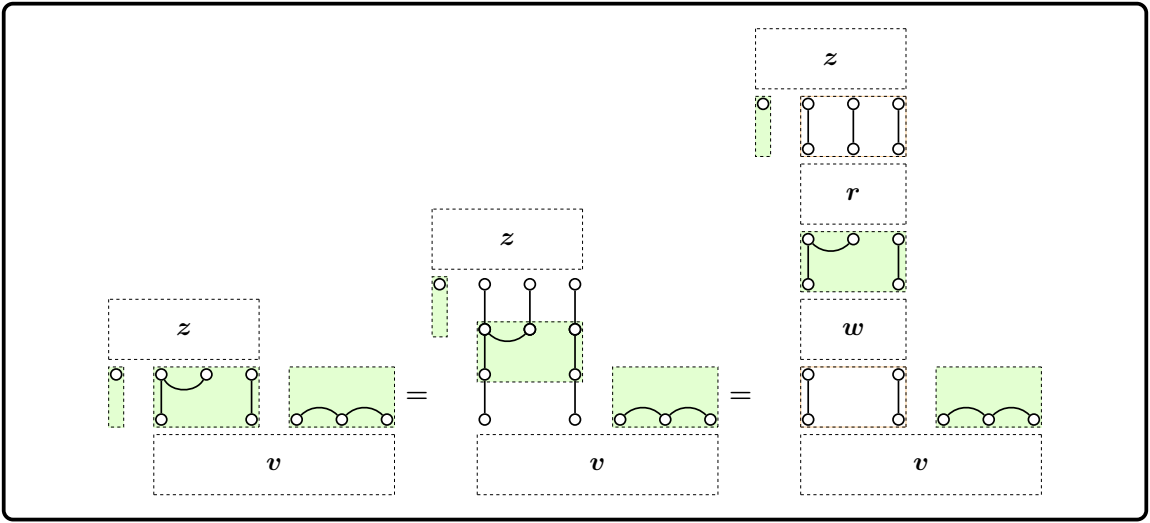}}
	\end{center}
	\end{tcolorbox}
	\caption{We show how matrix multiplication is implemented in \textbf{PlanarMult} for $S_n$ using the tensor product decomposition of the algorithmically planar 
	$(5,4)$--partition diagram given in Figure \ref{tensorproddecompSn} as an example.
	We perform the matrix multiplication as follows:
	first, we deform the entire tensor product decomposition diagram by pulling each individual diagram up one level higher than the previous one, going from right-to-left, and then we apply the functor $\Theta$ at each level. 
	Finally, we perform matrix multiplication at each level to obtain the final output vector.}
	\label{planarmultSn}
\end{figure*}


\begin{example}
	Suppose that we wish to perform the multiplication of $D_\pi$
	by $v \in (\mathbb{R}^{n})^{\otimes 5}$, where
	$D_\pi$ corresponds to the $(5,4)$--partition diagram $d_\pi$ 
	given in Figure \ref{symmfactoringSn} under $\Theta$,
	and $v$ is given by
		\begin{equation}
			\sum_{L \in [n]^5} v_Le_L
		\end{equation}
	We know that $D_\pi$ is a matrix in
	$\Hom_{S_n}((\mathbb{R}^{n})^{\otimes 5}, (\mathbb{R}^{n})^{\otimes 4})$.

	First, we apply the procedure \textbf{Factor}, which returns the three diagrams given in Figure~\ref{symmfactoringSn}.
	The first diagram corresponds to the permutation $(13)(24)$ in $S_5$; hence, the result of \textbf{Permute}$(v, (13)(24))$ is the vector
		\begin{equation}
			\sum_{L \in [n]^5} 
			v_{l_1, l_2, l_3, l_4, l_5}
			\left(
				e_{l_3} \otimes 
				e_{l_4} \otimes 
				e_{l_1} \otimes
				e_{l_2} \otimes
				e_{l_5} 
			\right)
		\end{equation}
	Next we apply \textbf{PlanarMult} with the decomposition given in Figure \ref{tensorproddecompSn}.

	Step 1: Apply the matrices that correspond to the bottom row blocks.

	We obtain the vector
		\begin{equation}
			w = 
			\sum_{l_3, l_4 \in [n]} 
			w_{l_3, l_4}
			\left(
				e_{l_3} \otimes 
				e_{l_4}
			\right)
		\end{equation}
	where
		\begin{equation}
			w_{l_3, l_4}
			=
			\sum_{j \in [n]}
			v_{j, j, l_3, l_4, j}
		\end{equation}

	Step 2: Apply the matrices that correspond to the middle diagram.

	We obtain the vector
		\begin{equation}
			r = 
			\sum_{l_3 \in [n]}
			\sum_{l_4 \in [n]}
			r_{l_3, l_3, l_4}
			\left(	
				e_{l_3} \otimes 
				e_{l_3} \otimes 
				e_{l_4}
			\right)
		\end{equation}
	where
		\begin{equation}
			r_{l_3, l_3, l_4}
			=
			w_{l_3, l_4}
		\end{equation}	

	Step 3: Apply the matrices that correspond to the top row blocks.

	We obtain the vector
		\begin{equation}
			z =
			\sum_{m \in [n]}	
			\sum_{l_3 \in [n]}
			\sum_{l_4 \in [n]}
			r_{l_3, l_3, l_4}
			\left(
				e_{m} \otimes 
				e_{l_3} \otimes 
				e_{l_3} \otimes 
				e_{l_4}
			\right)
		\end{equation}
	Substituting in, we get that
		\begin{equation}
			z =
			\sum_{m \in [n]}	
			\sum_{l_3 \in [n]}
			\sum_{l_4 \in [n]}
			w_{l_3, l_4}
			\left(
				e_{m} \otimes 
				e_{l_3} \otimes 
				e_{l_3} \otimes 
				e_{l_4}
			\right)
		\end{equation}
	and hence
		\begin{equation}
			z =
			\sum_{m \in [n]}	
			\sum_{l_3 \in [n]}
			\sum_{l_4 \in [n]}
			\sum_{j \in [n]}
			v_{j, j, l_3, l_4, j}
			\left(
				e_{m} \otimes 
				e_{l_3} \otimes 
				e_{l_3} \otimes 
				e_{l_4}
			\right)
		\end{equation}
	Finally, as the third diagram returned from \textbf{Factor} 
	corresponds to the permutation $(14)$ in $S_4$,
	we perform \textbf{Permute}$(z, (14))$,
	which returns the vector
		\begin{equation}
			z =
			\sum_{m \in [n]}	
			\sum_{l_3 \in [n]}
			\sum_{l_4 \in [n]}
			\sum_{j \in [n]}
			v_{j, j, l_3, l_4, j}
			\left(
				e_{l_4} \otimes 
				e_{l_3} \otimes 
				e_{l_3} \otimes 
				e_{m}
			\right)
		\end{equation}
	This is the vector that is returned by \textbf{MatrixMult}.	
\end{example}

\paragraph{Time Complexity}

We look at each step of the implementation that was given above.

	\textbf{Step 1:} For the part of the matrix multiplication that corresponds to $B_i$, for some $i = 1 \rightarrow b$, we map a vector in 
	$(\mathbb{R}^{n})^{\otimes k - \sum_{j = i+1}^{b} |B_j|}$
	to a vector in
	$(\mathbb{R}^{n})^{\otimes k - \sum_{j = i}^{b} |B_j|}$.

	Since the matrix corresponds to a bottom row block, 
	for each tuple $M$ of indices in the output coefficient $r_M$, 
	as in (\ref{indicesj}),
	there are only $n$ terms to multiply (an improvement over $n^{|B_i|}$), 
	and consequently only $n-1$ additions.

	Hence, in total, there are
	\begin{equation} \label{Snmultalgo}
		\sum_{i = 1}^{b}
		n^{k - \sum_{j = b+1-i}^{b} |B_j|} \cdot n
	\end{equation}
	multiplications and
	\begin{equation} \label{Snaddalgo}
		\sum_{i = 1}^{b}
		n^{k - \sum_{j = b+1-i}^{b} |B_j|} \cdot (n-1)
	\end{equation}
	additions.
	Note that, in calculating the overall time complexity for this step, 
	the highest order term is given by
	the size of $B_b$.
	Hence, the worst case occurs when $|B_b| = 1$, giving
	an overall time complexity of $O(n^{k})$.
	However, 
	assuming that Step 1 must take place, 
	the best case occurs when $|B_b| = k$, that is, there is only one bottom row block of size $k$,
	giving an overall time complexity of $O(n)$.

	\textbf{Step 2:} As these are transfer operations, there is no cost since we are simply 
	copying elements from an array.

	\textbf{Step 3:} Here we are copying arrays, hence there is no cost.

	Consequently, in the worst case, 
	we have reduced the overall time complexity from 
	$O(n^{l+k})$ to 
	$O(n^{k})$. 
	If Step 1 must occur, then in the best case, we have 
	reduced the overall time complexity from 
	$O(n^{l+k})$ to $O(n)$.
	However, the true best case occurs when Step 1 does not take place at all, 
	that is, when the number of bottom row blocks $b$ is zero, because,
	in this case, the computation is effectively free! 
	
	This analysis also shows why, in Definition \ref{algplanardefnSn} 
	of an algorithmically planar $(k,l)$--partition diagram, we have chosen
	to order the connected components that are solely
	in the bottom row by their size, in decreasing order from right to left. 
	We want to obtain the best overall time complexity for the multiplication algorithm,
	which, for the symmetric group, is determined by the number of additions and multiplications 
	that occur in Step 1.
	It is clear 
	from (\ref{Snmultalgo}) and (\ref{Snaddalgo})
	that the time complexity is best when the 
	blocks are ordered in decreasing size order from right to left.




\subsubsection{Orthogonal Group $O(n)$}

The implementation of Algorithm 1 for the orthogonal group $O(n)$ 
is similar to the implementation for the symmetric group $S_n$ as a result
of the similarity between the monoidal functors that are associated with each group.
In fact, the implementation for the orthogonal group will be less involved because
the only valid set partition diagrams that can be input into the \textbf{MatrixMult}
procedure for this group are Brauer diagrams.



\paragraph{Implementation}

We provide specific implementations of \textbf{Factor} and \textbf{PlanarMult}
for the orthogonal group $O(n)$.

\textbf{Factor}: The input is a $(k,l)$--Brauer diagram $d_\beta$.
We drag and bend the strings representing the connected components of $d_\beta$ 
to obtain a factoring of $d_\beta$ into three diagrams whose composition is equivalent to $d_\beta$:
a $(k,k)$--Brauer diagram that represents a permutation $\sigma_k$ in the symmetric group $S_k$; 
another $(k,l)$--Brauer diagram that is algorithmically planar; 
and a $(l,l)$--Brauer diagram that represents a permutation $\sigma_l$ in the symmetric group $S_l$.

To obtain the algorithmically planar $(k,l)$--Brauer diagram, 
we drag and bend the strings in any way such that 
\begin{itemize}
	\item the pairs that are solely in the bottom row of $d_\beta$ are pulled up to be next to each other
		in the far right hand side of the bottom row of the algorithmically planar $(k,l)$--Brauer diagram,
	\item the pairs that are solely in the top row of $d_\beta$ are pulled down to be next to each other
		in the far left hand side of the top row of the algorithmically planar $(k,l)$--Brauer diagram, and
	\item the pairs between vertices in different rows of $d_\beta$ are bent to be in between the other 
		vertices of the algorithmically planar $(k,l)$--Brauer diagram such that no two pairings in the algorithmically 
		planar diagram intersect each other.
\end{itemize}


\begin{figure*}[tb]
	\begin{tcolorbox}[colback=white!02, colframe=black]
	\begin{center}
		\scalebox{0.35}{\begin{tikzpicture}
	\begin{pgfonlayer}{nodelayer}
		\node [style=none] (104) at (47.5, -3) {};
		\node [style=none] (105) at (47.5, -6) {};
		\node [style=none] (106) at (60.5, -6) {};
		\node [style=none] (107) at (60.5, -3) {};
		\node [style=none] (0) at (14, -4.5) {\Huge$\bm{=}$};
		\node [style=node] (1) at (3, -9) {};
		\node [style=none] (2) at (-0.5, 0.5) {};
		\node [style=none] (3) at (-0.5, -9.5) {};
		\node [style=none] (4) at (12.5, -9.5) {};
		\node [style=none] (5) at (12.5, 0.5) {};
		\node [style=node] (6) at (3, 0) {};
		\node [style=node] (8) at (12, -9) {};
		\node [style=node] (9) at (0, -9) {};
		\node [style=node] (16) at (6, -9) {};
		\node [style=node] (17) at (9, -9) {};
		\node [style=node] (18) at (9, 0) {};
		\node [style=node] (19) at (6, 0) {};
		\node [style=none] (21) at (15.5, 0.5) {};
		\node [style=none] (22) at (15.5, -9.5) {};
		\node [style=none] (23) at (28.5, -9.5) {};
		\node [style=none] (24) at (28.5, 0.5) {};
		\node [style=node] (25) at (22, 0) {};
		\node [style=node] (26) at (28, -9) {};
		\node [style=node] (28) at (22, -9) {};
		\node [style=node] (29) at (25, -9) {};
		\node [style=node] (33) at (19, -9) {};
		\node [style=none] (37) at (30, -4.5) {\Huge$\bm{=}$};
		\node [style=node] (38) at (35, -9) {};
		\node [style=none] (39) at (31.5, 0.5) {};
		\node [style=none] (40) at (31.5, -9.5) {};
		\node [style=none] (41) at (44.5, -9.5) {};
		\node [style=none] (42) at (44.5, 0.5) {};
		\node [style=node] (43) at (32, 0) {};
		\node [style=node] (44) at (44, -9) {};
		\node [style=node] (45) at (32, -9) {};
		\node [style=node] (46) at (38, -9) {};
		\node [style=node] (47) at (41, -9) {};
		\node [style=node] (51) at (44, -6) {};
		\node [style=node] (52) at (41, -6) {};
		\node [style=none] (54) at (46, -4.5) {\Huge$\bm{=}$};
		\node [style=node] (73) at (35, -3) {};
		\node [style=node] (74) at (32, -3) {};
		\node [style=node] (75) at (25, 0) {};
		\node [style=node] (76) at (16, 0) {};
		\node [style=node] (80) at (51, -9) {};
		\node [style=none] (82) at (47.5, -9.5) {};
		\node [style=none] (83) at (60.5, -9.5) {};
		\node [style=node] (85) at (48, 0) {};
		\node [style=node] (86) at (60, -9) {};
		\node [style=node] (87) at (48, -9) {};
		\node [style=node] (88) at (54, -9) {};
		\node [style=node] (89) at (57, -9) {};
		\node [style=node] (90) at (48, -6) {};
		\node [style=node] (91) at (60, -6) {};
		\node [style=node] (92) at (57, -6) {};
		\node [style=node] (93) at (54, -6) {};
		\node [style=node] (94) at (51, -3) {};
		\node [style=node] (95) at (48, -3) {};
		\node [style=node] (96) at (51, 0) {};
		\node [style=node] (97) at (57, 0) {};
		\node [style=node] (99) at (51, -6) {};
		\node [style=none] (102) at (47.5, -6) {};
		\node [style=none] (103) at (60.5, -6) {};
		\node [style=none] (108) at (47.5, -3) {};
		\node [style=none] (109) at (60.5, -3) {};
		\node [style=none] (110) at (60.5, 0.5) {};
		\node [style=none] (111) at (47.5, 0.5) {};
		\node [style=node] (122) at (16, -9) {};
		\node [style=node] (123) at (44, 0) {};
		\node [style=node] (128) at (60, 0) {};
		\node [style=node] (129) at (60, -3) {};
		\node [style=node] (130) at (25, -6) {};
		\node [style=node] (131) at (28, -6) {};
		\node [style=node] (132) at (12, 0) {};
		\node [style=node] (133) at (0, 0) {};
		\node [style=node] (134) at (19, 0) {};
		\node [style=node] (135) at (28, 0) {};
		\node [style=node] (136) at (35, 0) {};
		\node [style=node] (137) at (38, 0) {};
		\node [style=node] (138) at (41, 0) {};
		\node [style=node] (146) at (54, 0) {};
		\node [style=node] (147) at (54, -3) {};
		\node [style=node] (148) at (57, -3) {};
	\end{pgfonlayer}
	\begin{pgfonlayer}{edgelayer}
		\draw [style={dash_2}] (22.center) to (23.center);
		\draw [style={dash_2}] (23.center) to (24.center);
		\draw [style={dash_2}] (24.center) to (21.center);
		\draw [style={dash_2}] (21.center) to (22.center);
		\draw [style={dash_2}] (40.center) to (41.center);
		\draw [style={dash_2}] (41.center) to (42.center);
		\draw [style={dash_2}] (42.center) to (39.center);
		\draw [style={dash_2}] (39.center) to (40.center);
		\draw [style=thick, bend right=45] (74) to (73);
		\draw [style={dash_5}] (110.center)
			 to (111.center)
			 to (108.center)
			 to (109.center)
			 to cycle;
		\draw [style={dash_5}] (103.center)
			 to (102.center)
			 to (82.center)
			 to (83.center)
			 to cycle;
		\draw [style=thick, bend left=45] (52) to (51);
		\draw [style=thick] (38) to (51);
		\draw [style=thick] (45) to (52);
		\draw [style=thick] (95) to (96);
		\draw [style=thick] (94) to (97);
		\draw [style=thick] (80) to (91);
		\draw [style=thick] (87) to (92);
		\draw [style={dash_2}] (3.center)
			 to (4.center)
			 to (5.center)
			 to (2.center)
			 to cycle;
		\draw [style=thick, bend left=315] (1) to (9);
		\draw [style=thick, bend left=45] (130) to (131);
		\draw [style=thick] (33) to (131);
		\draw [style=thick] (122) to (130);
		\draw [style=thick] (74) to (136);
		\draw [style=thick] (28) to (135);
		\draw [style=thick] (16) to (132);
		\draw [style=thick, bend right=45] (6) to (18);
		\draw [style=thick, bend right=45] (134) to (75);
		\draw [style=thick] (73) to (138);
		\draw [style=thick] (133) to (8);
		\draw [style=thick] (19) to (17);
		\draw [style=thick] (76) to (26);
		\draw [style=thick] (25) to (29);
		\draw [style=thick] (43) to (44);
		\draw [style=thick] (137) to (47);
		\draw [style=thick] (46) to (123);
		\draw [style={dash_2}] (104.center)
			 to (107.center)
			 to (106.center)
			 to (105.center)
			 to cycle;
		\draw [style=thick, bend left=45] (92) to (91);
		\draw [style=thick, bend right=45] (95) to (94);
		\draw [style=thick] (128) to (129);
		\draw [style=thick] (129) to (93);
		\draw [style=thick] (146) to (148);
		\draw [style=thick] (148) to (99);
		\draw [style=thick] (99) to (89);
		\draw [style=thick] (90) to (147);
		\draw [style=thick] (147) to (85);
		\draw [style=thick] (86) to (90);
		\draw [style=thick] (88) to (93);
	\end{pgfonlayer}
\end{tikzpicture}}
	\end{center}
	\end{tcolorbox}
	\caption{
		We use the string-like aspect of $(k,l)$--Brauer diagrams to \textbf{Factor} them as a composition of a permutation in $S_k$, an algorithmically planar $(k,l)$--Brauer diagram, and a permutation in $S_l$. Here $k = l = 5$.
	}
	\label{symmfactoringOn}
\end{figure*}
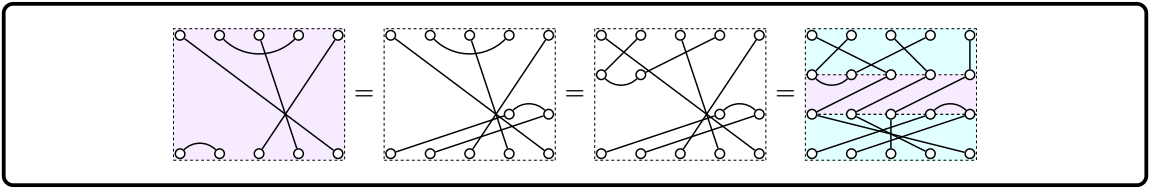


We give an example of this procedure in Figure \ref{symmfactoringOn}.

\textbf{PlanarMult}:
	We take as input the planar $(k,l)$--Brauer diagram $d_\beta$ that is the output of \textbf{Factor}, and a vector $v \in (\mathbb{R}^{n})^{\otimes k}$ that is the output of \textbf{Permute}, as per Algorithm 1.
	
	Let $t$ be the number of pairs that are solely in the top row of $d_\beta$,
	let $d$ be the number of pairs that are in different rows of $d_\beta$, and
	let $b$ be the number of pairs that are solely in the bottom row of $d_\beta$.

	Then it is clear that
	\begin{equation}
		2t + d = l \quad \text{and} \quad 2b + d = k
	\end{equation}
	and so
	\begin{equation}
		2t + 2d + 2b = l+k 
	\end{equation}

	Given how \textbf{Factor} constructs the planar $(k,l)$--Brauer diagram $d_\beta$, the Brauer partition $\beta$ corresponding to $d_\beta$ will be of the form
	\begin{equation} \label{brauerpartfactor}
		\left(\bigcup_{i = 1}^{t} T_i \right)
		\bigcup
		\left(\bigcup_{i = 1}^{d} D_i \right)
		\bigcup
		\left(\bigcup_{i = 1}^{b} B_i \right)
	\end{equation}
	where we have used $T_i$ to refer to a top row pair, $D_i$ to refer to a different row pair, and $B_i$ to refer to a bottom row pair.
	In particular, we have that
	\begin{itemize}
		\item $T_i = \left\{2i-1, 2i\right\}$ for all $i = 1 \rightarrow t$
		\item $D_i = \left\{l+i-d, l+i\right\}$ for all $i = 1 \rightarrow d$, and
		\item $B_i = \left\{l+k-2b+2i-1, l+k-2b+2i\right\}$ for all $i = 1 \rightarrow b$.
	\end{itemize}

	Next, we take $d_\beta$ and express it as a tensor product of three types of Brauer diagrams.
	The right-most type is itself a tensor product of diagrams corresponding to the $B_i$; the middle type is a single diagram corresponding to $\bigcup_{i = 1}^{d} D_i$; and the left-most type is itself a tensor product of diagrams corresponding to the $T_i$.

	We apply the monoidal functor $\Phi$ to this tensor product decomposition of diagrams, which returns a Kronecker product of matrices.
	We perform the matrix multiplication by applying the matrices 
	``right-to-left, diagram-by-diagram", as follows.

	\textbf{Step 1: Apply each matrix corresponding to a bottom row pair diagram, one-by-one, starting from the one that corresponds to $B_b$ and ending with the one that corresponds to $B_1$.}

	Suppose that we are performing the part of the matrix multiplication that corresponds to $B_i$, for some $i = 1 \rightarrow b$.
	The input will be a vector
	$w \in (\mathbb{R}^{n})^{\otimes k - 2(b-i)}$.

	We can express $w$ in the standard basis of $\mathbb{R}^{n}$ as
	\begin{equation}
		w = \sum_{L \in [n]^{k-2(b-i)}} w_Le_L
	\end{equation}

	This will be mapped to the vector
	$r \in (\mathbb{R}^{n})^{\otimes k - 2(b-i) - 2}$,
	where $r$ is of the form
	\begin{equation}
		r = \sum_{M \in [n]^{k-2(b-i)-2}} r_Me_M
	\end{equation}
	and
	\begin{equation} \label{rMcoeff}
		r_M = \sum_{j \in [n]} w_{M,j,j}
	\end{equation}
	At the end of this process, we obtain a vector in 
	$(\mathbb{R}^{n})^{\otimes k - 2b}$.

	As before, the matrices corresponding to bottom row pairs are merely performing tensor contractions.

	\textbf{Step 2: Now apply the matrix corresponding to the middle diagram, that is, to the set $\bigcup_{i = 1}^{d} D_i$.}
	
	The input will be a vector
	$w \in (\mathbb{R}^{n})^{\otimes k - 2b}$.
	
	Expressing $w$ in the standard basis of $\mathbb{R}^{n}$ as
	\begin{equation}
		w = \sum_{L \in [n]^{k-2b}} w_Le_L
	\end{equation}
	the multiplication of the matrix merely returns $w$ itself!
	
	Hence the transfer operations for the orthogonal group are simply the identity map.

	\textbf{Step 3: Finally, apply each matrix corresponding to a top row pair diagram, one-by-one, starting from the one that corresponds to $T_t$ and ending with the one that corresponds to $T_1$.}
	
	Suppose that we are performing the part of the matrix multiplication that corresponds to $T_i$, for some $i = 1 \rightarrow t$.

	Then we begin with a vector  
	$w \in (\mathbb{R}^{n})^{\otimes k - 2b + 2(t - i)}$ 
	that is of the form
	\begin{equation}
		w = \sum_{J \in [n]^{t-i}}\sum_{L \in [n]^{k-2b}} v_L \left(e_{j_1} \otimes e_{j_1} \otimes \dots \otimes e_{j_{t-i}} \otimes e_{j_{t-i}} \otimes e_L\right)
	\end{equation}
	where $v_L$ is the coefficient of $e_L$ appearing in the vector at the end of Step 2.
	
	This will be mapped to the vector
	$r \in (\mathbb{R}^{n})^{\otimes k - 2b+2(t-i) + 2}$,
	where $r$ is of the form
	\begin{equation}
		r = \sum_{m \in [n]}\sum_{J \in [n]^{t-i}}\sum_{L \in [n]^{k-2b+2(t-i)}} v_L \left(e_m \otimes e_m \otimes e_{j_1} \otimes e_{j_1} \otimes \dots \otimes e_{j_{t-i}} \otimes e_{j_{t-i}} \otimes e_L\right)
	\end{equation}
	At the end of this process, we obtain a vector in 
	$(\mathbb{R}^{n})^{\otimes l}$, since $k - 2b + 2t = l$, which is of the form
	\begin{equation}
		\sum_{J \in [n]^t}\sum_{L \in [n]^{k-2b}} v_L \left(e_{j_1} \otimes e_{j_1} \otimes \dots \otimes e_{j_t} \otimes e_{j_t} \otimes e_L\right)
	\end{equation}
	This is the vector that is returned by \textbf{PlanarMult} for the orthogonal group $O(n)$.

	To summarise, the implementation of \textbf{PlanarMult} for the orthogonal group $O(n)$
	takes as input 
	the algorithmically planar $(k,l)$--Brauer diagram that comes from \textbf{Factor} and 
	expresses it as a tensor product of three types of Brauer diagrams.
	The right-most type is itself a tensor product of Brauer diagrams, where each diagram has only two connected vertices in the bottom row.
	These diagrams correspond to tensor contraction operations under the functor $\Phi$ given in Theorem \ref{brauerO(n)functor}.
	The middle type is a Brauer diagram that consists of all of the pairs in the planar $(k,l)$--Brauer diagram between vertices in different rows. This diagram corresponds to the identity under the functor $\Phi$.
	The left-most type is a tensor product of Brauer diagrams having only two connected vertices in the top row. These diagrams correspond to indexing operations that perform copies under the functor $\Phi$. 


\begin{figure*}[tb]
	\begin{tcolorbox}[colback=white!02, colframe=black]
	\begin{center}
		\scalebox{0.6}{\begin{tikzpicture}
	\begin{pgfonlayer}{nodelayer}
		\node [style=node] (0) at (0, -3) {};
		\node [style=node] (1) at (12, -3) {};
		\node [style=node] (2) at (9, -3) {};
		\node [style=node] (3) at (5.75, -3) {};
		\node [style=node] (4) at (6, 0) {};
		\node [style=node] (5) at (3, 0) {};
		\node [style=node] (6) at (0, 0) {};
		\node [style=node] (7) at (3, -3) {};
		\node [style=none] (10) at (-0.5, 0.5) {};
		\node [style=none] (11) at (-0.5, -3.5) {};
		\node [style=none] (12) at (12.5, -3.5) {};
		\node [style=none] (13) at (12.5, 0.5) {};
		\node [style=none] (14) at (14, -1.5) {\Huge$\bm{=}$};
		\node [style=node] (15) at (16, 0) {};
		\node [style=node] (16) at (34, -3) {};
		\node [style=node] (17) at (31, -3) {};
		\node [style=node] (18) at (28, -3) {};
		\node [style=node] (19) at (25, 0) {};
		\node [style=node] (20) at (22, 0) {};
		\node [style=node] (21) at (28, 0) {};
		\node [style=node] (22) at (25, -3) {};
		\node [style=none] (23) at (15.5, -3.5) {};
		\node [style=none] (24) at (15.5, 0.5) {};
		\node [style=none] (25) at (19.5, 0.5) {};
		\node [style=none] (26) at (19.5, -3.5) {};
		\node [style=none] (27) at (34.5, -3.5) {};
		\node [style=none] (28) at (30.5, -3.5) {};
		\node [style=none] (29) at (34.5, 0.5) {};
		\node [style=none] (30) at (30.5, 0.5) {};
		\node [style=none] (35) at (28.5, -3.5) {};
		\node [style=none] (36) at (21.5, -3.5) {};
		\node [style=none] (37) at (28.5, 0.5) {};
		\node [style=none] (38) at (21.5, 0.5) {};
		\node [style=node] (39) at (9, 0) {};
		\node [style=node] (40) at (22, -3) {};
		\node [style=node] (41) at (12, 0) {};
		\node [style=node] (42) at (19, 0) {};
	\end{pgfonlayer}
	\begin{pgfonlayer}{edgelayer}
		\draw [style={dash_2}] (13.center)
			 to (10.center)
			 to (11.center)
			 to (12.center)
			 to cycle;
		\draw [style=thick] (0) to (4);
		\draw [style=thick] (7) to (39);
		\draw [style=thick] (3) to (41);
		\draw [style=thick, bend left=45] (2) to (1);
		\draw [style=thick, bend left=45] (5) to (6);
		\draw [style={dash_2}] (24.center)
			 to (23.center)
			 to (26.center)
			 to (25.center)
			 to cycle;
		\draw [style=thick, bend right=45] (15) to (42);
		\draw [style={dash_2}] (35.center)
			 to (37.center)
			 to (38.center)
			 to (36.center)
			 to cycle;
		\draw [style={dash_2}] (27.center)
			 to (29.center)
			 to (30.center)
			 to (28.center)
			 to cycle;
		\draw [style=thick, bend left=45] (17) to (16);
		\draw [style=thick] (40) to (20);
		\draw [style=thick] (22) to (19);
		\draw [style=thick] (18) to (21);
	\end{pgfonlayer}
\end{tikzpicture}}
	\end{center}
	\end{tcolorbox}
	\caption{
		The tensor product decomposition of the planar $(5,5)$--Brauer diagram that appears in Figure \ref{symmfactoringOn}.
		}
	\label{tensorproddecompOn}
\end{figure*}
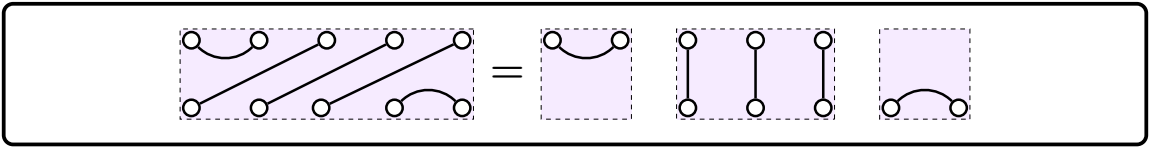


In
Figure \ref{tensorproddecompOn},
we present the tensor product decomposition 
of the algorithmically planar $(5,5)$--Brauer diagram 
that is shown in Figure \ref{symmfactoringOn}.

The diagrammatic representation of the matrix multiplication step 
is very similar to the one for the symmetric group, 
in that to obtain the matrices we perform the same deformation of the 
tensor product decomposition of diagrams before applying the functor $\Phi$
at each level.
We give an example in Figure \ref{planarmultOn} of how the computation takes place at each stage 
for the tensor product decomposition given in Figure \ref{tensorproddecompOn}, using its equivalent diagram form.

\begin{example} \label{orthogAlgoEx}
	Suppose that we wish to perform the multiplication of $E_\beta$
	by $v \in (\mathbb{R}^{n})^{\otimes 5}$, where
	$E_\beta$ corresponds to the $(5,5)$--Brauer diagram $d_\beta$ given in Figure \ref{symmfactoringOn} under $\Phi$,
	and $v$ is given by
		\begin{equation}
			\sum_{L \in [n]^5} v_Le_L
		\end{equation}

	We know that $E_\beta$ is a matrix in
	$\Hom_{O(n)}((\mathbb{R}^{n})^{\otimes 5}, (\mathbb{R}^{n})^{\otimes 5})$.

	First, we apply the procedure \textbf{Factor}, which returns the three diagrams given in Figure~\ref{symmfactoringOn}.
	The first diagram corresponds to the permutation $(1524)$ in $S_5$, hence, the result of \textbf{Permute}$(v, (1524))$ is the vector
		\begin{equation}
			\sum_{L \in [n]^5} 
			v_{l_1, l_2, l_3, l_4, l_5}
			\left(
				e_{l_5} \otimes 
				e_{l_4} \otimes 
				e_{l_3} \otimes
				e_{l_1} \otimes
				e_{l_2} 
			\right)
		\end{equation}
	Now we apply \textbf{PlanarMult} with the decomposition given in Figure \ref{tensorproddecompOn}.

	Step 1: Apply the matrices that correspond to the bottom row pairs.

	We obtain the vector
		\begin{equation}
			w = 
			\sum_{l_5, l_4, l_3 \in [n]} 
			w_{l_5, l_4, l_3}
			\left(
				e_{l_5} \otimes 
				e_{l_4} \otimes 
				e_{l_3}
			\right)
		\end{equation}
	where
		\begin{equation}
			w_{l_5, l_4, l_3}
			=
			\sum_{j \in [n]}
			v_{j, j, l_3, l_4, l_5}
		\end{equation}

	Step 2: Apply the matrices that correspond to the middle diagram.

	As the transfer operations are the identity mapping, we get $w$.	
	
	Step 3: Apply the matrices that correspond to the top row.

	We obtain the vector
		\begin{equation}
			z =
			\sum_{m \in [n]}	
			\sum_{l_5, l_4, l_3 \in [n]} 
			w_{l_5, l_4, l_3}
			\left(
				e_{m} \otimes e_{m} \otimes
				e_{l_5} \otimes 
				e_{l_4} \otimes 
				e_{l_3}
			\right)
		\end{equation}
	Substituting in, we get that
		\begin{equation}
			z =
			\sum_{m \in [n]}	
			\sum_{l_5, l_4, l_3 \in [n]} 
			\sum_{j \in [n]}
			v_{j, j, l_3, l_4, l_5}
			\left(
				e_{m} \otimes 
				e_{m} \otimes
				e_{l_5} \otimes 
				e_{l_4} \otimes 
				e_{l_3}
			\right)
		\end{equation}
	Finally, as the third diagram returned from \textbf{Factor} 
	corresponds to the permutation $(1342)$ in $S_5$,
	we perform \textbf{Permute}$(z, (1342))$,
	which returns the vector
		\begin{equation}
			\sum_{m \in [n]}	
			\sum_{l_5, l_4, l_3 \in [n]} 
			\sum_{j \in [n]}
			v_{j, j, l_3, l_4, l_5}
			\left(
				e_{l_5} \otimes 
				e_{m} \otimes 
				e_{l_4} \otimes 
				e_{m} \otimes
				e_{l_3}
			\right)
		\end{equation}
	This is the vector that is returned by \textbf{MatrixMult}.	
\end{example}


\begin{figure*}[tb]
	\begin{tcolorbox}[colback=white!02, colframe=black]
	\begin{center}
		\scalebox{0.4}{\input{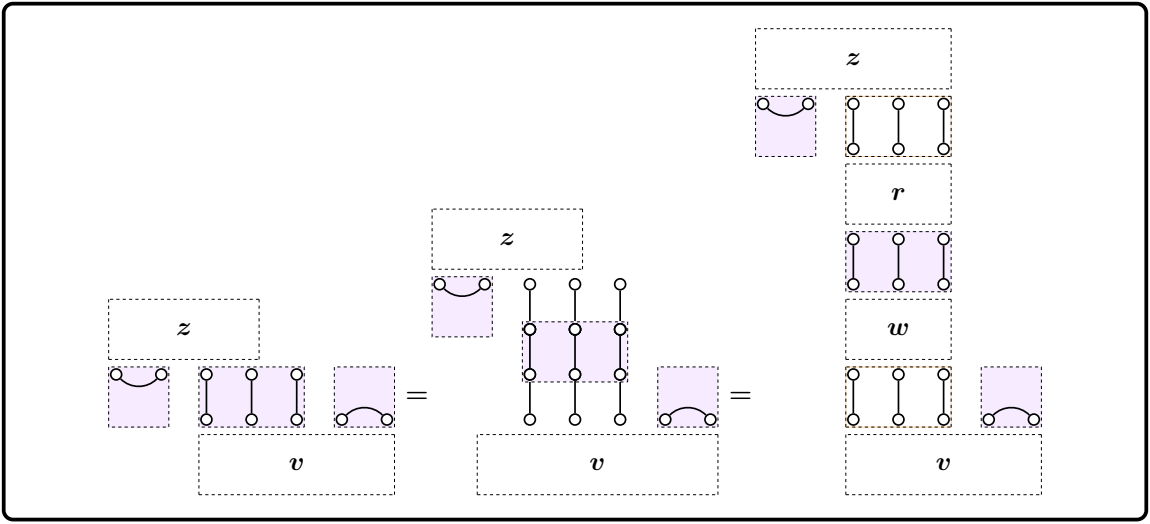}}
	\end{center}
	\end{tcolorbox}
		\caption{We show how matrix multiplication is implemented in \textbf{PlanarMult} for $O(n), Sp(n)$ and $SO(n)$ 
		using the tensor product decomposition of the planar 
		$(5,5)$--Brauer diagram
		given in Figure \ref{tensorproddecompOn}
		as an example.
		Effectively, we perform the matrix multiplication by applying the matrices ``right-to-left, diagram-by-diagram".
		In reality, we perform the matrix multiplication as follows:
		first, we deform the entire tensor product decomposition diagram by pulling each individual diagram up one level higher than the previous one, going from right--to-left, and then we apply the functor that corresponds to the group at each level. 
		Finally, we perform matrix multiplication at each level to obtain the final output vector.}
	\label{planarmultOn}
\end{figure*}


\paragraph{Time Complexity}

	We look at each step of the implementation that was given above.

	\textbf{Step 1:} For the part of the matrix multiplication that corresponds to $B_i$, for some $i = 1 \rightarrow b$, we map a vector in 
	$(\mathbb{R}^{n})^{\otimes k - 2(b-i)}$ 
	to a vector in
	$(\mathbb{R}^{n})^{\otimes k - 2(b-i) - 2}$.

	Since the matrix corresponds to a bottom row pair that is connected, for each tuple $M$ of indices in the output coefficient $r_M$, as in (\ref{rMcoeff}),
	there are only $n$ terms to multiply (an improvement over $n^2$), and consequently only $n-1$ additions.

	Hence, in total, there are
	\begin{equation}
		\sum_{i = 1}^{b}
		n^{k - 2(b-i) - 2} \cdot n
	\end{equation}
	multiplications and
	\begin{equation}
		\sum_{i = 1}^{b}
		n^{k - 2(b-i) -2} \cdot (n-1)
	\end{equation}
	additions, for an overall time complexity of $O(n^{k-1})$.

	\textbf{Step 2:} This corresponds to the identity transformation, hence there is no cost, as we do not need to perform this operation.

	\textbf{Step 3:} Here we are copying arrays, hence there is no cost.

	Consequently, 
	we have reduced the overall time complexity from 
	$O(n^{l+k})$ to 
	$O(n^{k-1})$.


\subsubsection{Symplectic Group $Sp(n)$}

The implementation of Algorithm 1 for the symplectic group $Sp(n)$
is related to the implementation for
the orthogonal group $O(n)$ because the 
only valid set partition diagrams that can be input into the procedure 
in each case are Brauer diagrams.
The difference in the implementations comes from the 
difference in the monoidal functor that is associated with each group.

\paragraph{Implementation}


The implementation of the \textbf{Factor} procedure is the same as for the orthogonal group.


\textbf{PlanarMult}:
	Again, we take as input the planar $(k,l)$--Brauer diagram $d_\beta$ that is the output 
	of \textbf{Factor}, and a vector $v \in (\mathbb{R}^{n})^{\otimes k}$ that is the output of \textbf{Permute}, as per Algorithm~1.

	The tensor product decomposition of $d_\beta$ is the same as for the orthogonal group $O(n)$; 
	in particular, the Brauer partition $\beta$ corresponding to $d_\beta$ is of the form (\ref{brauerpartfactor}).
	The important difference here is that we apply the monoidal functor $X$, instead of $\Phi$, 
	to this tensor product decomposition of diagrams, which returns a different Kronecker product of matrices.
	Again, we perform the same steps for the matrix multiplication, but this time the vectors returned at each stage are different.

	\textbf{Step 1: Apply each matrix corresponding to a bottom row pair diagram, one-by-one, 
	starting from the one that corresponds to $B_b$ and ending with the one that corresponds to $B_1$.}

	Suppose that we are performing the part of the matrix multiplication that corresponds to $B_i$, for some $i = 1 \rightarrow b$.
	The input will be a vector
	$w \in (\mathbb{R}^{n})^{\otimes k - 2(b-i)}$.

	We can express $w$ in the standard basis of $\mathbb{R}^{n}$ as
	\begin{equation}
		w = \sum_{L \in [n]^{k-2(b-i)}} w_Le_L
	\end{equation}

	This will be mapped to the vector
	$r \in (\mathbb{R}^{n})^{\otimes k - 2(b-i) - 2}$ 
	where $r$ is of the form
	\begin{equation}
		r = \sum_{M \in [n]^{k-2(b-i)-2}} r_Me_M
	\end{equation}
	where
	\begin{equation} 
		r_M = \sum_{j_{k-2(b-i)-1}, j_{k-2(b-i)} \in [n]} 
		\epsilon_{j_{k-2(b-i)-1},j_{k-2(b-i)}}
		w_{M,j_{k-2(b-i)-1},j_{k-2(b-i)}}
	\end{equation}
	Recall that $\epsilon_{j_{k-2(b-i)-1},j_{k-2(b-i)}}$ was defined 
	in (\ref{epsilondef1}) and (\ref{epsilondef2}). 

	At the end of this process, we obtain a vector in 
	$(\mathbb{R}^{n})^{\otimes k - 2b}$.
	
	As before, these matrices corresponding to bottom row pairs are performing tensor contractions.

	\textbf{Step 2: Now apply the matrix corresponding to the middle diagram, that is, to the set $\bigcup_{i = 1}^{d} D_i$.}

	This will be exactly the same as for the orthogonal group, by the definition of $\gamma_{r_p, u_p}$ given in (\ref{gammarpup}).

	\textbf{Step 3: Finally, apply each matrix corresponding to a top row pair diagram, one-by-one, 
	starting from the one that corresponds to $T_t$ and ending with the one that corresponds to $T_1$.}

	Suppose that we are performing the part of the matrix multiplication that corresponds to $T_i$, for some $i = 1 \rightarrow t$.

	Then we begin with a vector  
	$w \in (\mathbb{R}^{n})^{\otimes k - 2b + 2(t - i)}$ 
	that is of the form
	\begin{equation}
		w = 
		\sum_{J \in [n]^{2(t-i)}}\sum_{L \in [n]^{k-2b}} 
		\epsilon_J
		v_L 
		\left(
		e_J \otimes e_L\right)
	\end{equation}
	where 
	\begin{equation}
		\epsilon_J
		\coloneqq
		\epsilon_{j_1, j_2}
		\dots
		\epsilon_{j_{2(t-i)-1}, j_{2(t-i)}}
	\end{equation}
	and 
	where $v_L$ is the coefficient of $e_L$ appearing in the vector at the end of Step 2.
	
	This will be mapped to the vector
	$r \in (\mathbb{R}^{n})^{\otimes k - 2b+2(t-i) + 2}$,
	where $r$ is of the form
	\begin{equation}
		r = 
		\sum_{M \in [n]^2}
		\sum_{J \in [n]^{2(t-i)}}\sum_{L \in [n]^{k-2b}} 
		\epsilon_{M,J}
		v_L 
		\left(
		e_M \otimes e_J	\otimes e_L\right)
	\end{equation}
	where
	\begin{equation}
		\epsilon_{M,J}
		\coloneqq
		\epsilon_{m_1, m_2}
		\epsilon_{j_1, j_2}
		\dots
		\epsilon_{j_{2(t-i)-1}, j_{2(t-i)}}
	\end{equation}
	At the end of this process, we obtain a vector in 
	$(\mathbb{R}^{n})^{\otimes l}$, since $k - 2b + 2t = l$, which is of the form
	\begin{equation}
		\sum_{J \in [n]^{2t}}\sum_{L \in [n]^{k-2b}} 
		\epsilon_{J}
		v_L \left(
		e_J \otimes e_L\right)
	\end{equation}
	where $\epsilon_{J}$ is redefined to be
	\begin{equation}
		\epsilon_{j_1, j_2}
		\dots
		\epsilon_{j_{2t-1}, j_{2t}}
	\end{equation}
	This is the vector that is returned by \textbf{PlanarMult} for the symplectic group $Sp(n)$.

\begin{example}
	Suppose that we wish to perform the multiplication of $F_\beta$ 
	by $v \in (\mathbb{R}^{n})^{\otimes 5}$,
	for the same $(5,5)$--Brauer diagram $d_\beta$ given in Figure \ref{symmfactoringOn},
	and $v$ is given by
		\begin{equation}
			\sum_{L \in [n]^5} v_Le_L
		\end{equation}

	We know that $F_\beta$ is a matrix in
	$\Hom_{Sp(n)}((\mathbb{R}^{n})^{\otimes 5}, (\mathbb{R}^{n})^{\otimes 5})$.

	The \textbf{Factor} and \textbf{Permute} steps are the same as for Example \ref{orthogAlgoEx},
	hence, prior to the \textbf{PlanarMult} step, we have the vector
		\begin{equation}
			\sum_{L \in [n]^5} 
			v_{l_1, l_2, l_3, l_4, l_5}
			\left(
				e_{l_5} \otimes 
				e_{l_4} \otimes 
				e_{l_3} \otimes
				e_{l_1} \otimes
				e_{l_2} 
			\right)
		\end{equation}
	We now apply \textbf{PlanarMult} with the decomposition given in Figure \ref{tensorproddecompOn}.

	Step 1: Apply the matrices that correspond to the bottom row pairs.

	We obtain the vector
		\begin{equation}
			w = 
			\sum_{l_5, l_4, l_3 \in [n]} 
			w_{l_5, l_4, l_3}
			\left(
				e_{l_5} \otimes 
				e_{l_4} \otimes 
				e_{l_3}
			\right)
		\end{equation}
	where
		\begin{equation}
			w_{l_5, l_4, l_3}
			=
			\sum_{j_1, j_2 \in [n]}
			\epsilon_{j_1, j_2}
			v_{j_1, j_2, l_3, l_4, l_5}
		\end{equation}

	Step 2: Apply the matrices that correspond to the middle diagram.

	As the transfer operations are the identity mapping, we get $w$.	
	
	Step 3: Apply the matrices that correspond to the top row.

	We obtain the vector
		\begin{equation}
			z =
			\sum_{m_1, m_2 \in [n]}	
			\sum_{l_5, l_4, l_3 \in [n]} 
			\epsilon_{m_1, m_2}
			w_{l_5, l_4, l_3}
			\left(
				e_{m_1} \otimes e_{m_2} \otimes
				e_{l_5} \otimes 
				e_{l_4} \otimes 
				e_{l_3}
			\right)
		\end{equation}
	Substituting in, we get that
		\begin{equation}
			z =
			\sum_{m_1, m_2 \in [n]}	
			\sum_{l_5, l_4, l_3 \in [n]} 
			\sum_{j_1, j_2 \in [n]}
			\epsilon_{m_1, m_2}
			\epsilon_{j_1, j_2}
			v_{j_1, j_2, l_3, l_4, l_5}
			\left(
				e_{m_1} \otimes 
				e_{m_2} \otimes
				e_{l_5} \otimes 
				e_{l_4} \otimes 
				e_{l_3}
			\right)
		\end{equation}
	Finally, as the third diagram returned from \textbf{Factor} 
	corresponds to the permutation $(1342)$ in $S_5$,
	we perform \textbf{Permute}$(z, (1342))$,
	which returns the vector
		\begin{equation}
			\sum_{m_1, m_2 \in [n]}	
			\sum_{l_5, l_4, l_3 \in [n]} 
			\sum_{j_1, j_2 \in [n]}
			\epsilon_{m_1, m_2}
			\epsilon_{j_1, j_2}
			v_{j_1, j_2, l_3, l_4, l_5}
			\left(
				e_{l_5} \otimes 
				e_{m_1} \otimes 
				e_{l_4} \otimes 
				e_{m_2} \otimes
				e_{l_3}
			\right)
		\end{equation}
	This is the vector that is returned by \textbf{MatrixMult}.	
\end{example}

\paragraph{Time Complexity}
	The time complexity is exactly the same as for the orthogonal group.


\subsubsection{Special Orthogonal Group $SO(n)$}

We can 
perform matrix multiplication between
either
$E_\beta$ or $H_\alpha
\in \Hom_{SO(n)}((\mathbb{R}^{n})^{\otimes k}, (\mathbb{R}^{n})^{\otimes l})$
and
$v \in (\mathbb{R}^{n})^{\otimes k}$,
where $d_\beta$ is a $(k,l)$--Brauer diagram,
and $d_\alpha$ is an $(l+k) \backslash n$--diagram.
Given that
the implementation for the $E_\beta$ case is the same as for the orthogonal group, 
by Theorem~\ref{spanningsetSO(n)}, we
only consider the $H_\alpha$ case below.

\paragraph{Implementation}
	\hphantom{x}

\textbf{Factor}:
The input is an $(l+k) \backslash n$--diagram $d_\alpha$.
As before, we drag and bend the strings representing the connected components 
of $d_\alpha$ to obtain a factoring of $d_\alpha$ into the three diagrams, 
except this time the middle diagram will be an algorithmically 
planar $(l+k) \backslash n$--diagram.

To obtain the algorithmically planar 
$(l+k) \backslash n$--diagram
we want to drag and bend the strings in any way such that 
\begin{itemize}
	\item the free vertices in the top row of $d_\alpha$ are pulled down to the far right of the top row of the algorithmically planar $(l+k) \backslash n$--diagram, maintaining their order,
	\item the free vertices in the bottom row of $d_\alpha$ are pulled up to the far right of the bottom row of the algorithmically planar $(l+k) \backslash n$--diagram, maintaining their order,
	\item the pairs in the bottom row of $d_\alpha$ are pulled up to be next to each other in the right hand side of the bottom row of the algorithmically planar $(l+k) \backslash n$--diagram, but next to and to the left of the free vertices in the bottom row of the algorithmically planar $(l+k) \backslash n$--diagram, 
	\item the pairs in the top row of $d_\alpha$ are pulled down to be next to each other in the far left hand side of the top row of the algorithmically planar $(l+k) \backslash n$--diagram,
	\item the pairs connecting vertices in different rows of $d_\alpha$ are ordered in the algorithmically planar $(l+k) \backslash n$--diagram in between the other vertices such that no two pairings in the algorithmically planar diagram intersect each other.
\end{itemize}
We give an example of this procedure in Figure \ref{symmfactoringSOn}.


\begin{figure*}[tb]
	\begin{tcolorbox}[colback=white!02, colframe=black]
	\begin{center}
		\scalebox{0.35}{\begin{tikzpicture}
	\begin{pgfonlayer}{nodelayer}
		\node [style=none] (0) at (14, -4.5) {\Huge$\bm{=}$};
		\node [style=node] (1) at (3, -9) {};
		\node [style=none] (2) at (-0.5, 0.5) {};
		\node [style=none] (3) at (-0.5, -9.5) {};
		\node [style=none] (4) at (12.5, -9.5) {};
		\node [style=none] (5) at (12.5, 0.5) {};
		\node [style=node] (6) at (7.5, 0) {};
		\node [style=node] (8) at (12, -9) {};
		\node [style=node] (9) at (0, -9) {};
		\node [style=node] (16) at (6, -9) {};
		\node [style=node] (17) at (9, -9) {};
		\node [style=node] (18) at (10.5, 0) {};
		\node [style=node] (19) at (1.5, 0) {};
		\node [style=none] (21) at (15.5, 0.5) {};
		\node [style=none] (22) at (15.5, -9.5) {};
		\node [style=none] (23) at (28.5, -9.5) {};
		\node [style=none] (24) at (28.5, 0.5) {};
		\node [style=node] (25) at (23.5, 0) {};
		\node [style=node] (26) at (28, -9) {};
		\node [style=node] (28) at (22, -9) {};
		\node [style=node] (29) at (25, -9) {};
		\node [style=node] (33) at (19, -9) {};
		\node [style=none] (37) at (30, -4.5) {\Huge$\bm{=}$};
		\node [style=none] (54) at (46, -4.5) {\Huge$\bm{=}$};
		\node [style=node] (75) at (26.5, 0) {};
		\node [style=node] (76) at (17.5, 0) {};
		\node [style=node] (122) at (16, -9) {};
		\node [style=node] (130) at (25, -6) {};
		\node [style=node] (131) at (28, -6) {};
		\node [style=node] (133) at (4.5, 0) {};
		\node [style=node] (134) at (20.5, 0) {};
		\node [style=none] (149) at (31.5, 0.5) {};
		\node [style=none] (150) at (31.5, -9.5) {};
		\node [style=none] (151) at (44.5, -9.5) {};
		\node [style=none] (152) at (44.5, 0.5) {};
		\node [style=node] (153) at (38, -6) {};
		\node [style=node] (154) at (44, -9) {};
		\node [style=node] (155) at (38, -9) {};
		\node [style=node] (156) at (41, -9) {};
		\node [style=node] (157) at (35, -9) {};
		\node [style=node] (158) at (42.5, 0) {};
		\node [style=node] (159) at (33.5, 0) {};
		\node [style=node] (160) at (32, -9) {};
		\node [style=node] (161) at (41, -6) {};
		\node [style=node] (162) at (44, -6) {};
		\node [style=node] (163) at (35, -6) {};
		\node [style=node] (164) at (36.5, 0) {};
		\node [style=node] (165) at (39.5, 0) {};
		\node [style=none] (170) at (62, -4.5) {\Huge$\bm{=}$};
		\node [style=node] (171) at (67, -9) {};
		\node [style=none] (172) at (63.5, -9.5) {};
		\node [style=none] (173) at (76.5, -9.5) {};
		\node [style=node] (174) at (65.5, 0) {};
		\node [style=node] (175) at (76, -9) {};
		\node [style=node] (176) at (64, -9) {};
		\node [style=node] (177) at (70, -9) {};
		\node [style=node] (178) at (73, -9) {};
		\node [style=node] (179) at (64, -6) {};
		\node [style=node] (180) at (76, -6) {};
		\node [style=node] (181) at (73, -6) {};
		\node [style=node] (182) at (70, -6) {};
		\node [style=node] (183) at (68.5, -3) {};
		\node [style=node] (184) at (65.5, -3) {};
		\node [style=node] (185) at (68.5, 0) {};
		\node [style=node] (186) at (74.5, 0) {};
		\node [style=node] (187) at (67, -6) {};
		\node [style=none] (188) at (63.5, -6) {};
		\node [style=none] (189) at (76.5, -6) {};
		\node [style=none] (190) at (63.5, -3) {};
		\node [style=none] (191) at (76.5, -3) {};
		\node [style=none] (192) at (76.5, 0.5) {};
		\node [style=none] (193) at (63.5, 0.5) {};
		\node [style=node] (196) at (71.5, 0) {};
		\node [style=node] (197) at (71.5, -3) {};
		\node [style=node] (198) at (74.5, -3) {};
		\node [style=none] (199) at (47.5, 0.5) {};
		\node [style=none] (200) at (47.5, -9.5) {};
		\node [style=none] (201) at (60.5, -9.5) {};
		\node [style=none] (202) at (60.5, 0.5) {};
		\node [style=node] (203) at (52.5, -3) {};
		\node [style=node] (204) at (60, -9) {};
		\node [style=node] (205) at (54, -9) {};
		\node [style=node] (206) at (57, -9) {};
		\node [style=node] (207) at (51, -9) {};
		\node [style=node] (208) at (58.5, 0) {};
		\node [style=node] (209) at (49.5, 0) {};
		\node [style=node] (210) at (48, -9) {};
		\node [style=node] (211) at (57, -6) {};
		\node [style=node] (212) at (60, -6) {};
		\node [style=node] (213) at (49.5, -3) {};
		\node [style=node] (214) at (52.5, 0) {};
		\node [style=node] (215) at (55.5, 0) {};
		\node [style=node] (216) at (54, -6) {};
		\node [style=node] (217) at (51, -6) {};
		\node [style=node] (218) at (58, -3) {};
		\node [style=node] (219) at (26.5, -3) {};
		\node [style=node] (220) at (42.5, -3) {};
		\node [style=none] (221) at (63.5, -3) {};
		\node [style=none] (222) at (63.5, -6) {};
		\node [style=none] (223) at (76.5, -6) {};
		\node [style=none] (224) at (76.5, -3) {};
	\end{pgfonlayer}
	\begin{pgfonlayer}{edgelayer}
		\draw [style={dash_2}] (22.center) to (23.center);
		\draw [style={dash_2}] (23.center) to (24.center);
		\draw [style={dash_2}] (24.center) to (21.center);
		\draw [style={dash_2}] (21.center) to (22.center);
		\draw [style=thick] (28) to (75);
		\draw [style=thick, bend right=45] (134) to (25);
		\draw [style={dash_2}] (150.center) to (151.center);
		\draw [style={dash_2}] (151.center) to (152.center);
		\draw [style={dash_2}] (152.center) to (149.center);
		\draw [style={dash_2}] (149.center) to (150.center);
		\draw [style=thick] (155) to (158);
		\draw [style={dash_5}] (192.center)
			 to (193.center)
			 to (190.center)
			 to (191.center)
			 to cycle;
		\draw [style={dash_5}] (189.center)
			 to (188.center)
			 to (172.center)
			 to (173.center)
			 to cycle;
		\draw [style={dash_2}] (200.center) to (201.center);
		\draw [style={dash_2}] (201.center) to (202.center);
		\draw [style={dash_2}] (202.center) to (199.center);
		\draw [style={dash_2}] (199.center) to (200.center);
		\draw [style=thick] (205) to (208);
		\draw [style=thick, bend right=45] (213) to (203);
		\draw [style=thick] (215) to (203);
		\draw [style=thick] (213) to (214);
		\draw [style=thick] (209) to (218);
		\draw [style=thick] (184) to (185);
		\draw [style=thick] (183) to (196);
		\draw [style=thick] (174) to (198);
		\draw [style=thick] (177) to (179);
		\draw [style=thick] (197) to (186);
		\draw [style=thick] (76) to (219);
		\draw [style=thick] (122) to (130);
		\draw [style=thick] (33) to (131);
		\draw [style=thick, bend left=45] (29) to (26);
		\draw [style=thick] (160) to (161);
		\draw [style=thick] (157) to (162);
		\draw [style=thick, in=135, out=45] (163) to (153);
		\draw [style=thick] (156) to (163);
		\draw [style=thick] (154) to (153);
		\draw [style=thick] (159) to (220);
		\draw [style=thick, bend right=45] (164) to (165);
		\draw [style=thick, bend left=45] (217) to (216);
		\draw [style=thick] (207) to (212);
		\draw [style=thick] (210) to (211);
		\draw [style=thick] (206) to (217);
		\draw [style=thick] (204) to (216);
		\draw [style=thick] (178) to (187);
		\draw [style=thick] (175) to (182);
		\draw [style=thick] (176) to (181);
		\draw [style=thick] (171) to (180);
		\draw [style={dash_4}] (4.center)
			 to (3.center)
			 to (2.center)
			 to (5.center)
			 to cycle;
		\draw [style=thick] (16) to (18);
		\draw [style=thick, bend right=45] (8) to (17);
		\draw [style=thick, bend right=45] (133) to (6);
		\draw [style={dash_4}] (222.center)
			 to (221.center)
			 to (224.center)
			 to (223.center)
			 to cycle;
		\draw [style=thick] (179) to (197);
		\draw [style=thick, bend right=45] (184) to (183);
		\draw [style=thick, bend left=45] (187) to (182);
	\end{pgfonlayer}
\end{tikzpicture}}
	\end{center}
	\end{tcolorbox}
		\caption{We use the string-like aspect of 
		$(l+k) \backslash n$--diagrams
		to \textbf{Factor} them as a composition of a permutation in $S_k$, 
		an algorithmically planar 
		$(l+k) \backslash n$--diagram,
		and a permutation in $S_l$.
		Here, $k = 5$ and $l = 4$.
		}	
	\label{symmfactoringSOn}
\end{figure*}
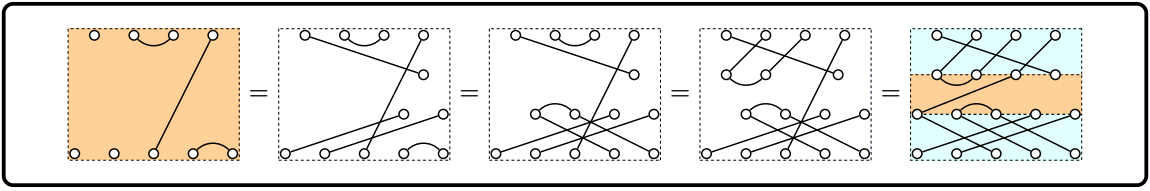


\textbf{PlanarMult}:
	We take as input the algorithmically planar $(l+k) \backslash n$--diagram $d_\alpha$ that
	is the output of \textbf{Factor}, and a vector $v \in (\mathbb{R}^{n})^{\otimes k}$ 
	that is the output of \textbf{Permute}, as per Algorithm 1.

	We need new subsets and notation to consider the impact of the 
	free vertices in the $(l+k) \backslash n$--diagram $d_\alpha$
	on the implementation for \textbf{PlanarMult}.

	As before,
	let $t$ be the number of pairs that are solely in the top row of $d_\alpha$,
	let $d$ be the number of pairs that are in different rows of $d_\alpha$, and
	let $b$ be the number of pairs that are solely in the bottom row of $d_\alpha$. 
	Now, let $s$ be the number of free vertices in the top row of $d_\alpha$. 
	Hence there are $n-s$ free vertices in the bottom row of $d_\alpha$.

	Then it is clear that
	\begin{equation} \label{SO(n)restrictions}
		2t + d + s = l \quad \text{and} \quad 2b + d + n - s = k
	\end{equation}
	and so
	\begin{equation}
		2t + 2d + 2b + n = l+k 
	\end{equation}

	Given how \textbf{Factor} constructs the algorithmically planar $(l+k) \backslash n$--diagram $d_\alpha$, 
	the set partition $\alpha$ corresponding to $d_\alpha$ will be of the form
	\begin{equation}
		\left(\bigcup_{i = 1}^{t} T_i \right)
		\bigcup
		\left(\bigcup_{i = 1}^{d} D_i \right)
		\bigcup
		\left(\bigcup_{i = 1}^{s} TF_i \right)
		\bigcup
		\left(\bigcup_{i = 1}^{b} B_i \right)
		\bigcup
		\left(\bigcup_{i = 1}^{n-s} BF_i \right)
	\end{equation}
	where we have used $T_i$ to refer to a top row pair, 
	$D_i$ to refer to a different row pair, 
	$TF_i$ to refer to a top row free vertex,
	$B_i$ to refer to a bottom row pair, and
	$BF_i$ to refer to a bottom row free vertex.
	In particular, we have that
	\begin{itemize}
		\item $T_i = \left\{2i-1, 2i\right\}$ for all $i = 1 \rightarrow t$,
		\item $D_i = \left\{l+i-d-s, l+i\right\}$ for all $i = 1 \rightarrow d$,
		\item $TF_i = \left\{l-s+i\right\}$ for all $i = 1 \rightarrow s$, 
		\item $B_i = \left\{l+d+2i-1, l+d+2i\right\}$ for all $i = 1 \rightarrow b$, and
		\item $BF_i = \left\{l+d+2b+i\right\}$ for all $i = 1 \rightarrow n-s$.
	\end{itemize}

	We take $d_\alpha$ and express it as a tensor product of four types of set partition diagrams.

	The right-most type is a diagram consisting of all the free vertices.
	Consequently, it corresponds to 
	$	
		\left(\bigcup_{i = 1}^{s} TF_i \right)
		\bigcup
		\left(\bigcup_{i = 1}^{n-s} BF_i \right)
	$.
	The type to its left is itself a tensor product of diagrams corresponding to the $B_i$; the next type is a single diagram corresponding to
	$\left(\bigcup_{i = 1}^{d} D_i\right)$; and, finally, the left-most type is itself a tensor product of diagrams corresponding to the $T_i$.

	We now apply the monoidal functor $\Psi$ to this tensor product decomposition of diagrams, which returns a Kronecker product of matrices.
	We perform the matrix multiplication by applying the matrices ``right-to-left, diagram-by-diagram", as follows.

	\textbf{Step 1: Apply the matrix that corresponds to the free vertices, that is, to 
	$	
		\left(\bigcup_{i = 1}^{s} TF_i \right)
		\bigcup
		\left(\bigcup_{i = 1}^{n-s} BF_i \right)
	$.}

	The input will be a vector
	$v \in (\mathbb{R}^{n})^{\otimes k}$.

	As $k = 2b + d + (n-s)$, we can express $v$ 
	in the standard basis of $\mathbb{R}^{n}$ as
	\begin{equation} \label{SOnStep1,1}
		v = 
		\sum_{J \in [n]^{2b + d}} 
		\sum_{B \in [n]^{n-s}} 
		v_{J,B}e_{J,B}
	\end{equation}
	This will be mapped to the vector $w \in (\mathbb{R}^{n})^{\otimes 2b+d+s}$, where $w$ is of the form
	\begin{equation} \label{SOnStep1,2}
		w =
		\sum_{J \in [n]^{2b + d}} 
		\sum_{T \in [n]^{s}} 
		w_{J,T}e_{J,T}
	\end{equation}
	where
	\begin{equation} \label{SOnStep1,3}
		w_{J,T} =
		\sum_{B \in [n]^{n-s}} 
		v_{J,B}
		\det(e_{T,B})
	\end{equation}

	\textbf{Step 2: Apply each matrix corresponding to a bottom row pair diagram, one-by-one, starting from the one that corresponds to $B_b$ and ending with the one that corresponds to $B_1$.}

	This is exactly the same as Step 1 for the orthogonal group.
	We obtain a vector $y \in (\mathbb{R}^{n})^{\otimes d+s}$

	\textbf{Step 3: Now apply the matrix corresponding to the middle diagram, that is, to the set 
	$\left(\bigcup_{i = 1}^{d} D_i\right)$.}

	This is exactly the same as Step 2 for the orthogonal group.
	We obtain a vector $r \coloneqq y \in (\mathbb{R}^{n})^{\otimes d+s}$

	\textbf{Step 4: Finally, apply each matrix corresponding to a top row pair diagram, one-by-one, starting from the one that corresponds to $T_t$ and ending with the one that corresponds to $T_1$.}
	
	This is exactly the same as Step 3 for the orthogonal group.
	We obtain a vector $z \in (\mathbb{R}^{n})^{\otimes 2t+d+s}$.
	As $2t+d+s = l$ by (\ref{SO(n)restrictions}),
	we have that $z \in (\mathbb{R}^{n})^{\otimes l}$, as required.

	To summarise, the implementation of \textbf{PlanarMult} 
	for $SO(n)$ differs slightly from the implementation of \textbf{PlanarMult} 
	for the groups that have come before.
	In the case where the spanning set element corresponds to a Brauer diagram, the implementation
	is the same as for the orthogonal group $O(n)$.
	Otherwise, the implementation takes as input 
	the algorithmically planar $(l+k) \backslash n$--diagram
	that comes from \textbf{Factor}, but expresses it as a tensor product of 
	\textit{four} types of set partition diagrams. 
	The right-most is a diagram consisting of all of the free vertices. 
	This corresponds to the determinant map 
	that is given in 
	(\ref{detmapdefn}).
	The three other types of diagrams that appear in the decomposition and 
	their corresponding operations are exactly the same as for the orthogonal group $O(n)$.

In
Figure \ref{tensorproddecompSOn}, we present an example of the tensor product decomposition for the
algorithmically planar $(4+5) \backslash 3$--diagram given in Figure \ref{symmfactoringSOn}.

The diagrammatic representation of the
matrix multiplication step is also very similar to the orthogonal group, 
in that to obtain the matrices we perform the same 
deformation of the tensor product decomposition of diagrams 
before applying the functor $\Psi$, defined in Theorem \ref{brauerSO(n)functor}, at each level.
Note, in particular, that we need to attach identity strings to the free vertices 
that appear in the top row.

We give an example in Figure \ref{planarmultSOn} of how the computation takes place at each stage 
for the tensor product decomposition given in Figure \ref{tensorproddecompSOn}, using its equivalent diagram form. 


\begin{figure*}[tb]
	\begin{tcolorbox}[colback=white!02, colframe=black]
	\begin{center}
		\scalebox{0.6}{\begin{tikzpicture}
	\begin{pgfonlayer}{nodelayer}
		\node [style=node] (0) at (0, -3) {};
		\node [style=node] (1) at (6, -3) {};
		\node [style=node] (2) at (3, -3) {};
		\node [style=node] (3) at (12, -3) {};
		\node [style=node] (4) at (7.5, 0) {};
		\node [style=node] (5) at (4.5, 0) {};
		\node [style=node] (6) at (1.5, 0) {};
		\node [style=node] (7) at (9, -3) {};
		\node [style=none] (10) at (-0.5, 0.5) {};
		\node [style=none] (11) at (-0.5, -3.5) {};
		\node [style=none] (12) at (12.5, -3.5) {};
		\node [style=none] (13) at (12.5, 0.5) {};
		\node [style=none] (14) at (14, -1.5) {$\bm{=}$};
		\node [style=node] (39) at (10.5, 0) {};
		\node [style=node] (43) at (16, 0) {};
		\node [style=node] (44) at (28, -3) {};
		\node [style=node] (45) at (25, -3) {};
		\node [style=node] (46) at (34, -3) {};
		\node [style=node] (47) at (32.5, 0) {};
		\node [style=node] (48) at (22, 0) {};
		\node [style=node] (49) at (31, -3) {};
		\node [style=none] (50) at (15.5, -3.5) {};
		\node [style=none] (51) at (15.5, 0.5) {};
		\node [style=none] (52) at (19.5, 0.5) {};
		\node [style=none] (53) at (19.5, -3.5) {};
		\node [style=none] (54) at (28.5, -3.5) {};
		\node [style=none] (55) at (24.5, -3.5) {};
		\node [style=none] (56) at (28.5, 0.5) {};
		\node [style=none] (57) at (24.5, 0.5) {};
		\node [style=none] (58) at (34.5, -3.5) {};
		\node [style=none] (59) at (30.5, -3.5) {};
		\node [style=none] (60) at (34.5, 0.5) {};
		\node [style=none] (61) at (30.5, 0.5) {};
		\node [style=node] (62) at (22, -3) {};
		\node [style=node] (63) at (19, 0) {};
		\node [style=none] (64) at (21.5, 0.5) {};
		\node [style=none] (65) at (22.5, 0.5) {};
		\node [style=none] (66) at (21.5, -3.5) {};
		\node [style=none] (67) at (22.5, -3.5) {};
	\end{pgfonlayer}
	\begin{pgfonlayer}{edgelayer}
		\draw [style={dash_2}] (52.center)
			 to (51.center)
			 to (50.center)
			 to (53.center)
			 to cycle;
		\draw [style=thick, bend right=60, looseness=1.25] (43) to (63);
		\draw [style={dash_2}] (55.center)
			 to (57.center)
			 to (56.center)
			 to (54.center)
			 to cycle;
		\draw [style=thick, bend left=60, looseness=1.25] (45) to (44);
		\draw [style={dash_2}] (64.center)
			 to (66.center)
			 to (67.center)
			 to (65.center)
			 to cycle;
		\draw [style=thick] (62) to (48);
		\draw [style={dash_4}] (11.center)
			 to (12.center)
			 to (13.center)
			 to (10.center)
			 to cycle;
		\draw [style=thick] (0) to (4);
		\draw [style=thick, bend right=45] (6) to (5);
		\draw [style=thick, bend left=45] (2) to (1);
		\draw [style={dash_4}] (58.center)
			 to (60.center)
			 to (61.center)
			 to (59.center)
			 to cycle;
	\end{pgfonlayer}
\end{tikzpicture}}
	\end{center}
	\end{tcolorbox}
		\caption{
		The tensor product decomposition of the algorithmically planar $(4+5) \backslash 3$--diagram that appears in Figure \ref{symmfactoringSOn}.}
	\label{tensorproddecompSOn}
\end{figure*}
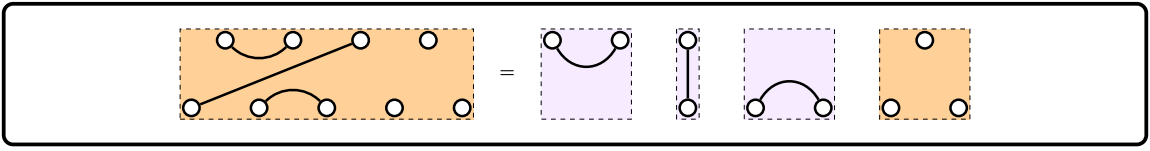


\begin{example}
	Suppose that we wish to perform the multiplication of $H_\alpha$
	by $v \in (\mathbb{R}^{3})^{\otimes 5}$, where
	$H_\alpha$ corresponds to the $(4+5) \backslash 3$--diagram $d_\alpha$ given in Figure \ref{symmfactoringSOn} under $\Psi$,
	and $v$ is given by
		\begin{equation}
			\sum_{L \in [3]^5} v_Le_L
		\end{equation}
	when expressed in the standard basis of $\mathbb{R}^{3}$.
	
	We know that $H_\alpha$ is a matrix in
	$\Hom_{SO(3)}((\mathbb{R}^{3})^{\otimes 5}, (\mathbb{R}^{3})^{\otimes 4})$.

	First, we apply the procedure \textbf{Factor}, which returns the three diagrams 
	given in Figure~\ref{symmfactoringSOn}.
	The first diagram corresponds to the permutation $(13524)$ in $S_5$, hence, the result of \textbf{Permute}$(v, (13524))$ is the vector
		\begin{equation}
			\sum_{L \in [3]^5} 
			v_{l_1, l_2, l_3, l_4, l_5}
			\left(
				e_{l_3} \otimes 
				e_{l_4} \otimes 
				e_{l_5} \otimes
				e_{l_1} \otimes
				e_{l_2} 
			\right)
		\end{equation}
	We now apply \textbf{PlanarMult} with the decomposition given in Figure \ref{tensorproddecompSOn}.

	Step 1: Apply the matrices that correspond to the free vertices.
	
	We obtain the vector
		\begin{equation}
			w = 
			\sum_{l_3, l_4, l_5, t_1 \in [3]} 
			w_{l_3, l_4, l_5, t_1}
			\left(
				e_{l_3} \otimes 
				e_{l_4} \otimes 
				e_{l_5} \otimes 
				e_{t_1}
			\right)
		\end{equation}
	where
		\begin{equation}
			w_{l_3, l_4, l_5, t_1}
			=
			\sum_{l_1, l_2 \in [3]}
			v_{l_1, l_2, l_3, l_4, l_5}
			\det(e_{t_1} \otimes e_{l_1} \otimes e_{l_2})
		\end{equation}

	Step 2: Apply the matrices that correspond to the bottom row pairs.

	We obtain the vector
		\begin{equation}
			y = 
			\sum_{l_3, t_1 \in [3]} 
			y_{l_3, t_1}
			\left(
				e_{l_3} \otimes 
				e_{t_1}
			\right)
		\end{equation}
	where
		\begin{equation}
			y_{l_3, t_1}
			=
			\sum_{j \in [3]}
			w_{l_3, j, j, t_1}
		\end{equation}
	
	Step 3: Apply the matrices that correspond to the different row pairs.

	Here, we get that $r = y$, as the transfer operations correspond to the identity.

	Step 4: Apply the matrices that correspond to the top row pairs.

	We obtain the vector
		\begin{equation}
			z =
			\sum_{m \in [3]}	
			\sum_{l_3, t_1 \in [3]}
			y_{l_3, t_1}
			\left(	
			e_{m} \otimes e_{m} 
			\otimes e_{l_3} \otimes e_{t_1}
			\right)
		\end{equation}
	Substituting in, we get that
		\begin{equation}
			z =
			\sum_{m \in [3]}	
			\sum_{l_3, t_1 \in [3]}
			\sum_{j \in [3]}
			w_{l_3, j, j, t_1}
			\left(	
			e_{m} \otimes e_{m} 
			\otimes e_{l_3} \otimes e_{t_1}
			\right)
		\end{equation}
	and hence
		\begin{equation}
			z =
			\sum_{m \in [3]}	
			\sum_{l_3, t_1 \in [3]}
			\sum_{j \in [3]}
			\sum_{l_1, l_2 \in [3]}
			v_{l_1, l_2, l_3, j, j}
			\det(e_{t_1} \otimes e_{l_1} \otimes e_{l_2})
			\left(	
			e_{m} \otimes e_{m} 
			\otimes e_{l_3} \otimes e_{t_1}
			\right)
		\end{equation}
	Finally, as the third diagram returned from \textbf{Factor} 
	corresponds to the permutation $(1432)$ in $S_4$,
	we perform \textbf{Permute}$(z, (1432))$
	which returns the vector
		\begin{equation}
			\sum_{m \in [3]}	
			\sum_{l_3, t_1 \in [3]}
			\sum_{j \in [3]}
			\sum_{l_1, l_2 \in [3]}
			v_{l_1, l_2, l_3, j, j}
			\det(e_{t_1} \otimes e_{l_1} \otimes e_{l_2})
			\left(	
			e_{t_1} \otimes e_{m} \otimes e_{m} 
			\otimes e_{l_3} 
			\right)
		\end{equation}
	This is the vector that is returned by \textbf{MatrixMult}.	
\end{example}


\begin{figure*}[tb]
	\begin{tcolorbox}[colback=white!02, colframe=black]
	\begin{center}
		\scalebox{0.4}{\input{algorithm/planarmultSOn.tikz}}
	\end{center}
	\end{tcolorbox}
		\caption{We show how matrix multiplication is implemented in \textbf{PlanarMult} for $SO(n)$ (here $n=3$) using the tensor product decomposition of the 
		algorithmically planar 
		$(4+5) \backslash 3$--diagram 
		given in Figure \ref{tensorproddecompSOn}
		as an example.
		We perform the matrix multiplication as follows:
		first, we deform the entire tensor product decomposition diagram by pulling each individual diagram up one level higher than the previous one, going from right--to-left, and then we apply the functor $\Psi$ at each level. 
Note, in particular, that we need to attach identity strings to the free vertices appearing in the top row.
		Finally, we perform matrix multiplication at each level to obtain the final output vector.}
	\label{planarmultSOn}
\end{figure*}


\paragraph{Time Complexity}

	For a matrix corresponding to a $(k,l)$--Brauer diagram, the analysis is the same as for the orthogonal group.
	For a matrix corresponding to an $(l+k) \backslash n$--diagram, we look at each step of the implementation that was given above.

	\textbf{Step 1:} 
	Recall that we map a vector in 
	$(\mathbb{R}^{n})^{\otimes k}$ 
	to a vector in
	$(\mathbb{R}^{n})^{\otimes 2b+d+s}$ in this step.

	To obtain the best time complexity, we need to look at the tuples 
	$J \in [n]^{2b+d}, T \in [n]^{s}$ and $B \in [n]^{n-s}$
	that appear in
	(\ref{SOnStep1,1}),
	(\ref{SOnStep1,2}), and
	(\ref{SOnStep1,3}).

	Indeed, for each tuple $J$, we first need to consider how many tuples $T$ come with pairwise different entries in $[n]$.
	The number of such tuples is $\frac{n!}{(n-s)!}$.
	Consequently, we only need to perform multiplications and additions for entries with such indices $(J,T)$.
	Let us call any such pair of indices $(J,T)$ \textbf{valid}.

	Hence, for each valid tuple $(J,T)$, there are $(n-s)!$ multiplications
	and $(n-s)! - 1$ additions.
	Since the number of valid tuples of the form $(J,T)$ is
	\begin{equation}
		n^{2b+d}\frac{n!}{(n-s)!}
	\end{equation}
	we have that the overall time complexity is
	$O(n^{k - (n-s)}n!)$ for this step, since $2b+d = k - (n-s)$ by (\ref{SO(n)restrictions}).


	\textbf{Step 2:} By the operations performed in this step, we can immediately apply the analysis of Step 1 for the orthogonal group.
	
	Since the part of the matrix multiplication that corresponds to $B_i$, for some $i = 1 \rightarrow b$, maps a vector in	
	$(\mathbb{R}^{n})^{\otimes d+s+2i}$ 
	to a vector in
	$(\mathbb{R}^{n})^{\otimes d+s+2i-2}$, we have that
	the overall time complexity is
	$O(n^{k+s-(n-s)-1})$ for this step, since
	$2b + d + s = k + s - (n-s)$
	by (\ref{SO(n)restrictions}).

	\textbf{Steps 3, 4:} These steps correspond to Steps 2, 3 for the orthogonal group, hence they have no cost.
	
	Hence, we have reduced the overall time complexity from
	$O(n^{l+k})$
	to
	\begin{equation}	
		O(n^{k - (n-s)}(n! + n^{s-1}))
	\end{equation}


\section{Conclusion} \label{conclusion}

Group equivariant neural networks that use high-order tensor power spaces of
$\mathbb{R}^{n}$ as their layers often face prohibitively high computational costs
when an equivariant weight matrix is applied to an input vector.
In this work, we have developed an algorithm that 
improves upon a na\"{i}ve weight matrix multiplication 
exponentially in terms of its Big-$O$ time complexity for four groups:
the symmetric, orthogonal, special orthogonal and symplectic groups.
We showed that a category-theoretic framework -- namely, expressing each weight matrix
as the image of
a linear combination of diagrams under a monoidal functor -- 
was very effective in helping to reorganise the overall computation so that it
can be executed in an optimal manner.
We hope that our algorithm will encourage more widespread adoption of group 
equivariant neural networks that use high-order tensor power spaces 
in practical applications.



\newpage

\acks{
This work was funded by the Doctoral
Scholarship for Applied Research which was
awarded to the first author under Imperial College
London's Department of Computing Applied
Research scheme.  This work will form part of
the first author's PhD thesis at Imperial College
London.
}


\nocite{*}
\bibliography{sample}

\end{document}